\documentclass{article}

\usepackage[main, final]{neurips_2026} %
\usepackage{amsmath}

\usepackage[utf8]{inputenc} %
\usepackage[T1]{fontenc}    %
\usepackage{hyperref}       %
\hypersetup{pdftitle={A Measure-Theoretic Analysis of Reasoning in Transformers: Structural Generalization and Approximation Limits}, pdfauthor={Yuyang Zhang, Yifu Zhang, Xuehai Zhou, Xiaoyin Chen}}
\usepackage{url}            %
\usepackage{xurl}           %
\usepackage{upquote}        %
\usepackage{booktabs}       %
\usepackage{amsfonts}       %
\usepackage{nicefrac}       %
\usepackage{microtype}      %
\usepackage{xcolor}         %
\usepackage{graphicx}      %

\usepackage{amssymb}

\usepackage{amsmath}

\usepackage{amsthm}

\newtheorem{assumption}{Assumption}
\newtheorem{lemma}{Lemma}
\newtheorem{definition}{Definition}
\newtheorem{corollary}{Corollary}
\newtheorem{theorem}{Theorem}

\title{A Measure-Theoretic Analysis of Reasoning in Transformers: Structural Generalization and Approximation Limits
}

\author{
  Yuyang Zhang\\
  McGill University\\
  \texttt{yuyang.zhang@mail.mcgill.ca} \\
  \And
  Yifu Zhang\thanks{Equal contribution.}\\
  McGill University \\
  \texttt{yifu.zhang3@mail.mcgill.ca} \\
  \AND
  Xuehai Zhou\footnotemark[1]\\
  McGill University \\
  \texttt{xuehai.zhou@mail.mcgill.ca} \\
  \And
  Xiaoyin Chen\thanks{Corresponding author.} \\ %
  Mila -- Quebec AI Institute \\
  \texttt{xiaoyin.chen@mila.quebec} \\
}

\begin{document}

\maketitle
\begin{abstract}
Scaling laws for reasoning in large language models (LLMs) are well documented, but little theory explains their out-of-distribution (OOD) generalization. We formalize reasoning with optimal transport: each search trajectory is mapped to twelve structural coordinates, such as tree depth and backtracking rate, and dataset shift is measured by the Wasserstein-1 distance. Kantorovich duality bounds the OOD risk through this distance, the architecture's Lipschitz constant and approximation limits. First, under absolute positional encoding (APE), shifting part of a trajectory to later positions almost surely changes the attention scores within that part; if this misroutes exact-match search, the Lipschitz constant and risk are $\Omega(1)$. Shift-invariant encodings such as rotary position embedding (RoPE) leave the scores unchanged and bound the error. Second, we map backtracking to a Dyck-$k$ language and, under an implicit-depth assumption, derive a risk floor for search trees deeper than a multiple of the number of layers. Enough layers remove this floor; more width only shrinks a separate Barron-space approximation term. Experiments with 54 trained Transformers on combinatorial search (Countdown) support these bounds. Appendix experiments measure structural shifts in public agent, embodied and robot trajectories, shift all position indices of 33 public language models, and split deeply nested tasks into separate model calls.
\end{abstract}

\section{Introduction}
\label{sec:intro}

Prompting techniques~\citep{wei2023chainofthoughtpromptingelicitsreasoning,yao2023treethoughtsdeliberateproblem} and post-training methods such as reinforcement learning~\citep{lightman2023let, o1, deepseekai2025deepseekr1incentivizingreasoningcapability,lambert2024tulu} enable large language models (LLMs) to generate explicit, multi-step reasoning traces. In formal logical domains, such as combinatorial planning and mathematical problem-solving, these traces correspond to traversals over an implicit algorithmic search tree. While the empirical scaling properties of reasoning are widely documented, theory lags behind. We ask: \textit{Under what architectural conditions can an autoregressive model reliably generalize its reasoning capabilities across structural domain shifts?} When a model trained on shallow search traces is evaluated on deeper, more sequential out-of-distribution (OOD) search trees, standard sequence-level log-likelihoods and discrete zero-one losses cannot measure the structural shift between training and test data, so the question remains largely unanswered.

Prior theoretical and empirical analyses of LLM reasoning address parts of this question separately. Complexity bounds for Transformers ($\text{TC}^0$) establish distribution-agnostic limits on sequential reasoning \citep{Hahn_2020, merrill2023parallelismtradeofflimitationslogprecision, feng2023revealingmysterychainthought}. Studies of length extrapolation evaluate mechanisms such as absolute positional encoding (APE) and rotary position embedding (RoPE) \citep{kazemnejad2023impactpositionalencodinglength, ruoss2023randomizedpositionalencodingsboost} but rarely bound computational capacity at the same time. Data-centric analyses focus on distribution shifts and treat the architecture as a black box \citep{zhao2026chainofthoughtreasoningllmsmirage, prystawski2023thinkstepstepreasoning}. No existing framework connects discrete circuit complexity with continuous distribution shifts.

Our main contribution is a measure-theoretic framework that unifies four variables in LLM reasoning: discrete circuit complexity (layer depth $L$), representation capacity (hidden width $m$), architectural priors (positional encodings), and continuous data distributions (measure shifts). We formulate the evaluation of reasoning as an optimal transport problem. We project each discrete, variable-length reasoning trajectory into a continuous metric space whose twelve coordinates are structural statistics of the search, such as tree depth and backtracking rate; any reasoning dataset then becomes a pushforward probability measure, and the discrepancy between two reasoning tasks is quantified by the Wasserstein-1 distance between their measures. 

We introduce an $\epsilon$-relaxed structural sufficiency assumption that connects this continuous domain shift with the discrete expected risk and casts LLM reasoning as structural regression over the continuous metric space. We approximate the expected zero-one loss by a Lipschitz continuous structural risk function over this space, apply Kantorovich duality, and bound the target generalization risk by three terms: the optimal transport cost (Wasserstein domain shift), the structural Lipschitz constant of the architecture ($K_{f,\mathrm{PE}}$), and its width and depth capacity terms ($C_{\mathrm{width}}$ and $C_{\mathrm{depth}}$). Because the resulting lower bound is the maximum of these terms, exceeding any single capacity limit induces an $\Omega(1)$ expected risk. From this unified bound, we establish two architectural results:

First, we study translation equivariance through the Lipschitz gap between the two main families of positional encodings, absolute and relative \citep{zhao2024lengthextrapolationtransformerssurvey}. We take standard APE as the baseline for position-dependent attention kernels and RoPE \citep{su2023roformerenhancedtransformerrotary} as the representative of the broader class of shift-invariant kernels (Appendix~\ref{app:pe_equivalence} discusses other encodings).
Because the same subtree traversal can occur at arbitrary absolute indices, generalization across positions requires attention scores that are invariant to global index shifts.
We prove that APE breaks this symmetry: shifting a sub-trajectory to later positions changes the attention scores within it almost surely (Lemma~\ref{lemma:ape_shattering}). When this change misroutes exact-match search (Assumption~\ref{assum:ar_sensitivity}), APE has a structural Lipschitz constant $K_{f, \mathrm{APE}} = \Omega(1)$ and an $\Omega(1)$ expected OOD risk whenever the target measure places non-vanishing mass on the disrupted structures (Theorem~\ref{thm:lipschitz_gap}, Corollary~\ref{cor:ape_vs_rope}). %
 RoPE leaves these scores unchanged, keeps the Lipschitz constant at $\mathcal{O}(1/T_{\max})$ and is thus a continuous structural regularizer that prevents this failure.

Second, we build on the $\text{TC}^0$ limits of Transformers \citep{merrill2023parallelismtradeofflimitationslogprecision} to derive a depth-dependent risk lower bound. Mapping the grammar of sequential backtracking to a Dyck-$k$ language links continuous measure shifts with discrete circuit complexity. Because self-attention cannot model hierarchical structure unless the number of layers or heads grows with the input length \citep{Hahn_2020}, we assume that the fixed physical layer depth $L$ of $\text{TC}^0$ Transformers limits the tree depth that one forward pass can resolve (Assumption~\ref{ass:implicit-complexity}). %
 Scaling depth $L$ is necessary to avoid an $\Omega(1)$ expected risk; scaling width $m$ cannot, as it only shrinks a separate Barron-space term of order $m^{-1/2}$ (Lemma~\ref{lemma:width_capacity}; cf.\ the $\mathcal{O}(m^{-1/2})$ upper rates of \citet{Barron1993} and \citet{MAKOVOZ199698}). %

We test our framework on Countdown, a combinatorial planning task, with 54 Transformers trained from scratch and evaluated on controlled target distributions. The models cross three training measures, Small/Medium/Large scales, Deep/Wide/Balanced aspect ratios, and APE/RoPE. The training measures are generated by breadth-first search (BFS) with lateral state expansions, by depth-first search (DFS) with frequent backtracking, and by a 50/50 mixture of the two that induces large position shifts. The empirical results support our measure-theoretic bounds: 
(1) On the DFS evaluation set, all 54 models lie below the $\text{TC}^0$ ceiling motivated by Lemma~\ref{lemma:depth_bottleneck} and fitted to them (Appendix~\ref{app:empirical_calibration}); doubling the width at fixed depth changes the best accuracy by under one point. 
(2) Accuracy plateaus or collapses when parameters are added through width: under APE, Deep models beat Wide ones in all 18 Medium- and Large-scale comparisons (by up to $11.25$ points), a gap that RoPE closes to under one point. 
(3) Under the large position shifts of the mixture measure, scaling Wide APE models triggers an $\Omega(1)$ representation collapse ($0.4215$ accuracy at Large scale), while RoPE at the same depth and width does not collapse ($0.5425$). 
(4) RoPE outperforms APE in 25 of 27 matched model pairs (one-sided sign test $p<10^{-5}$; 71 of 81 comparisons across evaluation sets), including all 18 Medium- and Large-scale pairs, with a gap that grows with scale. 
Appendix~\ref{app:beyond} applies the bound beyond Countdown. Moving Qwen3-8B from the mini-SWE-agent harness to OpenHands shifts its trajectories along every coordinate, as measured by the estimated Wasserstein-1 distance, and raises its resolve rate on SWE-bench Verified from $1.7\%$ to $6.4\%$ (Appendix~\ref{app:c2-harness}). In fp32, the shifts of position indices that occur in serving (left padding, cached prefixes, sliding windows) cost APE models up to $0.52$ in accuracy at a shift of $256$ positions, while the accuracy of RoPE and ALiBi models changes by at most $0.0003$ (Appendix~\ref{app:serving}). Splitting nested tool calls into one level per call raises the accuracy of eight open-weight LLMs at depth $8$ from $0.01$--$0.15$ to $0.43$--$0.99$ (Appendix~\ref{app:llm-handoff}).

\section{Preliminaries and Measure-Theoretic Formulation}
\label{sec:preliminaries}

To analyze generalization on search-based reasoning tasks, we formalize autoregressive generation and map discrete, variable-length trajectories into a continuous metric space, where optimal transport quantifies domain shifts. Full constructions and proofs are in Appendices \ref{app:setup} and \ref{app:structural-results}.
 
\subsection{The Measurable Space of Search Trajectories}

Let $\mathcal{V}$ be a finite discrete vocabulary. The ambient space of bounded-length sequences up to a maximum physical length $T_{\max} \in \mathbb{N}^+$ is defined as $\mathcal{Z}_{\mathrm{all}} = \bigcup_{t=1}^{T_{\max}} \mathcal{V}^t$. We define the search trajectory space $\mathcal{Z} \subset \mathcal{Z}_{\mathrm{all}}$ as the subset of sequences satisfying a valid autoregressive search grammar $\Gamma$, i.e., the rules by which a search tree is serialized into tokens (expand a node, evaluate a child, prune a dead end, backtrack to a pending node); its transition alphabet is formalized in Appendix~\ref{app:dfs-dyck}. %
 Every trajectory $z \in \mathcal{Z}$ decomposes uniquely under the string concatenation operator $\oplus$ as $z = x_{\mathrm{init}}^{(z)} \oplus y_{\mathrm{search}}$, where $x_{\mathrm{init}}^{(z)}$ is the initial problem state and $y_{\mathrm{search}}$ is the generated search trace.

Given finite $|\mathcal{V}|$ and $T_{\max}$, the set $\mathcal{Z}$ is finite. We equip $\mathcal{Z}$ with the discrete $\sigma$-algebra $\Sigma_{\mathcal{Z}} := 2^{\mathcal{Z}}$, forming a complete measurable space with source and target probability measures $\mu_S, \mu_T \in \mathcal{P}(\mathcal{Z})$.

Let $\mathcal{F}_{L, m, \mathrm{PE}}$ denote the hypothesis class of autoregressive Transformer functions characterized by physical layer depth $L$, hidden dimension $m$, and positional encoding scheme $\mathrm{PE} \in \{\text{APE}, \text{RoPE}\}$. A model $f \in \mathcal{F}_{L, m, \mathrm{PE}}$ is a deterministic mapping $f: \mathcal{V}^* \to \mathcal{Z}_{\mathrm{all}}$. Let $S_z: \mathcal{Z}_{\mathrm{all}} \to \mathbb{R}$ be a deterministic function, parameterized by the target state of the reference trajectory $z$, that scores the formal validity of a generated output. We define the zero-one loss function $\ell: \mathcal{F}_{L, m, \mathrm{PE}} \times \mathcal{Z} \to \{0, 1\}$ as:
\begin{equation}
    \ell(f, z) = 1 - \mathbb{I}_{(0, \infty)}\left(S_z\left(f(x_{\mathrm{init}}^{(z)})\right)\right)
\end{equation}
For any measure $\mu \in \mathcal{P}(\mathcal{Z})$, the expected risk of $f$ is defined as the Lebesgue integral over the discrete space: $\mathcal{R}_{\mu}(f) = \int_{\mathcal{Z}} \ell(f, z) \, d\mu(z)$.

\subsection{Structural Projection and Pushforward Measures}
\label{subsec:structural_projection}

Distances between token sequences do not reflect the structure of the search, so we map trajectories into a continuous metric space $\mathcal{X} \subset \mathbb{R}^{12}$ equipped with the Euclidean metric $d_{\mathcal{X}}(\boldsymbol{u}, \boldsymbol{v}) = \|\boldsymbol{u} - \boldsymbol{v}\|_2$. As volumetric coordinates (e.g., $\chi_4$) scale with sequence capacity $T_{\max}$, the geometric diameter $\operatorname{diam}(\mathcal{X}) = \Theta(T_{\max})$. Hence, for any finite $T_{\max}$, $\mathcal{X}$ is a bounded, locally compact subset.

\begin{definition}[Structural Projection Operator $\Phi$]
\label{def:structural_projection}
We define a deterministic, Borel measurable projection operator $\Phi: \mathcal{Z} \to \mathcal{X}$, mapping $z \mapsto \boldsymbol{\chi} =[\chi_1, \dots, \chi_{12}]^\top$. 
The coordinates of $\boldsymbol{\chi}$ summarize the search in $z$: the maximum, mean and standard deviation of the tree depth ($\chi_1,\chi_2,\chi_3$); the search volume $\chi_4$, i.e., the number of node references in the trace; the fractions of consecutive node references that go deeper, go back up, or stay at the same depth of the tree (deepening rate $\chi_5$, backtracking rate $\chi_6$, lateral transition rate $\chi_7$); the prune count and pruning rate ($\chi_8,\chi_9$); whether the search reaches the target ($\chi_{10}$); and the counts of multiplicative and additive operators ($\chi_{11},\chi_{12}$). Formal definitions of all twelve coordinates are given in Appendix~\ref{app:trajectory-space}. %
\end{definition}

Pushing $\mu_S$ and $\mu_T$ forward through $\Phi$ gives the \textit{structural pushforward measures} $\nu_S := \Phi_\# \mu_S$ and $\nu_T := \Phi_\# \mu_T$ in $\mathcal{P}(\mathcal{X})$. The structural measure shift is quantified by the Wasserstein-1 (Kantorovich--Rubinstein) distance over the coupled metric space \citep{Villani2009}: %

\begin{equation}
    W_1(\nu_S, \nu_T) = \inf_{\pi \in \Pi(\nu_S, \nu_T)} \int_{\mathcal{X} \times \mathcal{X}} d_{\mathcal{X}}(\boldsymbol{u}, \boldsymbol{v}) \, d\pi(\boldsymbol{u}, \boldsymbol{v})
\end{equation}
where $\Pi(\nu_S, \nu_T)$ denotes the set of all joint probability measures on $\mathcal{X} \times \mathcal{X}$ with marginals $\nu_S$ and $\nu_T$. Intuitively, $W_1$ is the smallest average distance by which the mass of $\nu_S$ must be moved to obtain $\nu_T$. Because $\mathcal{X}$ is bounded for finite $T_{\max}$, the integral converges absolutely and $W_1 \le \mathcal{O}(T_{\max})$.

\textbf{Why a structural projection.} %
Token-level likelihoods cannot express how far apart two search strategies are: a shallow breadth-first trace and a deep backtracking trace may have similar likelihood yet differ entirely in how they route through the tree. Mapping each trajectory to its structural statistics turns every reasoning task into a probability measure on $\mathcal{X}$, so the task shift becomes a single transport cost $W_1(\nu_S,\nu_T)$, while the architecture enters only through the Lipschitz constant of its risk function over $\mathcal{X}$ (Lemma~\ref{lemma:error_decomposition}). The rest of the paper builds on this separation of data shift from architectural sensitivity.

\subsection{Error Decomposition via Kantorovich Duality}

We relate the continuous domain shift $W_1(\nu_S, \nu_T)$ to the discrete expected risk $\mathcal{R}_\mu(f)$ through a sufficiency assumption under which, up to $\epsilon$, this risk depends on $\mu$ only through $\Phi_\# \mu$.

\begin{assumption}[$\epsilon$-Relaxed Structural Sufficiency]
\label{assum:epsilon_relaxed}
Let $\epsilon > 0$ be an irreducible approximation error. We assume there exists a measurable risk function $\mathcal{L}_f: \mathcal{X} \to [0, 1]$. For any trajectory measure $\mu \in \{\mu_S, \mu_T\}$, the absolute difference between the true expected risk on $\mathcal{Z}$ and the expected risk integrated over $\Phi_\# \mu$ on $\mathcal{X}$ is bounded by $\epsilon$:
\begin{equation}
    \left| \mathcal{R}_{\mu}(f) - \int_{\mathcal{X}} \mathcal{L}_f(\boldsymbol{\chi}) \, d(\Phi_\# \mu)(\boldsymbol{\chi}) \right| \le \epsilon
\end{equation}
\end{assumption}

This assumption states that the structural projection $\Phi(z)$ captures the task-relevant components of the search trajectory up to a residual $\epsilon$. It targets formal combinatorial planning and algorithmic search, where the structure of a problem's search (e.g., tree depth and backtracking) summarizes its difficulty (Appendix~\ref{app:empirical_sufficiency}). The out-of-distribution generalization gap is then bounded by the Lipschitz continuity of the neural architecture (Lemma~\ref{lemma:error_decomposition}). Appendix~\ref{app:gbc_check} checks the assumption on two trained models that differ in scale, training measure and positional encoding: a gradient boosting classifier on $\boldsymbol{\chi}$ alone predicts whether the model solves a problem with out-of-fold ROC-AUC $0.93$ and McFadden pseudo-$R^2$ $0.48$--$0.51$, against a pseudo-$R^2$ of $0.44$--$0.46$ for token length alone and $0.40$--$0.41$ for the problem alone, and its ROC-AUC exceeds that of both baselines in every fold. %

\begin{lemma}[$\epsilon$-Relaxed Error Decomposition]
\label{lemma:error_decomposition}
Assume the induced risk function $\mathcal{L}_f(\boldsymbol{\chi})$ is Lipschitz continuous on $(\mathcal{X}, d_{\mathcal{X}})$ with minimal Lipschitz constant $K_f = \sup_{\boldsymbol{u} \neq \boldsymbol{v}} |\mathcal{L}_f(\boldsymbol{u}) - \mathcal{L}_f(\boldsymbol{v})| / d_{\mathcal{X}}(\boldsymbol{u}, \boldsymbol{v})$. Then the generalization error gap is bounded by:
\begin{equation}
    |\mathcal{R}_{\mu_T}(f) - \mathcal{R}_{\mu_S}(f)| \le K_f \cdot W_1(\nu_S, \nu_T) + 2\epsilon
\end{equation}
\end{lemma}

\begin{proof}[Proof Sketch]
The triangle inequality and Assumption~\ref{assum:epsilon_relaxed} bound the risk difference $|\Delta \mathcal{R}|$ by $|\int \mathcal{L}_f d\nu_T - \int \mathcal{L}_f d\nu_S| + 2\epsilon$. Normalizing the structural risk function as $g = \mathcal{L}_f / K_f$ yields $\|g\|_{\mathrm{Lip}} \le 1$. By the Kantorovich--Rubinstein duality \citep{Villani2009}, the normalized integral difference $|\int g \, d\nu_T - \int g \, d\nu_S|$ is bounded by $W_1(\nu_S, \nu_T)$. Scaling the duality bound by $K_f$ yields the claim.
\end{proof}

\section{Theoretical Limits of Architectural Choices}
\label{sec:theoretical_limits}

By Lemma~\ref{lemma:error_decomposition}, the out-of-distribution generalization gap is bounded by the Lipschitz constant $K_f$ and the approximation bounds of $\mathcal{L}_f$. We now formalize how architectural priors shape these bounds along two computational axes of the autoregressive Transformer: the positional encoding scheme (APE vs.\ RoPE) and the aspect ratio of the computational graph (Deep vs.\ Wide).

\subsection{Sequence Equivariance and the Lipschitz Gap}

Algorithmic search trajectories are modular: the same local sub-trajectory (e.g., a subtree traversal) can occur at arbitrary absolute sequence indices, depending on how much of the tree was searched before it. To quantify a model's sensitivity to these shifts, we examine how its attention kernel behaves under sequence translations.

\begin{definition}[Minimal Sequence Translation Operator $\mathcal{T}_k$]
\label{def:translation_operator}
Let $z \in \mathcal{Z}$. For $k \in \mathbb{N}^+$, the translation operator $\mathcal{T}_k: \mathcal{Z} \to \mathcal{Z}$ inserts one structurally inert node of exactly $k$ tokens immediately before a local sub-trajectory $\omega$ of $z$.

Under $\Phi$, this increments the search volume ($\chi_4'=\chi_4+1$), while every token in the subsequent $\omega$ is shifted by $+k$ in absolute position. Hence the distance in $\mathcal{X}$ is
\begin{equation}
    d_{\mathcal X}\!\left(\Phi(z),\Phi(\mathcal T_k(z))\right) = \left(1^2+\sum_{m\neq 4}(\Delta\chi_m)^2\right)^{1/2} =: \delta \le C_\delta
\end{equation}
where $C_\delta$ is independent of $T_{\max}$.
\end{definition}

For a model to generalize across diverse tree structures, its local computational graph must remain invariant under $\mathcal{T}_k$. We evaluate the pre-softmax attention score $A_{a,b}$ between two tokens in $\omega$ originally at indices $a$ and $b$, now shifted to $a+k$ and $b+k$.

\begin{lemma}[APE Absolute Coordinate Sensitivity]
\label{lemma:ape_shattering}
Assume the network parameters (projection matrices $\mathbf{W}_Q, \mathbf{W}_K$ and absolute embeddings $\mathbf{p}_t$) are generically parameterized, and the projection matrices satisfy $\mathbf{W}_Q^\top \mathbf{W}_K \neq \mathbf{0}$. Under the operator $\mathcal{T}_k$ ($k>0$), the local computational attention kernel of an APE model differs almost surely: $\Delta A^{\mathrm{APE}} = |\tilde{A}_{a, b}^{\mathrm{APE}} - A_{a, b}^{\mathrm{APE}}| > 0$.
\end{lemma}

\begin{proof}[Proof Sketch] 
Expanding the bilinear forms of the APE attention kernel cancels the position-independent terms. The residual difference $\Delta A^{\mathrm{APE}}$ is a nontrivial multivariate quadratic polynomial, whose zero-locus is a proper algebraic variety with Lebesgue measure zero in $\mathbb{R}^{2d}$. Consequently, generic parameters avoid this set, so exact cancellation does not occur. RoPE instead uses orthogonal rotation matrices, so $\Delta A^{\mathrm{RoPE}} = 0$. See Appendix \ref{app:pe-results}.
\end{proof}

In pretrained models, under a shift of all position indices by $16$, $64$ or $256$, the largest change of a first-layer attention score is $0.42$--$1.19$ of the largest score in four APE models and at most $4.4\times10^{-12}$ (float64, all layers) in four RoPE models (Appendix~\ref{app:c1-kernel}).

\begin{assumption}[Autoregressive Sensitivity via Markov Compounding]
\label{assum:ar_sensitivity}
We assume the search-based reasoning task has an intrinsic global smoothness, with a baseline structural Lipschitz constant $\mathcal{O}(1/T_{\max})$. Generating a formal search trajectory is a Markov decision process over a directed acyclic graph. We assume that a local coordinate disruption of the attention kernel ($\Delta A > 0$) at a critical routing step sends the autoregressive process into an incorrect subtree. Under exact-match evaluation, the process then cannot reach the target state, so the local error compounds into the failure of the whole trajectory. The structural risk therefore varies by at least the uniform error probability: $\Delta \mathcal{L} = |\mathcal{L}_f(\boldsymbol{\chi}') - \mathcal{L}_f(\boldsymbol{\chi})| \ge 1 - 1/|\mathcal{V}| =: \gamma > 0$.
\end{assumption}

\begin{theorem}[The Lipschitz Gap]
\label{thm:lipschitz_gap}
Under Assumption \ref{assum:ar_sensitivity}, for a sufficiently large sequence capacity $T_{\max}$, the Lipschitz constant of the induced risk function for APE is asymptotically larger than that of RoPE: $K_{f, \mathrm{RoPE}} = \mathcal{O}(1/T_{\max}) \ll \Omega(1) \le K_{f, \mathrm{APE}}$.
\end{theorem}

\begin{proof}[Proof Sketch] 
The Lipschitz constant is bounded below by the difference quotient at the perturbed pair: $K_{f, \mathrm{APE}} \ge |\Delta \mathcal{L}| / \delta \ge \gamma / C_\delta$. Because $\gamma$ and $C_\delta$ are positive constants independent of $T_{\max}$, it follows that $K_{f, \mathrm{APE}} = \Omega(1)$. RoPE preserves relative sequence distances, so it keeps the local dependency structure free of artificial discontinuities and inherits the task's intrinsic baseline smoothness $\mathcal{O}(1/T_{\max})$.
\end{proof}

\textbf{Causal chain from positional encoding to risk.} %
Because $\mathcal{L}_f$ is indexed by the model $f$, the positional encoding shapes the risk function over $\mathcal{X}$; $K_f$ measures how much this risk can change under a small structural shift $\delta$. Under $\mathcal{T}_k$, which moves $\boldsymbol{\chi}$ by $\delta \le C_\delta$ and every token of $\omega$ by $k$ positions, the APE attention kernel changes (Lemma~\ref{lemma:ape_shattering}), and in exact-match search this local change diverts greedy routing into a wrong subtree (Assumption~\ref{assum:ar_sensitivity}). A shift of $\delta$ in $\mathcal{X}$ thus produces a jump $\gamma$ in risk, i.e., $K_{f,\mathrm{APE}}=\Omega(1)$. Through Lemma~\ref{lemma:error_decomposition}, the guarantee scales with $K_f W_1$, and RoPE keeps $K_f$ at $\mathcal{O}(1/T_{\max})$.

\subsection{Capacity Limits: The \texorpdfstring{$\text{TC}^0$}{TC0} Algorithmic Bottleneck}

Different traversal patterns of a search also place different computational demands on the depth $L$ and width $m$ of a model in $\mathcal{F}_{L, m, \mathrm{PE}}$. 

In breadth-first (lateral) expansions, the states are mutually independent, so a wide network can evaluate them in parallel across its attention heads. By contrast, deep recursive backtracking requires a last-in-first-out (LIFO) stack mechanism. To backtrack, the model must identify and resolve the most recently suspended, non-adjacent sibling node across a flattened sequence history. The balanced strings over $k$ bracket types, the Dyck-$k$ language $\mathcal{D}_k$, follow the same LIFO order.

\begin{lemma}[Surjective Monoid Homomorphism to Dyck-$k$]
\label{lemma:dyck_homomorphism}
Let $k$ denote the maximum branching factor of the implicit search tree. Defining the traversal alphabet $\Sigma_{\mathrm{search}} = \{ \delta_{\mathrm{down}}^{(1)}, \dots, \delta_{\mathrm{down}}^{(k)} \} \cup \{ \delta_{\mathrm{up}}^{(1)}, \dots, \delta_{\mathrm{up}}^{(k)} \} \cup \{ \omega_{\mathrm{eval}} \}$, there exists a surjective monoid homomorphism $\psi: \Sigma_{\mathrm{search}}^* \to \Sigma_{\mathrm{Dyck}}^*$ mapping child expansions to left parentheses $(_c$ and backtracking steps to right parentheses $)_c$. Every valid search trajectory projects to a prefix of a word in $\mathcal{D}_k$, where the maximum tree depth $\chi_1$ determines the maximum parenthesis nesting depth.
\end{lemma}

A Transformer computing $L$ sequential layers under bounded numerical precision belongs to $\text{TC}^0$, the class of constant-depth, polynomial-size threshold circuits \citep{merrill2023parallelismtradeofflimitationslogprecision}, and self-attention cannot model hierarchical structure unless the number of layers or heads grows with the input length \citep{Hahn_2020}. We therefore assume that resolving deeper nesting within one forward pass requires more layers (Assumption~\ref{ass:implicit-complexity}, Appendix~\ref{app:depth-complexity}). We formalize this limitation as two independent capacity bounds on the expected risk. %

Let $\mathcal{L}^*(\boldsymbol{\chi})
=
\inf_{h}
\mathbb{E}_{z \sim \mu_T}
[\ell(h,z)\mid \Phi(z)=\boldsymbol{\chi}]$
denote the Bayes-optimal structural risk function. 
Lemma~\ref{lemma:width_capacity} transfers the $L_1(\nu_T)$ approximation error in $\mathcal{H}_{\mathrm{Barron}}$ to the discrete zero-one risk via the Bayes optimality of $\mathcal{L}^*$.

\begin{lemma}[Width Capacity Bound via Barron Spaces]
\label{lemma:width_capacity}
Under Assumptions~\ref{assum:epsilon_relaxed} and~\ref{ass:width-capacity} (in particular $\mathcal{L}^* \in \mathcal{H}_{\mathrm{Barron}}$), any autoregressive Transformer hypothesis bounded by an internal representation width $m$ satisfies the lower bound: $\mathcal{R}_{\mu_T}(f) \ge \frac{C_{\mathrm{width}}}{\sqrt{m}} - \epsilon$.
\end{lemma}

\begin{proof}[Proof Sketch] 
Routing recursive transitions correctly from token-level embeddings requires the internal representations to track global tree states, so the hidden width $m$ (which also fixes the FFN width) bounds the capacity of $\mathcal{L}_f$. Since $\mathcal{L}^*\ge 0$ is the pointwise Bayes risk, $\int\mathcal{L}^*\,d\nu_T\le\mathcal{R}_{\mu_T}(f)$; together with Assumption~\ref{assum:epsilon_relaxed} and $|\mathcal{L}_f-\mathcal{L}^*|\le\mathcal{L}_f+\mathcal{L}^*$, this gives $2\mathcal{R}_{\mu_T}(f)\ge\int|\mathcal{L}_f-\mathcal{L}^*|\,d\nu_T-\epsilon$. Classical Barron-space results supply the upper rate: some width-$m$ network attains an $\mathcal{O}(m^{-1/2})$ error \citep{Barron1993, MAKOVOZ199698}. The lemma needs the matching lower bound, which Assumption~\ref{ass:width-capacity}(iii) (Appendix~\ref{app:structural-capacity}) states as a minimax-hardness condition: no width-$m$ hypothesis approximates $\mathcal{L}^*$ faster than order $m^{-1/2}$. With the inequality above, this yields the claim. %
\end{proof}

Autoregressive generation extends the trace one step at a time, but the routing computation at each step happens inside the model and is not serialized in the output. Resolving the non-adjacent sibling dependencies of a backtracking step therefore relies only on the fixed depth $L$ of one forward pass.

Motivated by these depth limits, we introduce a task-dependent truncation constant $\alpha_{\mathrm{circ}}>0$ and model the recoverable structural depth of a depth-$L$ architecture by the threshold $\chi_1 \le \alpha_{\mathrm{circ}}L$.

\begin{lemma}[Implicit Depth Bottleneck Bound]
\label{lemma:depth_bottleneck}
Under Assumptions~\ref{assum:epsilon_relaxed} and~\ref{ass:implicit-complexity}, defining the Borel measurable indicator function $P_{\mathrm{implicit}}(\boldsymbol{\chi}; L, \alpha_{\mathrm{circ}}) = \mathbb{I}_{(\alpha_{\mathrm{circ}} L, \infty)}(\chi_1)$, the expected target risk is bounded below by: $\mathcal{R}_{\mu_T}(f) \ge C_{\mathrm{depth}} \int_{\mathcal{X}} P_{\mathrm{implicit}} \, d\nu_T - \epsilon$, where $C_{\mathrm{depth}} = 1 - \frac{1}{|\mathcal{V}|} \in (0, 1)$.
\end{lemma}

\section{The Main Bound}
\label{sec:unified_bound}

\begin{theorem}[The Generalization Bound]
\label{thm:unified_bound}
Under Assumptions~\ref{assum:epsilon_relaxed}, \ref{ass:width-capacity} and~\ref{ass:implicit-complexity} and a Lipschitz $\mathcal{L}_f$, for any autoregressive Transformer $f \in \mathcal{F}_{L, m, \mathrm{PE}}$, the target risk $\mathcal{R}_{\mu_T}(f)$ is bounded below by the maximum of three independent lower bounds, and from above by its source risk and domain shift penalty:

{\small
\begin{align}
    \mathcal{R}_{\mu_T}(f) &\ge \max\Bigl\{ \frac{C_{\mathrm{width}}}{\sqrt{m}} - \epsilon,\; C_{\mathrm{depth}} \int_{\mathcal{X}} P_{\mathrm{implicit}} \, d\nu_T - \epsilon,\; \mathcal{R}_{\mu_S}(f) - K_{f, \mathrm{PE}} \, W_1(\nu_S, \nu_T) - 2\epsilon \Bigr\} \tag*{\normalsize(LB)} \label{eq:lower_bound} \\
    \mathcal{R}_{\mu_T}(f) &\le \mathcal{R}_{\mu_S}(f) + K_{f, \mathrm{PE}} \, W_1(\nu_S, \nu_T) + 2\epsilon \tag*{\normalsize(UB)} \label{eq:upper_bound}
\end{align}}
\end{theorem}

\begin{proof}[Proof Sketch]
Kantorovich duality applied to the $\epsilon$-relaxed structural risk (Lemma~\ref{lemma:error_decomposition}) gives the upper bound and the measure-shift term of (LB); the Barron-space width bound (Lemma~\ref{lemma:width_capacity}) and the implicit-depth bound (Lemma~\ref{lemma:depth_bottleneck}) give the other two terms. See Appendix \ref{app:main-results}.
\end{proof}

\textbf{Non-substitutable lower bounds.} %
Each term of (LB) is a lower bound from a different mechanism (width, depth, measure shift), so their maximum is also a lower bound, whereas their sum could exceed the trivial bound $\mathcal{R}\le 1$. Every term thus gives a necessary condition, and meeting one cannot compensate for violating another: if a target tree has depth $\chi_1=20$ but $L=6$, the depth term remains $\Omega(1)$ however large $m$ is, even if RoPE removes the shift penalty (Appendix~\ref{app:intuition}).

\begin{corollary}[The Divergent OOD Guarantees of APE vs.\ RoPE]
\label{cor:ape_vs_rope}
Let $\mathcal{R}_{\mu_S}(f) = \epsilon_{\mathrm{train}} \approx 0$ and $W_1(\nu_S, \nu_T) = \Omega(T_{\max})$. RoPE preserves equivariance: with $K_{f,\mathrm{RoPE}}\le C_{\mathrm{smooth}}/T_{\max}$ (Theorem~\ref{thm:lipschitz_gap}), its expected generalization error degrades at most proportionally to the normalized shift, $\mathcal{R}_{\mu_T}(f_{\mathrm{RoPE}}) \le \epsilon_{\mathrm{train}} + C_{\mathrm{smooth}} \cdot (W_1 / T_{\max}) + 2\epsilon$. APE, however, suffers coordinate disruption and incurs an $\Omega(1)$ expected risk: $\mathcal{R}_{\mu_T}(f_{\mathrm{APE}}) \ge \gamma \cdot p_{\mathrm{shift}} - \epsilon$, where $p_{\mathrm{shift}}>0$ is the target mass of structures on which its routing fails ($\mathcal{L}_{f_{\mathrm{APE}}}\ge\gamma$; Assumption~\ref{assum:ar_sensitivity}).
\end{corollary}

\begin{corollary}[The ``Deep is Better'' Law for Search Reasoning]
\label{cor:deep_is_better}
Let the target measure $\nu_T$ place probability mass on trajectories satisfying the tree depth threshold $\chi_1 > \alpha_{\mathrm{circ}} L$. Scaling the discrete physical depth to $L \ge \alpha_{\mathrm{circ}}^{-1} \sup_{\boldsymbol{\chi}} \chi_1$ eliminates the $\Omega(1)$ lower bound. Scaling the representation width $m$ cannot: it only removes the width term, since $\lim_{m \to \infty} (C_{\mathrm{width}} / \sqrt{m} - \epsilon) = -\epsilon \le 0$, and the depth term does not depend on $m$.
\end{corollary}

\section{Empirical Verification}
\label{sec:experiments}

To test the architectural predictions of Theorem \ref{thm:unified_bound}, we compare architectures that differ in layer depth $L$, width $m$ and positional encoding at a nearly constant core parameter count $12Ld^2$ per scale. Motivated by depth-separation theorems \citep{eldan2016powerdepthfeedforwardneural, levine2021depthtowidthinterplayselfattention} and $\text{TC}^0$ bounds \citep{merrill2023parallelismtradeofflimitationslogprecision}, we design a grid of nine Transformer architectures crossing three parameter scales (Small, Medium, Large) and three aspect ratios (Deep, Wide, Balanced).

To vary the Lipschitz constant $K_{f, \mathrm{PE}}$, we cross every architecture with APE and RoPE \citep{su2023roformerenhancedtransformerrotary}, using the GPT-Neo and GPT-NeoX implementations, respectively (Appendix~\ref{app:implementations}). APE imposes an absolute coordinate system that translations disrupt, whereas RoPE is equivariant to translations and preserves the structural regularities of the task. Table \ref{tab:models} summarizes these models.

To instantiate the grammar $\Gamma$ and probability measures ($\mu_S, \mu_T$), we use the Stream of Search (SoS) framework \citep{gandhi2024streamsearchsoslearning} as a discrete trajectory simulator. SoS lets us fix the search procedure that generates each measure and thus provides a controlled testbed.
We generate 500,000 ground-truth trajectories per measure on a combinatorial planning task (the Countdown game), in which a set of input numbers must be combined with arithmetic operations to reach a target (e.g., reaching $10$ from $[2,3,5]$); SoS serializes the resulting search tree, including dead ends and backtracking, into one token sequence (worked example in Appendix~\ref{app:countdown_example}). Exact-match verification on this formal task isolates structural routing errors from semantic noise, as required to test the depth and Lipschitz terms of Theorem~\ref{thm:unified_bound}. %
The measures differ in the search strategy of the simulator. Queue-based state expansion gives the lateral traversal traces of $\mu_{BFS}$, and recursive stack expansion gives $\mu_{DFS}$, whose traces have nearly four times the backtracking rate $\chi_6$ of $\mu_{BFS}$. The mixture $\mu_{MIXED} = 0.5 \mu_{BFS} + 0.5 \mu_{DFS}$ (with the search heuristic of each component drawn at random) tests architectures under large position shifts. Under exact-match scoring, the zero-one loss credits only outputs that reach the target with valid arithmetic. 

\begin{table}[t]
    \centering
    \caption{Summary of the 54 Transformer architectural configurations evaluated.}
    \label{tab:models}
    \resizebox{\columnwidth}{!}{
    \begin{tabular}{l c c c | r r}
        \toprule
        \textbf{Scale / Shape} & \textbf{Layers ($L$)} & \textbf{Width ($d=m$)} & \textbf{Heads ($h$)} & \textbf{Core Params$^\dagger$} & \textbf{Total Params$^\ddagger$} \\
        \midrule
        Small Deep & 12 & 192 & 4 & 5.3 M & 24.8 M \\
        Small Balanced & 6 & 256 & 4 & 4.7 M & 30.7 M \\
        Small Wide & 3 & 384 & 4 & 5.3 M & 44.3 M \\
        \midrule
        Medium Deep & 24 & 512 & 8 & 75.5 M & 127.6 M \\
        Medium Balanced & 12 & 768 & 12 & 84.9 M & 163.0 M \\
        Medium Wide & 6 & 1024 & 16 & 75.5 M & 179.5 M \\
        \midrule
        Large Deep & 48 & 768 & 12 & 339.7 M & 418.1 M \\
        Large Balanced$^*$ & 24 & 1024 & 16 & 302.0 M & 406.2 M \\
        Large Wide & 12 & 1536 & 24 & 339.7 M & 495.9 M \\
        \bottomrule
        \multicolumn{6}{l}{\small \textit{Note: All architectures crossed with PE $\in \{\text{APE, RoPE}\}$ and Training Data $\in \{\text{BFS, DFS, MIXED}\}$.}} \\
        \multicolumn{6}{l}{\small $^\dagger$\textit{Core Params = $12 L d^2$, isolating the architectural capacity independent of vocabulary.}} \\
        \multicolumn{6}{l}{\small $^\ddagger$\textit{Total Params $=12Ld^2+10Ld+2d+2Vd+1024d$ ($V=50{,}257$), i.e., GPT-Neo with untied}} \\
        \multicolumn{6}{l}{\small \phantom{$^\ddagger$}\textit{embeddings, an upper bound (APE models tie embeddings; RoPE models have no position table).}} \\
        \multicolumn{6}{l}{\small $^*$\textit{Adjusted $d=1024$ to match the computational capacity of its scale group.}}
    \end{tabular}
    }
\end{table}

\begin{figure}[t]
    \centering
    \includegraphics[width=\textwidth]{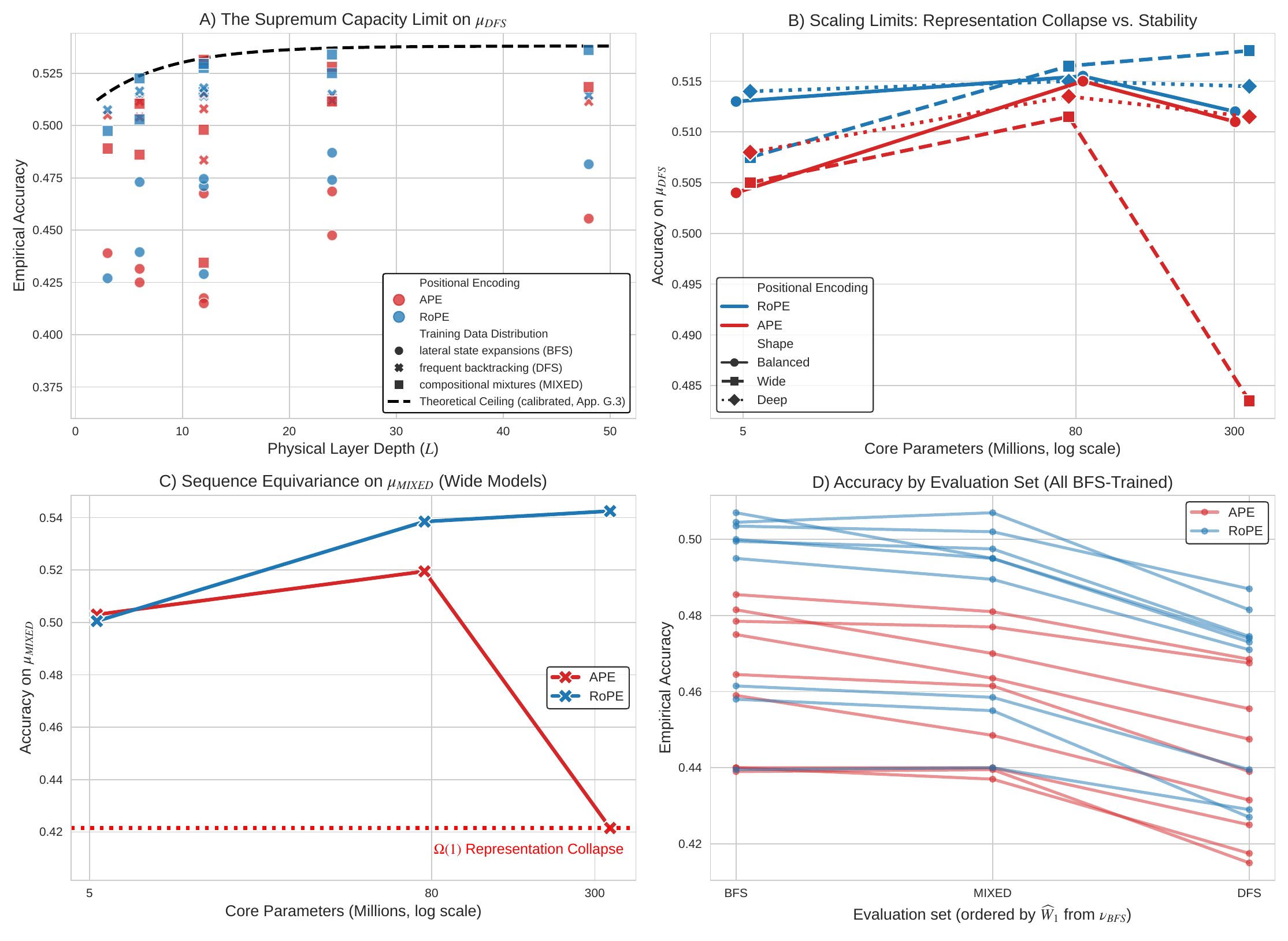}
    \caption{Validation of capacity bounds. (A) Supremum capacity limit on $\mu_{DFS}$: all accuracies lie below the calibrated upper bound curve (Appendix~\ref{app:empirical_calibration}). (B) Scaling limits: width plateaus and, under APE, collapses at Large scale (Wide $0.4835$ vs.\ Deep $0.5115$); with RoPE, Wide matches Deep. (C) Sequence equivariance on $\mu_{MIXED}$, where scaling Wide APE architectures triggers representation collapse, which RoPE prevents. (D) All BFS-trained models on the three evaluation sets, ordered by Wasserstein distance from $\nu_{BFS}$; RoPE outperforms APE on average on each set (in 23 of 27 comparisons). The sets share one problem distribution, so the decline toward $\mu_{DFS}$ reflects the difficulty of their problem samples (Appendix~\ref{app:eval_scope}).}
    \label{fig:analytical_grid}
\end{figure}

\textbf{Structural optimal transport estimation.}
To measure the transport cost in the Kantorovich bounds, we estimate the Wasserstein-1 distance between structural pushforward measures by solving the exact (unregularized) optimal transport problem between empirical measures with POT \citep{flamary2021pot}. We construct empirical structural measures on $\mathcal{X}$ from $10$ independent subsamples of $N=2{,}000$ uniformly sampled trajectories per distribution. After joint Z-score standardization, we obtain $\widehat{W}_{1}(\hat{\nu}_{BFS},\hat{\nu}_{MIXED}) = 1.38 \pm 0.05$ and $\widehat{W}_{1}(\hat{\nu}_{BFS},\hat{\nu}_{DFS}) = 2.97 \pm 0.05$, against a within-measure sampling baseline of $0.19 \pm 0.06$; both shifts exceed this baseline at least sevenfold, and the mixture sits at $46\%$ of the BFS-to-DFS shift. Estimation procedures, cost-matrix construction, normalization details, and the variation across repeated subsamples are reported in Appendix~\ref{app:w1_estimation}.

\subsection{Positional Encodings as Structural Regularizers}

We first examine the Lipschitz gap of Theorem \ref{thm:lipschitz_gap}. Our bounds give no generalization guarantee for APE under sequence translations. When trained on the $\mu_{MIXED}$ distribution, the model alternates between BFS- and DFS-style traversals, subjecting the attention kernel to large absolute sequence index shifts. The Large Wide MIXED APE model reaches only $0.4215$ accuracy on the MIXED evaluation set, which we attribute to the coordinate disruption of Lemma~\ref{lemma:ape_shattering}. 

In contrast, preserving local sequence equivariance lets the Large Wide MIXED RoPE model reach $0.5425$, over $12$ points higher. The two models share depth $L=12$ and width $m=1536$ and differ mainly in the positional encoding (Appendix~\ref{app:implementations}). %
The depth and width terms are the same for both and saturate at Large scale: on each evaluation set, the best accuracies at $L=12$, $24$ and $48$ lie within one point (Appendix~\ref{app:intuition}). Across the full grid, RoPE outperforms APE in 71 of 81 paired comparisons (same scale, shape, training measure and evaluation set) and in 25 of 27 model pairs averaged over the evaluation sets (one-sided sign test $p<10^{-5}$). The advantage grows with scale: the mean gap is $+0.6$, $+1.2$ and $+3.3$ points for Small, Medium and Large, and RoPE wins all 27 comparisons at Large scale. Figure~\ref{fig:analytical_grid}D shows the same ordering, on average, on every evaluation set for the BFS-trained models. This is consistent with the Lipschitz gap of Theorem~\ref{thm:lipschitz_gap}, under which RoPE is both a length extrapolator and a regularizer over the continuous metric space. We test Lemma~\ref{lemma:ape_shattering} and Assumption~\ref{assum:ar_sensitivity} on two of our models by shifting all position indices by $16$ or $64$ (Appendix~\ref{app:g1}): the shift changes none of the $94{,}278$ routing predictions of the Small Deep MIXED RoPE model on reference trajectories, and $25.7\%$ and $28.3\%$, respectively, of those of the Large Deep BFS APE model; when the APE model's own generation resumes with shifted indices, its success falls from $0.527$ by $0.40$ and $0.39$, while no outcome of the RoPE model changes.

\subsection{Depth versus Width Scaling and \texorpdfstring{$\text{TC}^0$}{TC0} Bottlenecks}

Our bounds imply an asymmetry between depth and width scaling: overcoming the $\text{TC}^0$ complexity bottleneck requires increasing the physical layer depth $L$ (Lemma \ref{lemma:depth_bottleneck}), whereas widening the representation dimension $m$ only shrinks the Barron-space term of order $m^{-1/2}$ (Lemma \ref{lemma:width_capacity}), so its gain plateaus. Under APE, Deep beats Wide in all 18 Medium- and Large-scale comparisons; at equal core parameters, the Large Deep MIXED APE model ($0.5330$) exceeds its Wide counterpart ($0.4215$) by $11.15$ points. Twelve layers suffice on this task: at $L=12$, Wide RoPE reaches $0.5425$ and the half-width Medium Balanced MIXED APE model reaches $0.5360$, so the extra depth of Deep APE compensates for a collapse specific to Wide APE (Appendix~\ref{app:Supp Fig}).
With RoPE, this collapse disappears and Wide matches Deep (Figure~\ref{fig:analytical_grid}B): at Large scale the gap stays below one point in all nine comparisons, against $2.6$--$11.25$ points under APE. This indicates that translation equivariance 
allows the architecture to approach the capacity set by its depth and width without the coordinate disruption that limits APE models (Appendix~\ref{app:intuition}).

\subsection{The Supremum Bound}

Finally, we validate the unified bound (Theorem \ref{thm:unified_bound}): its lower bound (LB) is the supremum of the width bottleneck, the depth bottleneck and the optimal transport term, so the largest of the three sets the risk floor. In Figure \ref{fig:analytical_grid}A, all models lie below the calibrated $\text{TC}^0$ ceiling, and doubling the width at fixed depth changes the best accuracy on $\mu_{DFS}$ (over the six models of each architecture) by under one point: from $0.5315$ to $0.5295$ at $L=12$, where widening from $d=768$ to $d=1536$ quadruples the core parameters ($84.9\,\mathrm{M}$ to $339.7\,\mathrm{M}$), and from $0.5280$ to $0.5340$ at $L=24$ ($d=512$ to $d=1024$; Appendix~\ref{app:empirical_calibration}). Width therefore does not raise this ceiling. The asymptote of the ceiling also absorbs failure sources outside the depth term, among them the generation budget: all models decode greedily with at most $512$ new tokens on the same $2{,}000$ problems per evaluation set (Appendix~\ref{app:reproducibility}), and within this budget the reference searches themselves reach the target on $0.4645$--$0.4915$ of these problems, rates that most models exceed (Appendix~\ref{app:empirical_calibration}). 

\newpage
\section{Related Work}
\label{sec:related_work}

\textbf{Circuit complexity and Transformer expressivity.} 
Bounded-precision Transformers are restricted to $\text{TC}^0$ \citep{merrill2023parallelismtradeofflimitationslogprecision}, and self-attention cannot model hierarchical structure such as Dyck languages unless the number of layers or heads grows with the input length \citep{Hahn_2020}. %
 Depth-separation theorems further prove that bounded depth cannot be efficiently compensated by width \citep{eldan2016powerdepthfeedforwardneural, levine2021depthtowidthinterplayselfattention}. These works give distribution-agnostic worst-case limits; our framework combines $\text{TC}^0$ bottlenecks with optimal transport and Barron-space bounds \citep{Barron1993, MAKOVOZ199698} and maps discrete expressivity to expected risk under specific data measures.

\textbf{Positional encodings as structural regularizers.} 
Positional encodings govern OOD length extrapolation \citep{zhao2024lengthextrapolationtransformerssurvey}; designs range from absolute embeddings (APE) through relative schemes such as ALiBi \citep{press2022trainshorttestlong} to no positional encoding (NoPE) \citep{kazemnejad2023impactpositionalencodinglength}. Recent variants, including functional interpolation (FIRE) \citep{li2024functionalinterpolationrelativepositions}, bilevel encodings (BiPE) \citep{he2024stoneshitbirdbilevel} and data-adaptive encodings (DAPE) \citep{zheng2024dapedataadaptivepositionalencoding}, report better length generalization. Beyond theories of decaying attention, we treat shift-invariant encodings as structural regularizers: any relative encoding with $\Delta A = 0$ under $\mathcal{T}_k$ inherits our Lipschitz bound. Theorem~\ref{thm:lipschitz_gap} and Corollary~\ref{cor:ape_vs_rope} show why APE incurs an $\Omega(1)$ risk under structural shifts, whereas equivariant priors bound the Kantorovich penalty.

\textbf{Optimal transport and reasoning domain shifts.} 
Recent work represents reasoning as search-tree traversal \citep{gandhi2024streamsearchsoslearning} and relates its generalization to the shift between training and test distributions \citep{prystawski2023thinkstepstepreasoning, zhao2026chainofthoughtreasoningllmsmirage}, but treats neural architectures as black boxes. We draw on optimal transport (OT) domain adaptation, which bounds target risk via Wasserstein distances over static Euclidean features \citep{courty2016optimaltransportdomainadaptation, redko2017theoreticalanalysisdomainadaptation}, and extend it to \textit{autoregressive sequence generation}. Our pushforward operator $\Phi$ projects variable-length trajectories into a continuous metric space, which lets us couple Kantorovich dual bounds to Transformer circuit complexity and attention kernels. The result, Theorem~\ref{thm:unified_bound}, is the first unified bound that links continuous domain shifts with discrete architectural limits.

\section{Conclusion and Limitations}
\label{sec:conclusion}

We introduced a measure-theoretic framework that quantifies the structural generalization of LLMs on formal search-based reasoning.
It maps discrete reasoning trajectories into a continuous metric space and decomposes the OOD generalization gap via Kantorovich duality. Under the stated assumptions, we proved two architectural results, and our experiments support their consequences: (1) APE suffers coordinate disruptions under sequence translations and thus has an $\Omega(1)$ Lipschitz constant and expected risk, whereas RoPE guarantees relative equivariance; and (2) fixed-depth Transformers incur a risk floor on Dyck-$k$ recursive backtracking that scaling depth ($L$) removes and width ($m$) cannot bypass. 

\textbf{Limitations and future work.} The $\epsilon$-relaxed structural sufficiency assumption targets formal combinatorial search. On Countdown it holds for every pair of our three search measures, with $\epsilon=0.004$ and $0.014$ for the two models of Appendix~\ref{app:gbc_check}; we leave the estimation of $\epsilon$ for open-ended natural-language reasoning to future work. The twelve coordinates of $\Phi$ are one instantiation of the framework, and our three measures share one problem distribution. Other projections are a natural next step, as are target measures that also change the problems posed and so test the $K_f W_1$ term directly (Appendix~\ref{app:eval_scope}). Appendix~\ref{app:ontology} derives the requirements that the bound places on the coordinates of a new projection. Extending $\boldsymbol{\chi}$ to $34$ coordinates on Countdown changes the ROC-AUC of Appendix~\ref{app:gbc_check} by at most $0.003$ and roughly doubles the estimated $W_1$ and its noise floor (Appendix~\ref{app:refine}). Our $\text{TC}^0$ bottleneck concerns the forward-pass routing of a single autoregressive trajectory; how inference-time search such as Tree-of-Thoughts or Monte Carlo Tree Search reshapes the target measure $\nu_T$ is a direction for future theory.

\textbf{Broader impact.} For inference-time scaling and search-based LLMs, which involve large sequence translations ($T_{\max} \to \infty$), our framework implies that translation equivariance, as in RoPE, is a necessary condition for stable scaling. The depth lower bound (Lemma~\ref{lemma:depth_bottleneck}) adds a second necessary condition: a model needs enough layers for the nesting that one forward pass must resolve; splitting a deeply nested task into separate model calls reduces the nesting within each call (Appendix~\ref{app:handoff}).

\begin{ack} %
We thank Qihuang Zhang (Department of Epidemiology, Biostatistics and Occupational Health, McGill University) for supporting this work.

This research was enabled in part by compute resources, software, and technical assistance provided by the Mila cluster.

\textbf{Funding.} Computing resources were provided by Mila -- Quebec AI Institute. This work was supported by the Natural Sciences and Engineering Research Council of Canada (NSERC) through a Discovery Grant (RGPIN-2023-03854). No other third-party funding was received.

\textbf{Competing interests.} The authors declare no competing interests.
\end{ack}

\bibliographystyle{plainnat}
\bibliography{reference}   

\clearpage
\appendix

\section{Formal Setup and Notation}
\label{app:setup}

This section collects the formal objects used throughout the appendix and fixes notation for the structural measure-theoretic formulation of search-based reasoning.

\subsection{Trajectory Space and Structural Feature Space}
\label{app:trajectory-space}

Let $\mathcal{V}$ be a finite vocabulary and let $T_{\max}\in\mathbb{N}^+$ denote the maximum physical sequence length. The ambient space of bounded-length sequences is
\begin{equation}
\mathcal{Z}_{\mathrm{all}} := \bigcup_{t=1}^{T_{\max}} \mathcal{V}^{t}.
\end{equation}
Let $\Gamma$ denote the autoregressive grammar of valid search trajectories. The trajectory space is
\begin{equation}
\mathcal{Z} := \{ z \in \mathcal{Z}_{\mathrm{all}} : z \text{ satisfies } \Gamma \}.
\end{equation}
Since $|\mathcal{V}| < \infty$ and $T_{\max} < \infty$, the set $\mathcal{Z}$ is finite. We equip $\mathcal{Z}$ with the discrete $\sigma$-algebra
\begin{equation}
\Sigma_{\mathcal{Z}} := 2^{\mathcal{Z}}.
\end{equation}
We consider source and target probability measures
\begin{equation}
\mu_S, \mu_T \in \mathcal{P}(\mathcal{Z}),
\end{equation}
where $\mathcal{P}(\mathcal{Z})$ denotes the set of probability measures on $(\mathcal{Z}, \Sigma_{\mathcal{Z}})$.

Each trajectory $z\in\mathcal Z$ is written as
\begin{equation}
z = x_{\mathrm{init}}^{(z)} \oplus y_{\mathrm{search}},
\end{equation}
where $x_{\mathrm{init}}^{(z)}\in\mathcal V^k$ is the initial problem state and $y_{\mathrm{search}}\in\mathcal V^{t-k}$ is the generated search trace.

We next define the structural feature space. Let $\mathcal{X} \subset \mathbb{R}^{12}$ be the structural state space, equipped with the Euclidean metric
\begin{equation}
d_{\mathcal{X}}(\mathbf{u}, \mathbf{v}) := \|\mathbf{u} - \mathbf{v}\|_2.
\end{equation}

We equip $\mathcal{X}$ with the Borel $\sigma$-algebra $\mathcal{B}(\mathcal{X})$ induced by this metric. Because all trajectories have bounded physical length, each structural coordinate is bounded. We therefore take $\mathcal X$ to be a bounded Borel subset of $\mathbb R^{12}$ containing $\Phi(\mathcal Z)$; when compactness is needed, we replace $\mathcal X$ by its closure, which leaves the pushforward measures of Appendix~\ref{app:pushforward} unchanged.

We give the coordinate-level construction of the structural projection
\begin{equation}
\Phi : \mathcal{Z} \to \mathcal{X}, 
\qquad
z \mapsto \boldsymbol{\chi}(z) =
[\chi_1(z),\dots,\chi_{12}(z)]^\top,
\end{equation}
which maps a variable-length discrete trajectory to a fixed-dimensional structural representation.

Let $N$ denote the number of node references in $z$, i.e., the node identifiers in \texttt{"Generated Node"} and \texttt{"Moving to Node"} lines in order of appearance (revisits counted), and let $N_{\mathrm{gen}}$ denote the number of explicitly generated nodes. The definitions below assume $N\ge 1$; a trajectory without node references has $\chi_1=\dots=\chi_9=0$. For the $i$-th node reference, let $\mathrm{id}_i$ denote its sequence identifier and let
\begin{equation}
D_i := |\mathrm{id}_i|
\end{equation}
be its structural tree depth. Let $C(z,\omega)$ denote the number of exact non-overlapping occurrences of a substring $\omega$ in $z$.

The vector $\boldsymbol{\chi}(z)$ is defined through three groups of statistics. These features summarize the search structure of a trajectory.

\paragraph{Tree-structural and volumetric statistics.}
\begin{align}
\chi_1(z) &= \max_{1 \le i \le N} D_i 
&& \text{(maximum depth)}, \\
\chi_2(z) &= \frac{1}{N}\sum_{i=1}^{N} D_i 
&& \text{(mean depth)}, \\
\chi_3(z) &= \left( \frac{1}{N}\sum_{i=1}^{N}(D_i-\chi_2)^2 \right)^{1/2}
&& \text{(depth standard deviation)}, \\
\chi_4(z) &= N
&& \text{(search volume)}.
\end{align}

\paragraph{Depth-transition statistics.}
Let $M := \max(1, N-1)$. We define
\begin{align}
\chi_5(z) &= \frac{1}{M}\sum_{i=1}^{N-1}\mathbb{I}\{D_{i+1} > D_i\}
&& \text{(deepening rate)}, \\
\chi_6(z) &= \frac{1}{M}\sum_{i=1}^{N-1}\mathbb{I}\{D_{i+1} < D_i\}
&& \text{(backtracking rate)}, \\
\chi_7(z) &= \frac{1}{M}\sum_{i=1}^{N-1}\mathbb{I}\{D_{i+1} = D_i\}
&& \text{(lateral transition rate)}.
\end{align}

\paragraph{Pruning and outcome statistics.}
\begin{align}
\chi_8(z) &= C(z,\texttt{"No Solution"})
&& \text{(prune count)}, \\
\chi_9(z) &= \chi_8/N_{\mathrm{gen}}
&& \text{(pruning rate; $0$ if $N_{\mathrm{gen}}=0$)}, \\
\chi_{10}(z) &= \mathbb{I}\{C(z,\texttt{"Goal Reached"}) \ge 1\}
&& \text{(solution indicator)}, \\
\chi_{11}(z) &= C(z,\texttt{"*"}) + C(z,\texttt{"/"})
&& \text{(multiplicative bias)}, \\
\chi_{12}(z) &= C(z,\texttt{"+"}) + C(z,\texttt{"-"})
&& \text{(additive bias)}.
\end{align}

Since $\mathcal{Z}$ is finite and carries the discrete $\sigma$-algebra, every map from $\mathcal{Z}$ into $(\mathcal{X},\mathcal{B}(\mathcal{X}))$ is measurable. In particular, $\Phi$ is Borel measurable, and its image $\Phi(\mathcal Z)$ is finite.

\subsection{Theoretical Justification of Structural Sufficiency (\texorpdfstring{$\epsilon$}{epsilon}-Relaxation)}
\label{app:empirical_sufficiency}

Assumption~\ref{assum:epsilon_relaxed} states that the projection $\Phi(z)$ is a sufficient structural proxy for the expected zero-one risk, up to a residual error $\epsilon$. This relaxation is motivated by the deterministic nature of formal combinatorial planning (e.g., the Countdown task). Appendix~\ref{app:gbc_check} tests the assumption on two trained models, where $\boldsymbol{\chi}$ alone predicts per-problem success with out-of-fold ROC-AUC $0.93$.

Success in algorithmic search tasks is largely determined by the routing strategy of the search. Whether a trajectory reaches a valid solution depends on searching to sufficient tree depth (maximum depth $\chi_1$), recovering from dead ends by backtracking (backtracking rate $\chi_6$), and balancing the use of multiplicative and additive operators (biases $\chi_{11}$ and $\chi_{12}$). 

By design, our structural feature space $\mathcal{X}$ records these three factors. The residual $\epsilon$ collects what $\Phi$ abstracts away, in particular the token identities (i.e., the specific integer values); it is small when the structure of a problem's reference search, as recorded by $\Phi$, predicts whether the model solves that problem (Appendix~\ref{app:gbc_check}). Given Assumption~\ref{assum:epsilon_relaxed} and a Lipschitz risk function over $\mathcal{X}$, the trajectory-level risk gap is then at most $K_f\,W_1(\nu_S,\nu_T)+2\epsilon$ (Lemma~\ref{lemma:error_decomposition}). The $\epsilon$-relaxation thus connects discrete sequence generation to the measure-theoretic analysis on $\mathcal{X}$.

\subsection{Pushforward Measures and Structural Shift}
\label{app:pushforward}

The structural source and target measures are defined as the pushforwards of $\mu_S$ and $\mu_T$ under $\Phi$:
\begin{equation}
\nu_S := \Phi_{\#}\mu_S,
\qquad
\nu_T := \Phi_{\#}\mu_T.
\end{equation}
Since $\Phi$ is Borel measurable, $\nu_S$ and $\nu_T$ are Borel probability measures on $(\mathcal{X},\mathcal{B}(\mathcal{X}))$. Because $\Phi(\mathcal Z)$ is finite and $\mathcal X$ is bounded, both measures also have finite first moment; equivalently, $\nu_S,\nu_T\in\mathcal{P}_1(\mathcal X)$.
By definition, for any Borel set $A \in \mathcal{B}(\mathcal{X})$,
\begin{equation}
\nu_S(A) = \mu_S(\Phi^{-1}(A)),
\qquad
\nu_T(A) = \mu_T(\Phi^{-1}(A)).
\end{equation}
Because $\mathcal{Z}$ is finite, the pushforward measures admit the explicit representations
\begin{align}
\nu_S(A) &= \sum_{z \in \mathcal{Z}} \mathbb{I}_A(\Phi(z))\,\mu_S(\{z\}), \\
\nu_T(A) &= \sum_{z \in \mathcal{Z}} \mathbb{I}_A(\Phi(z))\,\mu_T(\{z\}).
\end{align}

To quantify the structural discrepancy between source and target reasoning behaviors, we use the Wasserstein-1 distance between their pushforward measures
\begin{equation}
W_1(\nu_S,\nu_T)
:=
\inf_{\pi \in \Pi(\nu_S,\nu_T)}
\int_{\mathcal{X}\times\mathcal{X}}
d_{\mathcal{X}}(\mathbf{u},\mathbf{v})\,d\pi(\mathbf{u},\mathbf{v}),
\end{equation}
where $\Pi(\nu_S,\nu_T)$ denotes the set of couplings of $\nu_S$ and $\nu_T$, i.e., the set of probability measures on $\mathcal{X}\times\mathcal{X}$ whose marginals are $\nu_S$ and $\nu_T$ \citep{Villani2009}. Since $\nu_S,\nu_T\in\mathcal{P}_1(\mathcal X)$ and $\mathcal X$ is bounded, the transportation cost is finite for every $\pi \in \Pi(\nu_S,\nu_T)$. Hence $W_1(\nu_S,\nu_T)$ is well-defined and finite.
\subsection{Risk and Structural Risk}
\label{app:risk}

Let $\mathcal{F}_{L,m,\mathrm{PE}}$ denote the class of autoregressive Transformer models with layer depth $L$, hidden dimension $m$, and positional encoding scheme $\mathrm{PE}$. By the trajectory decomposition defined earlier, each $z \in \mathcal{Z}$ induces an initial problem state, denoted by $x_{\mathrm{init}}^{(z)}$. For a model $f \in \mathcal{F}_{L,m,\mathrm{PE}}$, the generated trajectory obtained from the prompt induced by $z$ is
\begin{equation}
f(x_{\mathrm{init}}^{(z)}) \in \mathcal{Z}_{\mathrm{all}}.
\end{equation}

Let $S_z : \mathcal{Z}_{\mathrm{all}} \to \mathbb{R}$ be a deterministic function, parameterized by the target state of the reference trajectory $z$, that scores the formal validity of a generated output. We define the zero-one loss by
\begin{equation}
\ell(f,z) := 1 - \mathbb{I}_{(0, \infty)}\left(S_z\left(f(x_{\mathrm{init}}^{(z)})\right)\right).
\end{equation}
For any probability measure $\mu \in \mathcal{P}(\mathcal{Z})$, the expected risk of $f$ under $\mu$ is
\begin{equation}
\mathcal{R}_{\mu}(f) := \mathbb{E}_{z\sim\mu}[\ell(f,z)].
\end{equation}

In the main text, we relate the trajectory-level risk $\mathcal{R}_{\mu}(f)$ to a structural risk functional defined on $\mathcal{X}$ through the pushforward measure $\Phi_{\#}\mu$. 

\section{Proofs of Structural Generalization Results}
\label{app:structural-results}

\subsection{Structural Sufficiency Assumption}
\label{app:structural-sufficiency}

We recall Assumption~\ref{assum:epsilon_relaxed}, which is used throughout the proofs below. It states that there exists an irreducible approximation error $\epsilon>0$ such that, for each model $f$, there is a Borel measurable structural risk function $\mathcal{L}_f:\mathcal{X}\to[0,1]$ satisfying, for every $\mu\in\{\mu_S,\mu_T\}$,
\begin{equation}
\left|
\mathcal{R}_{\mu}(f)
-
\int_{\mathcal{X}}
\mathcal{L}_f(\boldsymbol{\chi})\,d(\Phi_{\#}\mu)(\boldsymbol{\chi})
\right|
\le \epsilon .
\end{equation}

\subsection{Proof of Lemma~\ref{lemma:error_decomposition}}
\label{app:proof-error-decomp}
We now give the full proof of Lemma~\ref{lemma:error_decomposition}, which the main text only sketches. For completeness, we restate Lemma~\ref{lemma:error_decomposition}.

Assume that $\mathcal{L}_f$ is Lipschitz continuous on $(\mathcal{X}, d_{\mathcal{X}})$ with minimal Lipschitz constant
\begin{equation}
K_f
:=
\sup_{\boldsymbol{u}\neq \boldsymbol{v} \in \mathcal{X}}
\frac{|\mathcal{L}_f(\boldsymbol{u})-\mathcal{L}_f(\boldsymbol{v})|}{d_{\mathcal{X}}(\boldsymbol{u},\boldsymbol{v})}.
\end{equation}
Then the generalization error gap satisfies
\begin{equation}
\left|
\mathcal{R}_{\mu_T}(f)-\mathcal{R}_{\mu_S}(f)
\right|
\le
K_f\,W_1(\nu_S,\nu_T)+2\epsilon.
\end{equation}

\begin{proof}
Let
\begin{equation}
\Delta \mathcal{R}
:=
\left|
\mathcal{R}_{\mu_T}(f)-\mathcal{R}_{\mu_S}(f)
\right|.
\end{equation}
Adding and subtracting the structural expectations under $\nu_T$ and $\nu_S$ gives
\begin{align}
\Delta \mathcal{R}
&=
\Bigg|
\mathcal{R}_{\mu_T}(f)
-
\int_{\mathcal{X}}\mathcal{L}_f(\boldsymbol{\chi})\,d\nu_T(\boldsymbol{\chi})
+
\int_{\mathcal{X}}\mathcal{L}_f(\boldsymbol{\chi})\,d\nu_T(\boldsymbol{\chi})
\notag\\
&\qquad
-
\int_{\mathcal{X}}\mathcal{L}_f(\boldsymbol{\chi})\,d\nu_S(\boldsymbol{\chi})
+
\int_{\mathcal{X}}\mathcal{L}_f(\boldsymbol{\chi})\,d\nu_S(\boldsymbol{\chi})
-
\mathcal{R}_{\mu_S}(f)
\Bigg|.
\end{align}
Applying the triangle inequality yields
\begin{align}
\Delta \mathcal{R}
&\le
\left|
\mathcal{R}_{\mu_T}(f)
-
\int_{\mathcal{X}}\mathcal{L}_f(\boldsymbol{\chi})\,d\nu_T(\boldsymbol{\chi})
\right|
\notag\\
&\qquad+
\left|
\int_{\mathcal{X}}\mathcal{L}_f(\boldsymbol{\chi})\,d\nu_T(\boldsymbol{\chi})
-
\int_{\mathcal{X}}\mathcal{L}_f(\boldsymbol{\chi})\,d\nu_S(\boldsymbol{\chi})
\right|
\notag\\
&\qquad+
\left|
\int_{\mathcal{X}}\mathcal{L}_f(\boldsymbol{\chi})\,d\nu_S(\boldsymbol{\chi})
-
\mathcal{R}_{\mu_S}(f)
\right|.
\end{align}
By Assumption~\ref{assum:epsilon_relaxed},
\begin{equation}
\Delta \mathcal{R}
\le
\left|
\int_{\mathcal{X}}\mathcal{L}_f(\boldsymbol{\chi})\,d\nu_T(\boldsymbol{\chi})
-
\int_{\mathcal{X}}\mathcal{L}_f(\boldsymbol{\chi})\,d\nu_S(\boldsymbol{\chi})
\right|
+2\epsilon.
\end{equation}

If $K_f=0$, then $\mathcal{L}_f$ is constant on $\mathcal{X}$, so the integral difference vanishes and the claim follows immediately.

Assume now that $K_f>0$, and define
\begin{equation}
g(\boldsymbol{\chi}) := \frac{\mathcal{L}_f(\boldsymbol{\chi})}{K_f}.
\end{equation}
Then $g$ is $1$-Lipschitz. Since $\nu_S,\nu_T\in\mathcal{P}_1(\mathcal X)$, the Kantorovich--Rubinstein duality for $W_1$ applies on $(\mathcal X,d_{\mathcal X})$ \citep{Villani2009,weed2019sharp}. Therefore,
\begin{equation}
\left|
\int_{\mathcal{X}} g(\boldsymbol{\chi})\,d\nu_T(\boldsymbol{\chi})
-
\int_{\mathcal{X}} g(\boldsymbol{\chi})\,d\nu_S(\boldsymbol{\chi})
\right|
\le
W_1(\nu_S,\nu_T).
\end{equation}
Substituting $g=\mathcal{L}_f/K_f$ gives
\begin{equation}
\left|
\int_{\mathcal{X}} \mathcal{L}_f(\boldsymbol{\chi})\,d\nu_T(\boldsymbol{\chi})
-
\int_{\mathcal{X}} \mathcal{L}_f(\boldsymbol{\chi})\,d\nu_S(\boldsymbol{\chi})
\right|
\le
K_f\,W_1(\nu_S,\nu_T).
\end{equation}
Combining this bound with the previous inequality yields
\begin{equation}
\Delta \mathcal{R}
\le
K_f\,W_1(\nu_S,\nu_T)+2\epsilon,
\end{equation}
which proves the claim.
\end{proof}

\section{Proofs of Positional Encoding Results}
\label{app:pe-results}

\subsection{Attention Kernels and the Translation Operator}
\label{app:attention-translation}

We first restate the attention kernels for APE and RoPE. Let $\mathbf{x}_i,\mathbf{x}_j \in \mathbb{R}^d$ denote hidden state vectors at absolute sequence indices $i$ and $j$, and let $\mathbf{W}_Q,\mathbf{W}_K \in \mathbb{R}^{d\times d}$ denote the learned query and key projection matrices.

For absolute positional encoding (APE), let $\mathbf{p}_i,\mathbf{p}_j \in \mathbb{R}^d$ denote the absolute position embeddings. The pre-softmax attention score is
\begin{equation}
A_{i,j}^{\mathrm{APE}}
=
(\mathbf{x}_i+\mathbf{p}_i)^\top
\mathbf{W}_Q^\top \mathbf{W}_K
(\mathbf{x}_j+\mathbf{p}_j).
\end{equation}

For rotary position embedding (RoPE) \citep{su2023roformerenhancedtransformerrotary}, let $\mathbf{R}_\Theta^t \in \mathbb{R}^{d\times d}$ denote the orthogonal block-diagonal rotation matrix associated with position $t$. Using $(\mathbf{R}_\Theta^i)^\top = (\mathbf{R}_\Theta^i)^{-1} = \mathbf{R}_\Theta^{-i}$, the pre-softmax attention score is
\begin{align}
A_{i,j}^{\mathrm{RoPE}}
&=
(\mathbf{R}_\Theta^i \mathbf{W}_Q \mathbf{x}_i)^\top
(\mathbf{R}_\Theta^j \mathbf{W}_K \mathbf{x}_j) \\
&=
\mathbf{x}_i^\top \mathbf{W}_Q^\top
(\mathbf{R}_\Theta^i)^\top
\mathbf{R}_\Theta^j
\mathbf{W}_K \mathbf{x}_j \\
&=
\mathbf{x}_i^\top \mathbf{W}_Q^\top
\mathbf{R}_\Theta^{j-i}
\mathbf{W}_K \mathbf{x}_j .
\end{align}

\paragraph{Bounds on the structural perturbation $\delta$.}
Let $z \in \mathcal{Z}$ be written as $z = x_{\mathrm{init}}^{(z)} \oplus y_{\mathrm{search}}$. For $k \in \mathbb{N}^+$, the translation operator
\[
\mathcal{T}_k : \mathcal{Z} \to \mathcal{Z}
\]
inserts an isolated, structurally inert node (e.g., an identity operation such as \texttt{x=x}) of exactly $k$ tokens immediately before a local reasoning sub-trajectory $\omega \subset y_{\mathrm{search}}$.

Under the projection $\Phi$, this insertion increments the absolute volume coordinate by one: $\chi_4' = \chi_4 + 1$. Because the inserted node does not alter the maximum tree depth ($\Delta \chi_1 = 0$) and its contribution to the global transition rates and depth variance decays at a rate of $\mathcal{O}(1/N)$, the remaining structural perturbations are negligible for long trajectories. Hence, the exact Euclidean structural distance is
\begin{equation}
    d_{\mathcal{X}}(\boldsymbol{\chi},\boldsymbol{\chi}')
    =
    \left(
    1^2+\sum_{m\neq 4}(\Delta\chi_m)^2
    \right)^{1/2}
    =: \delta.
\end{equation}
As the trajectory length $N$ grows, $\sum_{m\neq 4}(\Delta\chi_m)^2 \to 0$, so the perturbation distance stays near one: $1 \le \delta \le 1 + \mathcal{O}(1/N) =: C_\delta$. The insertion also shifts the absolute index of every token in the subsequent sub-trajectory $\omega$ by $+k$.

\subsection{Proof of Lemma~\ref{lemma:ape_shattering}}
\label{app:proof-ape-shattering}

For completeness, we restate Lemma~\ref{lemma:ape_shattering}.

Assume the network parameters (projection matrices $\mathbf{W}_Q, \mathbf{W}_K$ and the absolute position embeddings $\mathbf{p}_t$) are generically parameterized (i.e., avoiding exact algebraic degeneracies), and that
\begin{equation}
\mathbf{W}_Q^\top \mathbf{W}_K \neq \mathbf{0}.
\end{equation}
Then, under the translation operator $\mathcal{T}_k$ with $k>0$, the local attention kernel of an APE model differs almost surely:
\begin{equation}
\Delta A^{\mathrm{APE}} = |\tilde{A}_{a,b}^{\mathrm{APE}} - A_{a,b}^{\mathrm{APE}}| > 0
\qquad \text{a.s.}
\end{equation}

\begin{proof}
Let $\mathbf{x}_a,\mathbf{x}_b \in \omega$ be tokens originally appearing at indices $a$ and $b$. After applying $\mathcal{T}_k$, these indices become $a+k$ and $b+k$. The resulting attention difference is
\begin{align}
\Delta A^{\mathrm{APE}}
&:=
\left|
\widetilde{A}_{a,b}^{\mathrm{APE}} - A_{a,b}^{\mathrm{APE}}
\right| \\
&=
\Big|
(\mathbf{x}_a+\mathbf{p}_{a+k})^\top \mathbf{W}_Q^\top \mathbf{W}_K (\mathbf{x}_b+\mathbf{p}_{b+k})
-
(\mathbf{x}_a+\mathbf{p}_a)^\top \mathbf{W}_Q^\top \mathbf{W}_K (\mathbf{x}_b+\mathbf{p}_b)
\Big|.
\end{align}
Expanding both bilinear forms and canceling the shared term $\mathbf{x}_a^\top \mathbf{W}_Q^\top \mathbf{W}_K \mathbf{x}_b$ yields
\begin{align}
\Delta A^{\mathrm{APE}}
&=
\Big|
\mathbf{x}_a^\top \mathbf{W}_Q^\top \mathbf{W}_K (\mathbf{p}_{b+k}-\mathbf{p}_b)
+
(\mathbf{p}_{a+k}-\mathbf{p}_a)^\top \mathbf{W}_Q^\top \mathbf{W}_K \mathbf{x}_b
\notag\\
&\qquad\qquad
+
\mathbf{p}_{a+k}^\top \mathbf{W}_Q^\top \mathbf{W}_K \mathbf{p}_{b+k}
-
\mathbf{p}_a^\top \mathbf{W}_Q^\top \mathbf{W}_K \mathbf{p}_b
\Big|.
\end{align}

Introduce $\mathbf{u} := \mathbf{p}_{a+k}$ and $\mathbf{v} := \mathbf{p}_{b+k}$, and define
\begin{align}
S(\mathbf{u},\mathbf{v})
&:=
\mathbf{u}^\top \mathbf{W}_Q^\top \mathbf{W}_K \mathbf{v}
+
\mathbf{x}_a^\top \mathbf{W}_Q^\top \mathbf{W}_K \mathbf{v}
+
\mathbf{u}^\top \mathbf{W}_Q^\top \mathbf{W}_K \mathbf{x}_b
\notag\\
&\qquad
-
\Big(
\mathbf{x}_a^\top \mathbf{W}_Q^\top \mathbf{W}_K \mathbf{p}_b
+
\mathbf{p}_a^\top \mathbf{W}_Q^\top \mathbf{W}_K \mathbf{x}_b
+
\mathbf{p}_a^\top \mathbf{W}_Q^\top \mathbf{W}_K \mathbf{p}_b
\Big).
\end{align}
Then perfect cancellation would require $S(\mathbf{u},\mathbf{v}) = 0$.

The leading term of $S$ is the bilinear form
\begin{equation}
\mathbf{u}^\top \mathbf{W}_Q^\top \mathbf{W}_K \mathbf{v}
=
\sum_{i=1}^d \sum_{j=1}^d
u_i (\mathbf{W}_Q^\top \mathbf{W}_K)_{ij} v_j.
\end{equation}
Since $\mathbf{W}_Q^\top \mathbf{W}_K \neq \mathbf{0}$, this is a nontrivial polynomial in $(\mathbf{u},\mathbf{v}) \in \mathbb{R}^{2d}$. Therefore by the properties of real algebraic geometry \citep{mityagin2020zeroset}, the zero set
\begin{equation}
\{(\mathbf{u},\mathbf{v}) \in \mathbb{R}^{2d} : S(\mathbf{u},\mathbf{v}) = 0\}
\end{equation}
is a proper algebraic variety and has Lebesgue measure zero.

This zero set contains exact degeneracies such as position embeddings with period $k$ ($\mathbf{p}_{t+k}=\mathbf{p}_t$, so that $\mathbf{u}=\mathbf{p}_a$ and $\mathbf{v}=\mathbf{p}_b$) and position differences $\mathbf{p}_{a+k}-\mathbf{p}_a$ and $\mathbf{p}_{b+k}-\mathbf{p}_b$ in the null spaces of $\mathbf{W}_K^\top \mathbf{W}_Q$ and $\mathbf{W}_Q^\top \mathbf{W}_K$, respectively. Under the genericity assumption of the lemma, $(\mathbf{u},\mathbf{v})$ lies outside this measure-zero set almost surely. Hence
\begin{equation}
\Delta A^{\mathrm{APE}} > 0
\qquad \text{almost surely}.
\end{equation}
\end{proof}

\subsection{RoPE Equivariance under Sequence Translation}
\label{app:proof-rope-equivariance}

We provide the RoPE invariance argument used in the proof of Theorem~\ref{thm:lipschitz_gap}. Under the translation operator $\mathcal{T}_k$, the local attention kernel of RoPE is invariant: $\Delta A^{\mathrm{RoPE}} = 0$.

\begin{proof}
Let $\mathbf{x}_a,\mathbf{x}_b$ be translated from indices $a,b$ to $a+k,b+k$. The translated RoPE kernel is
\begin{align}
\widetilde{A}_{a,b}^{\mathrm{RoPE}}
&=
\mathbf{x}_a^\top \mathbf{W}_Q^\top
\mathbf{R}_\Theta^{(b+k)-(a+k)}
\mathbf{W}_K \mathbf{x}_b \\
&=
\mathbf{x}_a^\top \mathbf{W}_Q^\top
\mathbf{R}_\Theta^{b-a}
\mathbf{W}_K \mathbf{x}_b \\
&=
A_{a,b}^{\mathrm{RoPE}}.
\end{align}
Therefore
\begin{equation}
\Delta A^{\mathrm{RoPE}}
=
\left|
\widetilde{A}_{a,b}^{\mathrm{RoPE}} - A_{a,b}^{\mathrm{RoPE}}
\right|
=
0.
\end{equation}
\end{proof}

\subsection{Details of Assumption~\ref{assum:ar_sensitivity}}
\label{app:assumption-smoothness}

We expand Assumption~\ref{assum:ar_sensitivity}, which is used in the proof of Theorem~\ref{thm:lipschitz_gap}.

\textbf{(i) Global smoothness.}
The base reasoning task admits an intrinsic structural smoothness scale satisfying $K_{\mathrm{base}} = \mathcal{O}(1/T_{\max})$. Because RoPE preserves relative position information exactly under sequence translations, it inherits this smoothness scale: $K_{f_{\mathrm{RoPE}}} \le \mathcal{O}(1/T_{\max})$.

\textbf{(ii) Autoregressive sensitivity and fragile logit margin.} Let $f_{\mathrm{APE}} \in \mathcal{F}_{L,m,\mathrm{APE}}$. In formal combinatorial planning, sequence generation is a deterministic Markov decision process. We assume that the logit margin is fragile at structurally critical routing steps (e.g., recursive backtracking steps in deep subtrees). Let $\Delta_{\mathrm{margin}} > 0$ denote the minimal pre-softmax logit separation required between the Bayes-optimal routing token $a^\star$ and competing incorrect tokens. For out-of-distribution reasoning trajectories with deep hierarchical structures, we assume this margin shrinks ($\Delta_{\mathrm{margin}} \to 0^+$). Consequently, once the margin falls below the disruption ($\Delta_{\mathrm{margin}} < \Delta A$, with $\Delta A > 0$ induced by APE translations; Lemma~\ref{lemma:ape_shattering}), deterministic greedy decoding selects an incorrect subtree and cannot reach the target state. The change in the expected structural risk is then bounded below by the deterministic separation margin:
\begin{equation}
|\mathcal{L}_{f_{\mathrm{APE}}}(\boldsymbol{\chi}') - \mathcal{L}_{f_{\mathrm{APE}}}(\boldsymbol{\chi})| \ge \gamma, \qquad \gamma := 1 - \frac{1}{|\mathcal{V}|} > 0.
\end{equation}

\paragraph{Routing separation under deterministic decoding.}
Let $a^\star(z)$ denote the Bayes-optimal routing action at the critical
transition of trajectory $z$, and let
\[
a_f(z) := \arg\max_{a\in\mathcal V} p_f(a\mid h_z)
\]
denote the action selected by the model under greedy decoding, where $h_z$ is the prefix of the model's own greedy trace for the problem $x_{\mathrm{init}}^{(z)}$ up to that transition. Assumption~\ref{assum:ar_sensitivity} states that, at a structurally critical transition, the APE-induced
coordinate perturbation changes the selected routing action away from the Bayes-optimal
action,
\[
a_f(z') \neq a^\star(z'),
\]
on problems $z'$ whose reference searches map near $\boldsymbol{\chi}'=\Phi(\mathcal T_k z)$ (see the next paragraph). This routing
deviation sends the autoregressive process into an incorrect subtree and prevents
the generated trajectory for $z'$ from reaching the target state. Here an incorrect subtree means a routing error after which the model's own trace reports no valid path to the target within the $512$-token generation budget. Because the scorer gives no partial credit (it checks the path that the trace reports to the target; Appendix~\ref{app:reproducibility}), the zero-one loss on $z'$ is then $1$. Hence the induced
structural risk changes by a constant amount independent of $T_{\max}$:
\begin{equation}
\left|
\mathcal{L}_{f_{\mathrm{APE}}}(\boldsymbol{\chi}')
-
\mathcal{L}_{f_{\mathrm{APE}}}(\boldsymbol{\chi})
\right|
\ge
\gamma,
\qquad
\gamma := 1-\frac{1}{|\mathcal V|}>0.
\end{equation}
The constant $1-1/|\mathcal V|$ is the error rate of uniform guessing over $\mathcal V$.
Under exact-match evaluation the jump can be as large as $1$; for
Theorem~\ref{thm:lipschitz_gap} it suffices that $\gamma>0$ is independent of
$T_{\max}$.

\paragraph{Assumption~\ref{assum:ar_sensitivity} as a condition on $\mathcal{L}_f$.}
The loss $\ell(f,z)$ depends on $z$ only through the problem $x_{\mathrm{init}}^{(z)}$ and its target (Appendix~\ref{app:risk}), while $\mathcal{T}_k$ acts on $y_{\mathrm{search}}$, so $\ell(f,\mathcal{T}_k z)=\ell(f,z)$. Assumption~\ref{assum:ar_sensitivity} therefore constrains how the structural risk function $\mathcal{L}_f$ varies between the neighboring points $\boldsymbol{\chi}=\Phi(z)$ and $\boldsymbol{\chi}'=\Phi(\mathcal{T}_k z)$ of $\mathcal{X}$. On other problems $z'$ whose reference searches map near $\boldsymbol{\chi}'$, an APE model places a familiar sub-search of its own trace at a shifted absolute position and thus undergoes the kernel change of Lemma~\ref{lemma:ape_shattering}; under Assumption~\ref{assum:ar_sensitivity}(ii), this change yields the jump $\gamma$.

\subsection{Proof of Theorem~\ref{thm:lipschitz_gap}}
\label{app:proof-lipschitz-gap}

For completeness, we restate Theorem~\ref{thm:lipschitz_gap}, which holds under Assumption~\ref{assum:ar_sensitivity}.
For sufficiently large $T_{\max}$, let 
$f_{\mathrm{APE}} \in \mathcal{F}_{L,m,\mathrm{APE}}$ and 
$f_{\mathrm{RoPE}} \in \mathcal{F}_{L,m,\mathrm{RoPE}}$ 
be two matched autoregressive Transformer models that differ only in their positional encoding scheme. Then
\begin{equation}
K_{f_{\mathrm{RoPE}}} = \mathcal{O}(1/T_{\max}),
\qquad
K_{f_{\mathrm{APE}}} = \Omega(1),
\end{equation}
and hence $K_{f_{\mathrm{RoPE}}} \ll K_{f_{\mathrm{APE}}}$.

\begin{proof}
By definition,
\begin{equation}
K_{f_{\mathrm{APE}}}
=
\sup_{\boldsymbol{u}\neq \boldsymbol{v}\in\mathcal{X}}
\frac{
|\mathcal{L}_{f_{\mathrm{APE}}}(\boldsymbol{u})-\mathcal{L}_{f_{\mathrm{APE}}}(\boldsymbol{v})|
}{
d_{\mathcal{X}}(\boldsymbol{u},\boldsymbol{v})
}.
\end{equation}
Let $z \in \mathcal{Z}$ be a trajectory whose translated sub-trajectory under $\mathcal{T}_k$ contains a critical routing step in the sense of Assumption~\ref{assum:ar_sensitivity}(ii), and let
\begin{equation}
\boldsymbol{\chi} = \Phi(z),
\qquad
\boldsymbol{\chi}' = \Phi(\mathcal{T}_k(z)).
\end{equation}
By Definition~\ref{def:translation_operator}, the operator $\mathcal{T}_k$ inserts an inert node immediately before a local reasoning sub-trajectory. By Lemma~\ref{lemma:ape_shattering}, this induces a nonzero local APE attention perturbation almost surely. Assumption~\ref{assum:ar_sensitivity}(ii), read as a condition on $\mathcal{L}_f$ (last paragraph of Appendix~\ref{app:assumption-smoothness}), then yields
\begin{equation}
|\mathcal{L}_{f_{\mathrm{APE}}}(\boldsymbol{\chi}')-\mathcal{L}_{f_{\mathrm{APE}}}(\boldsymbol{\chi})|
\ge \gamma.
\end{equation}
Hence
\begin{equation}
K_{f_{\mathrm{APE}}}
\ge
\frac{
|\mathcal{L}_{f_{\mathrm{APE}}}(\boldsymbol{\chi}')-\mathcal{L}_{f_{\mathrm{APE}}}(\boldsymbol{\chi})|
}{
d_{\mathcal{X}}(\boldsymbol{\chi},\boldsymbol{\chi}')
}
\ge
\frac{\gamma}{\delta}
\ge
\frac{\gamma}{C_\delta}
=
\Omega(1).
\end{equation}

Assumption~\ref{assum:ar_sensitivity}(i) gives
\begin{equation}
K_{f_{\mathrm{RoPE}}} \le \mathcal{O}(1/T_{\max}).
\end{equation}
Therefore, for sufficiently large $T_{\max}$,
\begin{equation}
K_{f_{\mathrm{RoPE}}} \ll K_{f_{\mathrm{APE}}}.
\end{equation}
\end{proof}

\section{Proofs of Capacity Bounds}
\label{app:capacity-results}

\subsection{Structural Capacity Analysis}
\label{app:structural-capacity}

\begin{definition}[Bayes-Optimal Structural Risk] \label{def:bayes-structural-risk} The Bayes-optimal structural risk function is defined by \begin{equation} \mathcal{L}^*(\boldsymbol{\chi}) := \inf_{h} \mathbb{E}_{z \sim \mu_T}\!\left[\ell(h,z)\mid \Phi(z)=\boldsymbol{\chi}\right], \qquad \boldsymbol{\chi}\in\mathcal{X}. \end{equation} \end{definition} \begin{definition}[Barron Space] \label{def:barron-space} Let $\mathcal{X}\subset\mathbb{R}^{12}$ be compact. The Barron space $\mathcal{H}_{\mathrm{Barron}}(\mathcal{X})$ consists of continuous functions $g:\mathcal{X}\to\mathbb{R}$ whose Fourier transform $\hat g(\boldsymbol{\omega})$ satisfies \begin{equation} \int_{\mathbb{R}^{12}} \|\boldsymbol{\omega}\|_1 \,|\hat g(\boldsymbol{\omega})| \,d\boldsymbol{\omega} <\infty. \end{equation} The infimum of this quantity defines the Barron norm $\|g\|_{\mathcal{B}}$ \citep{Barron1993}. \end{definition}
\begin{assumption}[Tsybakov Margin Condition]
\label{ass:tsybakov-margin}
There exist constants $C > 0$ and $\alpha \ge 0$ such that for the conditional label probability 
$\eta(\boldsymbol{\chi}) := \mathbb{P}_{z\sim\mu_T}(\ell(f,z)=1 \mid \Phi(z)=\boldsymbol{\chi})$,
\[
\nu_T\big(|\eta(\boldsymbol{\chi}) - 1/2| \le t\big) 
\le C t^\alpha, \quad \forall t > 0.
\]
\end{assumption}
Assumption~\ref{ass:tsybakov-margin} is a standard condition for relating the $L_1(\nu_T)$ approximation error of a structural risk function to the zero-one risk \citep{tsybakov2004optimal}. The proof of Lemma~\ref{lemma:width_capacity} (Appendix~\ref{app:proof-width-capacity}) measures this error against the pointwise Bayes risk $\mathcal{L}^*$ (Definition~\ref{def:bayes-structural-risk}), which gives the relation directly, with constant $\frac12$ and residual $\frac{\epsilon}{2}$, so this proof rests on Assumptions~\ref{assum:epsilon_relaxed} and~\ref{ass:width-capacity}.

\begin{assumption}[Structural Capacity Mapping and Approximation Limit] \label{ass:width-capacity} We assume the following. \textbf{(i) Barron regularity of the target risk.} The Bayes-optimal structural risk belongs to the Barron space: \begin{equation} \mathcal{L}^* \in \mathcal{H}_{\mathrm{Barron}}(\mathcal{X}), \qquad 0 < \|\mathcal{L}^*\|_{\mathcal{B}} < \infty. \end{equation} \textbf{(ii) Finite-width structural hypothesis class.} For each autoregressive Transformer $f \in \mathcal{F}_{L,m,\mathrm{PE}}$, the induced structural risk function $\mathcal{L}_f$ belongs to a width-constrained hypothesis class \begin{equation} \mathcal{L}_f \in \mathcal{H}_m \subset \mathcal{H}_{\mathrm{Barron}}(\mathcal{X}), \end{equation} where $\mathcal{H}_m$ denotes a structural approximation class whose capacity is controlled by the internal representation width $m$. 

\textbf{(iii) Width-limited approximation lower bound.} There exists a constant $C_{\mathcal{B}}>0$, independent of $m$, such that \begin{equation} \inf_{g \in \mathcal{H}_m} \int_{\mathcal{X}} |g(\boldsymbol{\chi})-\mathcal{L}^*(\boldsymbol{\chi})| \,d\nu_T(\boldsymbol{\chi}) \ge \frac{C_{\mathcal{B}}}{\sqrt{m}}. \end{equation} This scaling mirrors the $\mathcal{O}(m^{-1/2})$ upper rate for Barron-type function classes \citep{Barron1993}. \end{assumption}

Assumption~\ref{ass:width-capacity} formalizes the idea that projecting discrete trajectories into $\mathcal{X}$ induces a regression problem whose best attainable predictor cannot be approximated arbitrarily well at fixed width.

\subsubsection{Proof of Lemma~\ref{lemma:width_capacity}}
\label{app:proof-width-capacity}

\begin{proof}
By Assumption~\ref{assum:epsilon_relaxed}, the true expected zero-one risk is bounded by the structural risk:
\begin{equation}
\mathcal{R}_{\mu_T}(f) \ge \int_{\mathcal{X}} \mathcal{L}_f(\boldsymbol{\chi}) \,d\nu_T(\boldsymbol{\chi}) - \epsilon.
\end{equation}

Since $\mathcal{L}^*$ is the pointwise Bayes risk, $\int_{\mathcal{X}}\mathcal{L}^*\,d\nu_T\le\mathcal{R}_{\mu_T}(f)$; with the bound above and the inequality $|\mathcal{L}_f-\mathcal{L}^*|\le\mathcal{L}_f+\mathcal{L}^*$, this gives $\mathcal{R}_{\mu_T}(f)\ge\frac12\int_{\mathcal{X}}|\mathcal{L}_f-\mathcal{L}^*|\,d\nu_T-\frac{\epsilon}{2}$.

Since $\mathcal{L}_f \in \mathcal{H}_m$, Assumption~\ref{ass:width-capacity}(iii) (the minimax-hardness of $\mathcal{L}^*$) implies that recovering $\mathcal{L}^*$ incurs an absolute $L_1$ error of at least $C_{\mathcal{B}}\, m^{-1/2}$. %
 Thus:
\begin{equation}
\int_{\mathcal{X}} \left| \mathcal{L}_f(\boldsymbol{\chi}) - \mathcal{L}^*(\boldsymbol{\chi}) \right| \,d\nu_T(\boldsymbol{\chi}) \ge \frac{C_{\mathcal{B}}}{\sqrt{m}}.
\end{equation}
Defining $C_{\mathrm{width}} := C_{\mathcal{B}}/2$ and substituting this back yields $\mathcal{R}_{\mu_T}(f) \ge \frac{C_{\mathrm{width}}}{\sqrt{m}} - \epsilon$.
\end{proof}

\subsection{Depth Bottleneck via Dyck-Type Reduction}
\label{app:depth-bound}

We now formalize the depth bottleneck of sequential backtracking in more detail. The central idea is that nested backtracking induces a Dyck-type dependency whose effective computational depth scales with the structural depth coordinate $\chi_1$.

\subsubsection{DFS-to-Dyck Encoding}
\label{app:dfs-dyck}

We give the full construction for Lemma~\ref{lemma:dyck_homomorphism}. Let $k$ denote the maximum branching factor of the search tree. We define the search transition alphabet
\begin{equation}
\Sigma_{\mathrm{search}}
:=
\{\delta_{\mathrm{down}}^{(1)},\dots,\delta_{\mathrm{down}}^{(k)}\}
\cup
\{\delta_{\mathrm{up}}^{(1)},\dots,\delta_{\mathrm{up}}^{(k)}\}
\cup
\{\omega_{\mathrm{eval}}\},
\end{equation}
where $\delta_{\mathrm{down}}^{(c)}$ denotes expansion to the $c$-th child, $\delta_{\mathrm{up}}^{(c)}$ denotes backtracking from the $c$-th child, and $\omega_{\mathrm{eval}}$ denotes terminal evaluation.

Let $\Sigma_{\mathrm{Dyck}}=\{(_1, )_1,\dots,(_k, )_k\}$, and let $\mathcal{D}_k \subset \Sigma_{\mathrm{Dyck}}^*$ denote the Dyck-$k$ language \citep{Strobl_2024}. Define the map
\begin{equation}
\psi:\Sigma_{\mathrm{search}}^*\to\Sigma_{\mathrm{Dyck}}^*
\end{equation}
pointwise by
\begin{equation}
\psi(\delta_{\mathrm{down}}^{(c)}) = (_c,
\qquad
\psi(\delta_{\mathrm{up}}^{(c)}) = )_c,
\qquad
\psi(\omega_{\mathrm{eval}})=\varepsilon.
\end{equation}
Since $\psi$ preserves concatenation, i.e., $\psi(x\oplus y)=\psi(x)\oplus\psi(y)$ for all $x,y\in\Sigma_{\mathrm{search}}^*$, it is a monoid homomorphism.

When a valid trajectory $z$ is read as the walk on the search tree through its successively visited nodes, with each transition expanded into parent-child moves (for DFS these are exactly the push and pop moves of backtracking), its image $\psi(z)$ records this nested push-pop structure as a prefix of a balanced-parenthesis string, which is balanced if and only if the walk returns to the root. In particular, the maximum nesting depth of $\psi(z)$ equals the maximum depth reached, i.e., $\chi_1-1$ under the identifier convention of Appendix~\ref{app:trajectory-space}.

\subsubsection{Complexity-Theoretic Interpretation and Finite Precision}
\label{app:depth-complexity}

Consider an autoregressive step $t$ at which the next token corresponds to a backtracking move $v_{t+1}=\delta_{\mathrm{up}}^{(c)}$. To emit this token correctly, the model must identify the most recent unmatched opening transition associated with the same branch type $c$. Under the encoding above, this is equivalent to resolving a nested Dyck-type dependency over the serialized search history.

Parsing Dyck-language dependencies relies on maintaining a stack counter within the continuous hidden representations. For a Transformer using bounded numerical precision with $p$ bits (e.g., $p=16$ for bfloat16), the capacity to reliably encode distinguishable nested states is information-theoretically constrained by $\mathcal{O}(2^p)$. As shown by \citet{Hahn_2020}, self-attention cannot model hierarchical languages unless the number of layers or heads grows with the input length; we therefore treat the evaluation of a nested structure as requiring sequential reductions across the layers. 

We formalize this depth requirement by letting the maximum resolvable tree depth $\chi_1$ scale linearly with the physical layer depth $L$ (Assumption~\ref{ass:implicit-complexity} below). The truncation constant $\alpha_{\mathrm{circ}}:=c^{-1}$ defined there, which may depend on the task and on the numerical precision $p$, sets the maximum parsable tree depth per physical layer.

\begin{assumption}[Implicit State-Transition Complexity via Dyck Structure]
\label{ass:implicit-complexity}
There exists a precision-dependent constant $c>0$ such that resolving a backtracking transition of structural depth $\chi_1$ requires implicit computational depth at least
\begin{equation}
D_{\mathrm{implicit}} \ge c\,\chi_1.
\end{equation}
Equivalently, defining $\alpha_{\mathrm{circ}}:=c^{-1}$, trajectories satisfying $\chi_1>\alpha_{\mathrm{circ}}L$ lie strictly beyond the implicit computational depth available within a single autoregressive forward pass. On these trajectories every $f\in\mathcal{F}_{L,m,\mathrm{PE}}$ has structural risk at least that of uniform guessing over $\mathcal{V}$: $\mathcal{L}_f(\boldsymbol{\chi})\ge 1-1/|\mathcal{V}|$ whenever $\chi_1>\alpha_{\mathrm{circ}}L$.
\end{assumption}

\subsubsection{Failure Beyond the Implicit Depth Threshold}
\label{app:depth-threshold-failure}

Define the severed region
\begin{equation}
\mathcal{X}_{\mathrm{severed}}
:=
\left\{
\boldsymbol{\chi}\in\mathcal{X}:\chi_1>\alpha_{\mathrm{circ}}L
\right\}.
\end{equation}
For trajectories in $\mathcal{X}_{\mathrm{severed}}$, the model lacks sufficient implicit depth to resolve the required backtracking dependency. In this case, the relevant routing decision can no longer be reliably distinguished from competing valid outputs, and Assumption~\ref{ass:implicit-complexity} lower-bounds the resulting structural risk by the uniform-guessing baseline
\begin{equation}
C_{\mathrm{depth}}:=1-\frac{1}{|\mathcal{V}|}.
\end{equation}

This motivates the truncation indicator
\begin{equation}
P_{\mathrm{implicit}}(\boldsymbol{\chi};L,\alpha_{\mathrm{circ}})
:=
\mathbb{I}_{(\alpha_{\mathrm{circ}}L,\infty)}(\chi_1),
\end{equation}
which records whether a structural state lies beyond the implicit depth threshold.

\subsubsection{Proof of Lemma~\ref{lemma:depth_bottleneck}}
\label{app:proof-depth-bottleneck}

For completeness, we restate Lemma~\ref{lemma:depth_bottleneck}, which holds under Assumptions~\ref{assum:epsilon_relaxed} and~\ref{ass:implicit-complexity}. Let $|\mathcal{V}|\ge2$, and define $C_{\mathrm{depth}}:=1-\frac{1}{|\mathcal{V}|}\in(0,1)$. Then for any autoregressive Transformer hypothesis $f\in\mathcal{F}_{L,m,\mathrm{PE}}$,
\begin{equation}
\mathcal{R}_{\mu_T}(f)
\ge
C_{\mathrm{depth}}
\int_{\mathcal{X}}
P_{\mathrm{implicit}}(\boldsymbol{\chi};L,\alpha_{\mathrm{circ}})
\,d\nu_T(\boldsymbol{\chi})
-\epsilon.
\end{equation}

\begin{proof}
For $\boldsymbol{\chi}\in\mathcal{X}_{\mathrm{severed}}$, the required implicit routing depth exceeds the available physical depth $L$, and Assumption~\ref{ass:implicit-complexity} lower-bounds the structural risk on this region by
\begin{equation}
\mathcal{L}_f(\boldsymbol{\chi})\ge C_{\mathrm{depth}}.
\end{equation}
It follows that
\begin{align}
\int_{\mathcal{X}}
\mathcal{L}_f(\boldsymbol{\chi})\,d\nu_T(\boldsymbol{\chi})
&=
\int_{\mathcal{X}_{\mathrm{severed}}}
\mathcal{L}_f(\boldsymbol{\chi})\,d\nu_T(\boldsymbol{\chi})
+
\int_{\mathcal{X}\setminus\mathcal{X}_{\mathrm{severed}}}
\mathcal{L}_f(\boldsymbol{\chi})\,d\nu_T(\boldsymbol{\chi}) \\
&\ge
\int_{\mathcal{X}_{\mathrm{severed}}}
\mathcal{L}_f(\boldsymbol{\chi})\,d\nu_T(\boldsymbol{\chi}) \\
&\ge
C_{\mathrm{depth}}\,
\nu_T(\mathcal{X}_{\mathrm{severed}}).
\end{align}
Since
\begin{equation}
\nu_T(\mathcal{X}_{\mathrm{severed}})
=
\int_{\mathcal{X}}
P_{\mathrm{implicit}}(\boldsymbol{\chi};L,\alpha_{\mathrm{circ}})
\,d\nu_T(\boldsymbol{\chi}),
\end{equation}
we obtain
\begin{equation}
\int_{\mathcal{X}}
\mathcal{L}_f(\boldsymbol{\chi})\,d\nu_T(\boldsymbol{\chi})
\ge
C_{\mathrm{depth}}
\int_{\mathcal{X}}
P_{\mathrm{implicit}}(\boldsymbol{\chi};L,\alpha_{\mathrm{circ}})
\,d\nu_T(\boldsymbol{\chi}).
\end{equation}
Finally, Assumption~\ref{assum:epsilon_relaxed} gives
\begin{equation}
\mathcal{R}_{\mu_T}(f)
\ge
\int_{\mathcal{X}}
\mathcal{L}_f(\boldsymbol{\chi})\,d\nu_T(\boldsymbol{\chi})
-\epsilon,
\end{equation}
which yields the claim.
\end{proof}

\subsubsection{Architecture-Invariance of the Physical Depth Limit}
\label{lemma:physical_depth_limit}

\begin{lemma}[Architecture-Invariance of the Physical Depth Limit]
\label{lem:physical-depth-limit}
The scalar $\alpha_{\mathrm{circ}}$ is determined by the implicit computational requirement of the backtracking problem and is therefore independent of the positional encoding scheme. In particular, $\alpha_{\mathrm{APE}}=\alpha_{\mathrm{RoPE}}=\alpha_{\mathrm{circ}}$.
\end{lemma}

\begin{proof}
A single autoregressive forward pass applies exactly $L$ sequential layers, regardless of whether the positional encoding is APE or RoPE. RoPE preserves relative position information (Appendix~\ref{app:proof-rope-equivariance}) but adds no sequential computational steps beyond the available physical depth. Therefore the same implicit depth threshold governs both positional-encoding choices.
\end{proof}

\section{Proof of the Main Theorem and Corollaries}
\label{app:main-results}

\subsection{Proof of Theorem~\ref{thm:unified_bound}}
\label{app:proof-main-theorem}

For completeness, we restate Theorem~\ref{thm:unified_bound}, which holds under Assumptions~\ref{assum:epsilon_relaxed}, \ref{ass:width-capacity} and~\ref{ass:implicit-complexity} and the Lipschitz continuity of $\mathcal{L}_f$ assumed in Lemma~\ref{lemma:error_decomposition}.
For any autoregressive Transformer hypothesis $f \in \mathcal{F}_{L,m,\mathrm{PE}}$, the target risk satisfies
\begin{equation}
\mathcal{R}_{\mu_T}(f)
\ge
\max\Bigg\{
\begin{aligned}
&\frac{C_{\mathrm{width}}}{\sqrt{m}}-\epsilon,\;
C_{\mathrm{depth}}
\int_{\mathcal{X}}
P_{\mathrm{implicit}}(\boldsymbol{\chi};L,\alpha_{\mathrm{circ}})
\,d\nu_T(\boldsymbol{\chi})
-\epsilon,\\
&\quad\mathcal{R}_{\mu_S}(f)-K_{f,\mathrm{PE}}W_1(\nu_S,\nu_T)-2\epsilon
\end{aligned}
\Bigg\}.
\end{equation}
In addition, $\mathcal{R}_{\mu_T}(f)\le\mathcal{R}_{\mu_S}(f)+K_{f,\mathrm{PE}}W_1(\nu_S,\nu_T)+2\epsilon$.

\begin{proof}

Lemmas~\ref{lemma:width_capacity} and~\ref{lemma:depth_bottleneck} give
\[
\mathcal{R}_{\mu_T}(f) \ge \frac{C_{\mathrm{width}}}{\sqrt{m}}-\epsilon
\quad\text{and}\quad
\mathcal{R}_{\mu_T}(f) \ge C_{\mathrm{depth}} \int_{\mathcal{X}} P_{\mathrm{implicit}}(\boldsymbol{\chi};L,\alpha_{\mathrm{circ}}) \,d\nu_T(\boldsymbol{\chi}) -\epsilon.
\]
Finally, Lemma~\ref{lemma:error_decomposition} implies $\left| \mathcal{R}_{\mu_T}(f)-\mathcal{R}_{\mu_S}(f) \right| \le K_{f,\mathrm{PE}}W_1(\nu_S,\nu_T)+2\epsilon$, and therefore $\mathcal{R}_{\mu_T}(f) \ge \mathcal{R}_{\mu_S}(f)-K_{f,\mathrm{PE}}W_1(\nu_S,\nu_T)-2\epsilon$. All three lower bounds hold for every $f \in \mathcal{F}_{L,m,\mathrm{PE}}$, so $\mathcal{R}_{\mu_T}(f)$ is bounded below by their maximum. The same inequality from Lemma~\ref{lemma:error_decomposition} gives the upper bound.
\end{proof}

\subsection{Proof of Corollary~\ref{cor:ape_vs_rope}}
\label{app:proof-cor-ape-rope}

For completeness, we restate Corollary~\ref{cor:ape_vs_rope}.
Assume the source risk satisfies $\mathcal{R}_{\mu_S}(f_{\mathrm{RoPE}})=\epsilon_{\mathrm{train}}\approx 0$, and suppose the target shift satisfies $W_1(\nu_S,\nu_T)=\Omega(T_{\max})$. Then the RoPE target risk degrades at most proportionally to the normalized shift: $\mathcal{R}_{\mu_T}(f_{\mathrm{RoPE}}) \le \epsilon_{\mathrm{train}} + C_{\mathrm{smooth}}\,W_1(\nu_S,\nu_T)/T_{\max} + 2\epsilon \le \epsilon_{\mathrm{train}} + C_{\mathrm{RoPE}} + 2\epsilon$, where $K_{f_{\mathrm{RoPE}}}\le C_{\mathrm{smooth}}/T_{\max}$ (Theorem~\ref{thm:lipschitz_gap}) and $C_{\mathrm{RoPE}}:=C_{\mathrm{smooth}}\sup_{T_{\max}}\operatorname{diam}(\mathcal{X})/T_{\max}<\infty$ does not depend on $T_{\max}$.

In contrast, suppose there exists a measurable subset $\mathcal{X}_{\mathrm{shift}}\subset\mathcal{X}$ such that $\nu_T(\mathcal{X}_{\mathrm{shift}})=p_{\mathrm{shift}}>0$ and $\mathcal{L}_{f_{\mathrm{APE}}}(\boldsymbol{\chi})\ge \gamma$ for all $\boldsymbol{\chi}\in\mathcal{X}_{\mathrm{shift}}$, where $\gamma=\Omega(1)$. Then $\mathcal{R}_{\mu_T}(f_{\mathrm{APE}}) \ge \gamma p_{\mathrm{shift}}-\epsilon$. In particular, if $\epsilon \ll \gamma p_{\mathrm{shift}}$, the APE target risk admits a non-vanishing lower bound.

\begin{proof}
For RoPE, Lemma~\ref{lemma:error_decomposition} gives
\begin{equation}
\mathcal{R}_{\mu_T}(f_{\mathrm{RoPE}})
\le
\mathcal{R}_{\mu_S}(f_{\mathrm{RoPE}})
+
K_{f_{\mathrm{RoPE}}}W_1(\nu_S,\nu_T)
+
2\epsilon.
\end{equation}
By Theorem~\ref{thm:lipschitz_gap}, $K_{f_{\mathrm{RoPE}}}\le C_{\mathrm{smooth}}/T_{\max}$, so $K_{f_{\mathrm{RoPE}}}W_1(\nu_S,\nu_T)\le C_{\mathrm{smooth}}\,W_1(\nu_S,\nu_T)/T_{\max}$. Since $W_1(\nu_S,\nu_T)$ is bounded above by $\operatorname{diam}(\mathcal{X})=\Theta(T_{\max})$, this term is at most $C_{\mathrm{RoPE}}$. Substituting \(\mathcal{R}_{\mu_S}(f_{\mathrm{RoPE}})=\epsilon_{\mathrm{train}}\) gives both stated bounds.

For APE, Assumption~\ref{assum:epsilon_relaxed} yields
\begin{equation}
\mathcal{R}_{\mu_T}(f_{\mathrm{APE}})
\ge
\int_{\mathcal{X}}
\mathcal{L}_{f_{\mathrm{APE}}}(\boldsymbol{\chi})
\,d\nu_T(\boldsymbol{\chi})
-\epsilon.
\end{equation}
Restricting the domain of integration to $\mathcal{X}_{\mathrm{shift}}$ and using the assumed lower bound on $\mathcal{L}_{f_{\mathrm{APE}}}$ gives
\begin{align}
\mathcal{R}_{\mu_T}(f_{\mathrm{APE}})
&\ge
\int_{\mathcal{X}_{\mathrm{shift}}}
\mathcal{L}_{f_{\mathrm{APE}}}(\boldsymbol{\chi})
\,d\nu_T(\boldsymbol{\chi})
-\epsilon \\
&\ge
\int_{\mathcal{X}_{\mathrm{shift}}}
\gamma\,d\nu_T(\boldsymbol{\chi})
-\epsilon \\
&=
\gamma\,\nu_T(\mathcal{X}_{\mathrm{shift}})-\epsilon \\
&=
\gamma p_{\mathrm{shift}}-\epsilon.
\end{align}
\end{proof}

\subsection{Proof of Corollary~\ref{cor:deep_is_better}}
\label{app:proof-cor-deep-better}

For completeness, we restate Corollary~\ref{cor:deep_is_better}.
Suppose the target measure assigns positive mass to trajectories whose structural depth exceeds the available implicit computation threshold: $p_{\mathrm{deep}} := \nu_T\!( \{\boldsymbol{\chi}\in\mathcal{X}:\chi_1>\alpha_{\mathrm{circ}}L\} ) > 0$. Then the depth bottleneck contributes the non-vanishing lower bound $\mathcal{R}_{\mu_T}(f)\ge C_{\mathrm{depth}}\,p_{\mathrm{deep}}-\epsilon$. In contrast, the width-based term $\frac{C_{\mathrm{width}}}{\sqrt m}-\epsilon$ vanishes as $m\to\infty$. If $L \ge \frac{1}{\alpha_{\mathrm{circ}}} \sup_{\boldsymbol{\chi}\in\operatorname{supp}(\nu_T)} \chi_1$, then $\nu_T\!( \{\boldsymbol{\chi}\in\mathcal{X}:\chi_1>\alpha_{\mathrm{circ}}L\} ) = 0$, so the depth bottleneck term is eliminated on the support of $\nu_T$.

\begin{proof}
By Lemma~\ref{lemma:depth_bottleneck},
\begin{equation}
\mathcal{R}_{\mu_T}(f)
\ge
C_{\mathrm{depth}}
\int_{\mathcal{X}}
P_{\mathrm{implicit}}(\boldsymbol{\chi};L,\alpha_{\mathrm{circ}})
\,d\nu_T(\boldsymbol{\chi})
-\epsilon.
\end{equation}
Since $P_{\mathrm{implicit}}(\boldsymbol{\chi};L,\alpha_{\mathrm{circ}}) = \mathbb{I}_{(\alpha_{\mathrm{circ}}L,\infty)}(\chi_1)$, we have
\begin{equation}
\int_{\mathcal{X}} P_{\mathrm{implicit}}(\boldsymbol{\chi};L,\alpha_{\mathrm{circ}}) \,d\nu_T(\boldsymbol{\chi}) = \nu_T\!( \{\boldsymbol{\chi}\in\mathcal{X}:\chi_1>\alpha_{\mathrm{circ}}L\} ) = p_{\mathrm{deep}},
\end{equation}
which gives $\mathcal{R}_{\mu_T}(f)\ge C_{\mathrm{depth}}p_{\mathrm{deep}}-\epsilon$.

On the other hand,
\begin{equation}
\lim_{m\to\infty} \left( \frac{C_{\mathrm{width}}}{\sqrt m}-\epsilon \right) = -\epsilon \le 0,
\end{equation}
so the width-only contribution cannot sustain a non-vanishing lower bound.

Finally, if $L \ge \frac{1}{\alpha_{\mathrm{circ}}} \sup_{\boldsymbol{\chi}\in\operatorname{supp}(\nu_T)} \chi_1$, then no point in $\operatorname{supp}(\nu_T)$ satisfies $\chi_1>\alpha_{\mathrm{circ}}L$. Hence $\nu_T\!( \{\boldsymbol{\chi}\in\mathcal{X}:\chi_1>\alpha_{\mathrm{circ}}L\} ) = 0$, which eliminates the depth bottleneck term on the target support.
\end{proof}

\section{Algebraic Equivalence Classes of Positional Encodings}
\label{app:pe_equivalence}

Our analysis takes absolute positional encoding (APE) and rotary position embedding (RoPE) as representatives of two algebraic equivalence classes. This appendix places ALiBi \citep{press2022trainshorttestlong} and T5 Relative Bias in the shift-invariant class and treats no positional encoding (NoPE) \citep{kazemnejad2023impactpositionalencodinglength} at the end.

For the structural Lipschitz gap (Theorem \ref{thm:lipschitz_gap}), the generalizability of an attention mechanism under sequence coordinate translations ($\mathcal{T}_k$) is determined by whether its attention kernel $A_{i,j}$ satisfies shift invariance. Positional encodings partition the architectural hypothesis space into two distinct algebraic equivalence classes:

\textbf{1. Position-Dependent Attention Kernels:} This class contains mechanisms where $A_{i,j}$ depends explicitly on the absolute indices $i$ and $j$. Standard learned APE and sinusoidal encodings belong to this class. As proven in Lemma \ref{lemma:ape_shattering} for learned APE, members of this equivalence class fail to preserve shift invariance, so the computational graph differs under translations. Under Assumption~\ref{assum:ar_sensitivity}, this yields a structural Lipschitz constant $K_{f} = \Omega(1)$. 

\textbf{2. Shift-Invariant Attention Mechanisms:} This class contains mechanisms where the attention kernel depends only on the relative distance $i - j$. RoPE, ALiBi, and T5 Relative Bias all satisfy this constraint. For any mechanism in this equivalence class, the translation operator therefore gives $\Delta A = 0$, so the mechanism inherits the baseline structural smoothness $\mathcal{O}(1/T_{\max})$ established in Theorem \ref{thm:lipschitz_gap}.

Finally, no positional encoding (NoPE) injects no positional vectors, so its kernel satisfies $\Delta A^{\text{NoPE}} \equiv 0$ trivially, and the model recovers structural routing information implicitly through causal masking. \citet{kazemnejad2023impactpositionalencodinglength} show that such implicit recovery is possible (the first layer can reconstruct absolute positions) and that NoPE length-generalizes on par with or better than explicit encodings in their experiments. Its behavior under translations therefore depends on the positions recovered inside the network.

\section{Estimation of Structural Wasserstein Distances}
\label{app:w1_estimation}

The theoretical bounds in Theorem~\ref{thm:unified_bound} are stated in terms of the exact Wasserstein-1 distance $W_1(\nu_S,\nu_T)$ between structural pushforward measures. 
In the empirical analysis, we estimate this quantity from finite samples by solving the exact (unregularized) optimal transport problem between empirical measures. %
We denote these finite-sample estimates by $\widehat W_{1}$ to distinguish them from the population-level $W_1$.

\subsection{Computational Pipeline}

Let $\hat{\nu}_S$ and $\hat{\nu}_T$ denote the empirical structural measures obtained by sampling trajectories from the source and target distributions and applying the structural projection $\Phi$. 
The estimation pipeline consists of the following steps.

\paragraph{Joint standardization.}
The structural coordinates have different numerical scales; for example, the search-volume coordinate $\chi_4$ can be orders of magnitude larger than transition-rate coordinates such as $\chi_5,\chi_6,\chi_7$. 
To prevent the Euclidean cost from being dominated by large-scale coordinates, we apply Z-score normalization to each coordinate. 
The normalization statistics are computed once, on the pooled structural vectors of all three measures, so all pairs share one scale and their distances are comparable.

\paragraph{Cost matrix construction.}
For standardized structural vectors $\boldsymbol{u}_i,\boldsymbol{v}_j\in\mathbb{R}^{12}$, we define the empirical cost matrix by
\begin{equation}
M_{ij}=\|\boldsymbol{u}_i-\boldsymbol{v}_j\|_2.
\end{equation}
For the entropically regularized cross-check below, we additionally rescale the cost matrix as
\begin{equation}
\widetilde M_{ij}
=
\frac{M_{ij}}{\max_{i,j} M_{ij}+10^{-9}}.
\end{equation}
This places the costs of the cross-check in $[0,1]$.

\paragraph{Exact transport.}
With uniform marginal weights $\mathbf{p}=\mathbf{q}=N^{-1}\mathbf{1}$, we compute $\widehat W_{1}(\hat{\nu}_S,\hat{\nu}_T)=\min_{\mathbf{P}\in\Pi(\mathbf{p},\mathbf{q})}\langle \mathbf{P},\mathbf{M}\rangle_F$ with the network-simplex solver of the Python Optimal Transport library \citep{flamary2021pot}. As a cross-check, the debiased Sinkhorn divergence (Sinkhorn--Knopp with $\lambda=0.1$ on the rescaled cost $\widetilde{\mathbf{M}}$, \citealp{cuturi2013sinkhorndistanceslightspeedcomputation}) preserves the separation reported below (within-measure $<$ BFS--MIXED $<$ BFS--DFS); the debiasing removes the entropic bias that dominates the uncorrected regularized cost at this sample size.

\paragraph{Subsampling.}
Each training measure contains 500,000 trajectories, so building the all-pairs cost matrix on a full measure is expensive. 
For each distribution, we first form a candidate pool of 20,000 trajectories and then sample $N=2{,}000$ trajectories without replacement for each pairwise distance computation. 
We repeat this procedure over $10$ independent subsamples with fixed seeds and report the mean $\pm$ standard deviation.

\subsection{Estimates and Interpretation}

Across the $10$ subsamples, the within-measure comparison (two disjoint samples of the same measure) remains small, while the MIXED and DFS target distributions produce much larger shifts:
\begin{gather*}
\widehat W_{1}(\hat{\nu}_{BFS},\hat{\nu}_{BFS}) = 0.19 \pm 0.06,\qquad
\widehat W_{1}(\hat{\nu}_{BFS},\hat{\nu}_{MIXED}) = 1.38 \pm 0.05,\\
\widehat W_{1}(\hat{\nu}_{BFS},\hat{\nu}_{DFS}) = 2.97 \pm 0.05,
\end{gather*}
and $\widehat W_{1}(\hat{\nu}_{DFS},\hat{\nu}_{DFS}) = 0.28 \pm 0.05$, $\widehat W_{1}(\hat{\nu}_{DFS},\hat{\nu}_{MIXED}) = 1.76 \pm 0.06$.
The separation between the three measures, with BFS--MIXED and BFS--DFS distances $7$ and $16$ times the within-BFS baseline, shows that the structural feature space $\mathcal{X}$ captures the intended differences in search behavior between BFS, MIXED, and DFS trajectory distributions.

\subsection{Empirical Calibration of the Theoretical Ceiling}
\label{app:empirical_calibration}

Figure \ref{fig:analytical_grid}A plots the empirical accuracy of the evaluated architectures against the $\text{TC}^0$ depth capacity bottleneck. We now describe how the ``Theoretical Ceiling'' is motivated by Lemma~\ref{lemma:depth_bottleneck} and calibrated. %

By Theorem \ref{thm:unified_bound}, the expected target risk on the DFS evaluation measure ($\mu_{DFS}$) is lower-bounded by the depth penalty: 
\begin{equation}
\mathcal{R}_{\mu_{DFS}}(f) \ge C_{\mathrm{depth}} \int_{\mathcal{X}} P_{\mathrm{implicit}}(\boldsymbol{\chi}; L, \alpha_{\mathrm{circ}}) \, d\nu_{DFS}(\boldsymbol{\chi}) - \epsilon.
\end{equation}
Because empirical accuracy is defined as $1 - \mathcal{R}_{\mu_T}(f)$, this expected risk lower bound gives an empirical accuracy upper bound (the theoretical ceiling). Since $\chi_1$ takes finitely many values (Appendix~\ref{app:trajectory-space}), the truncated mass $\int_{\mathcal{X}} P_{\mathrm{implicit}}\,d\nu_{DFS}$ is a non-increasing step function of the physical layer depth $L$, and we summarize its decay over the evaluated depths by a smooth exponential profile. On this task $\chi_1\le 4$ for every reference trajectory (Appendix~\ref{app:eval_scope}), so the curve below is a calibrated envelope, drawn for the DFS evaluation set, the hardest of the three.

We formulate the theoretical ceiling as:
\begin{equation}
    \text{Accuracy}_{\max}(L) = \mathcal{A}_{\max} - \beta \cdot \exp(-\alpha \cdot L)
\end{equation}
where $\beta, \alpha > 0$ are calibration constants and $\mathcal{A}_{\max}$ is the asymptotic empirical accuracy. The risk bound above constrains only the depth-dependent term and gives $\mathcal{A}_{\max}\le 1+\epsilon$, so we calibrate $\mathcal{A}_{\max}$ empirically, and it absorbs the failure sources outside the depth term. One of them is the generation budget: within the $512$ new tokens allowed at evaluation (Appendix~\ref{app:reproducibility}), the reference trajectories themselves reach the target for $0.4645$, $0.4860$ and $0.4915$ of the BFS, DFS and MIXED evaluation problems; most trained models, which are scored on their own searches, exceed these rates (up to $0.5360$ on the DFS set), so the budget is a loose constraint on $\mathcal{A}_{\max}$.

We calibrated the constants $(\mathcal{A}_{\max} \approx 0.538, \beta \approx 0.035, \alpha \approx 0.15)$ so that the curve is a strict upper envelope of all 54 evaluated models. As shown in Figure \ref{fig:analytical_grid}A, every empirical data point lies strictly below this ceiling. The width conclusion rests on measured accuracies at fixed depth, independent of the fitted constants. Taking the best of the six models of each architecture on $\mu_{DFS}$ (three training measures, APE and RoPE), the gain from width falls below one point at the largest widths: at $L=12$, widening from $d=192$ to $d=768$ adds $1.55$ points ($0.5160\to0.5315$), but widening further to $d=1536$, which quadruples the core parameters ($84.9\,\mathrm{M}\to339.7\,\mathrm{M}$), changes accuracy by $-0.2$ points ($0.5295$); at $L=24$, doubling $d$ from $512$ to $1024$ adds $0.6$ points ($0.5280\to0.5340$); and at $L=6$, quadrupling $d$ from $256$ to $1024$ adds $0.95$ points ($0.5130\to0.5225$). Model by model, the two doublings that start from $d\ge512$ ($L=12$ and $L=24$) lower APE accuracy on $\mu_{DFS}$ for every training measure (by $0.25$--$9.7$ points) and change RoPE accuracy by $-0.3$ to $+1.3$ points. This confirms that, over the evaluated range, width scaling does not raise the ceiling.

\section{Supplementary Empirical Analyses}
\label{app:Supp Fig}

\subsection{Mitigation by Translation Equivariance}

In Section~\ref{sec:experiments}, we stated that rotary position embedding (RoPE), through its structural smoothness, is a continuous structural regularizer that removes the extra depth APE requires at Large scale on recursive backtracking ($\mu_{DFS}$). Figure~\ref{fig:supp_rope_mitigation} isolates this effect: it plots the accuracy of the Wide, Balanced and Deep variants of Large- and Medium-scale models trained and evaluated on $\mu_{DFS}$ only. 

\begin{figure}[h]
    \centering
    \includegraphics[width=0.95\textwidth]{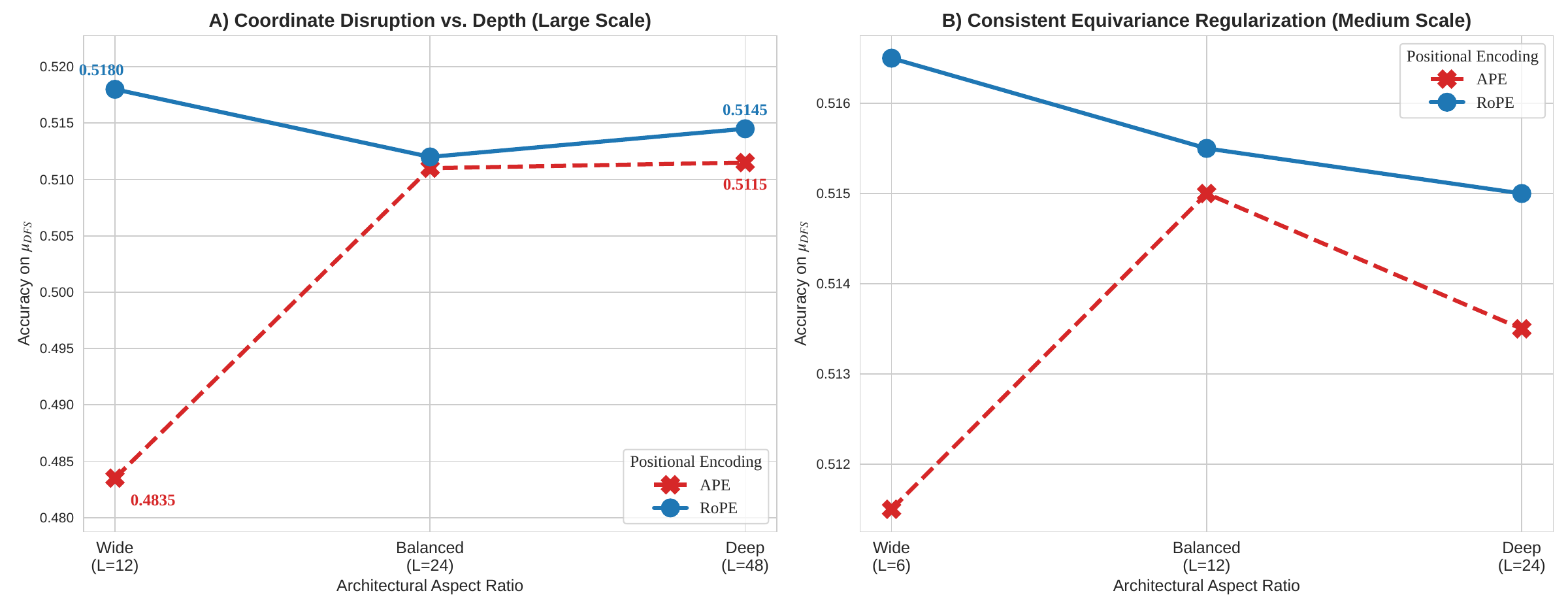}
    \caption{Empirical verification of structural regularization via translation equivariance on $\mu_{DFS}$. \textbf{(A) Large Scale Models:} Under absolute positional encoding (APE, red dashed line), increasing physical layer depth ($L=12 \to L\ge 24$) raises accuracy from $0.4835$ to $0.5110$--$0.5115$. Because $\alpha_{\mathrm{circ}}$ is the same for both encodings (Lemma~\ref{lem:physical-depth-limit}) and Wide RoPE (blue solid line, $L=12$) achieves $0.5180$, 12 layers suffice on this task. We therefore attribute the failure of Wide APE ($L=12$, $0.4835$) to its coordinate disruption (Lemma~\ref{lemma:ape_shattering}, Assumption~\ref{assum:ar_sensitivity}): at this scale APE needs additional physical depth ($L\ge 24$) to memorize absolute-position-dependent routing, whereas RoPE, which preserves translation equivariance, succeeds with 12 layers.
    \textbf{(B) Medium Scale Models:} The same comparison at $L=6$, $12$ and $24$ (Wide, Balanced, Deep): APE reaches $0.5115$, $0.5150$ and $0.5135$, and RoPE $0.5165$, $0.5155$ and $0.5150$, so RoPE lies above APE for every shape, by at most $0.5$ points.
    }
    \label{fig:supp_rope_mitigation}
\end{figure}

At Large scale (Panel A), APE accuracy (red dashed line) rises monotonically with depth, from $0.4835$ (Wide, $L=12$) to $0.5110$ ($L=24$) and $0.5115$ ($L=48$); the same monotone rise holds at this scale for every combination of training measure and evaluation set. Twelve layers nonetheless suffice on this task: Wide RoPE reaches $0.5180$ at the same $L=12$, and the Medium APE models of Panel B reach $0.5115$ at $L=6$ and $0.5150$ at $L=12$. The added depth thus compensates for a failure specific to the widest APE model, which the next paragraph attributes to coordinate disruption.

In contrast, RoPE flattens this curve. As proven in Lemma \ref{lem:physical-depth-limit}, RoPE and APE have identical theoretical expressivity limits ($\alpha_{\mathrm{circ}}$). However, without relative equivariance, APE architectures are limited by coordinate disruption (Assumption~\ref{assum:ar_sensitivity}; Corollary~\ref{cor:ape_vs_rope}) before reaching their capacity, so that at Large scale APE needs greater depths ($L\ge 24$) to provide implicit routing redundancy. Because RoPE preserves structural shift invariance, it keeps its structural risk $\mathcal{L}_{f_{\mathrm{RoPE}}}$ smooth over $\mathcal{X}$ (Theorem~\ref{thm:lipschitz_gap}) and avoids these coordinate disruptions. The Large Wide DFS RoPE configuration achieves $0.5180$, on par with $0.5145$ for its Deep counterpart. Translation equivariance is thus a structural regularizer that lets the architecture reach its expressivity limits without premature collapse.

\subsection{Baseline Capacity and Structural Variance}
\label{app:baseline_variance}

Figure \ref{fig:app_baseline_variance} gives an aggregate view of in-distribution performance ($\nu_T = \nu_S$) across all three parameter scales and training measures. 

Each bar averages the accuracy of the Deep, Wide and Balanced variants, which have nearly equal core parameter counts, for one scale, training measure and positional encoding. Its standard-deviation error bar, the \textit{structural variance}, therefore measures how sensitive accuracy is to the aspect ratio of the computational graph.

\begin{figure}[h]
    \centering
    \includegraphics[width=\textwidth]{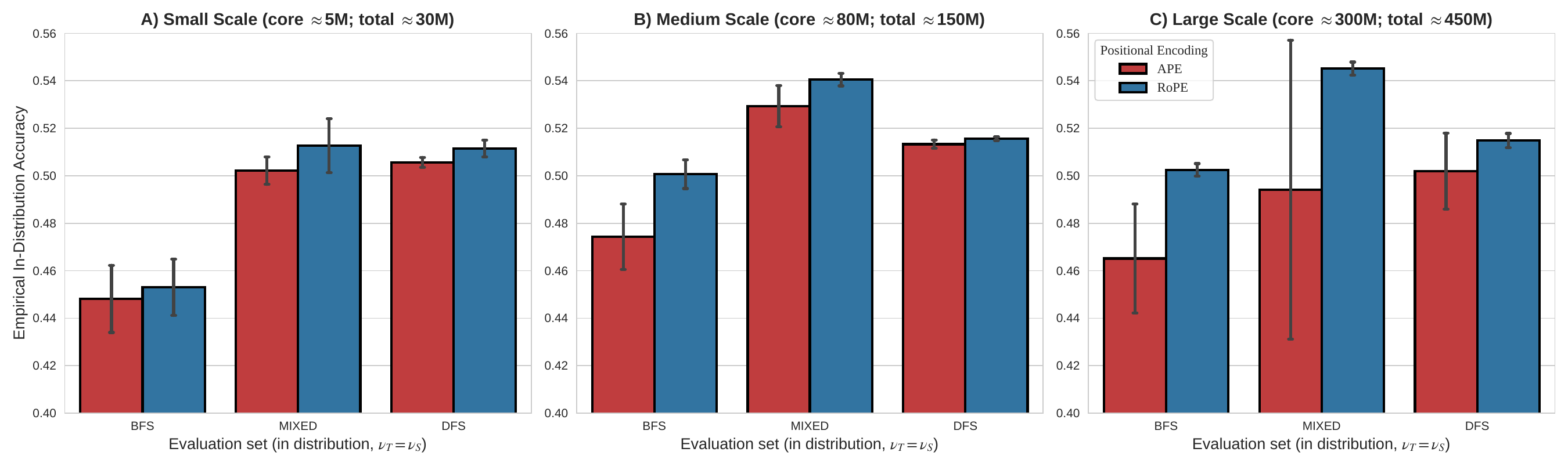}
    \caption{Baseline performance and structural variance across the experimental grid at Small (A), Medium (B) and Large (C) scale. For in-distribution evaluation ($\nu_T = \nu_S$), the bars represent the mean empirical accuracy across different architectural aspect ratios (Deep, Wide, Balanced), while the error bars denote $\pm 1$ standard deviation. The large variance under APE at the Large scale (largest on $\mu_{MIXED}$) shows that aspect ratio dominates performance, whereas RoPE reduces this variance at the Medium and Large scales, as expected of a structural regularizer.}
    \label{fig:app_baseline_variance}
\end{figure}

In Figure~\ref{fig:app_baseline_variance}, the structural variance under absolute positional encoding (APE, red bars) is high, particularly at the Large scale, where it is driven by the Wide architecture (most visibly its collapse under $\mu_{MIXED}$, Figure~\ref{fig:analytical_grid}C). Under absolute sequence coordinates, accuracy at a matched parameter budget thus depends strongly on the aspect ratio: at Large scale, APE accuracy rises with depth for every training measure, by up to $11.15$ points in distribution.

By contrast, rotary position embedding (RoPE, blue bars) has smaller error bars at the Medium and Large scales. Its translation equivariance reduces the architecture's sensitivity to the discrete layout of the computational graph. In line with our theory, preserving local structural equivariance makes RoPE a continuous structural regularizer: it raises the lowest accuracy across depth-to-width ratios in eight of the nine scale--measure cells (all but Small-scale BFS, where the two differ by $0.05$ points).

\section{Further Intuition on the Bounds} %
\label{app:intuition}

\paragraph{What bounds the accuracy of RoPE models.}
RoPE removes the APE-specific coordinate-disruption failure (Corollary~\ref{cor:ape_vs_rope}) and, under a measure shift, controls the Lipschitz penalty $K_{f,\mathrm{PE}}\,W_1(\nu_S,\nu_T)$ of Theorem~\ref{thm:unified_bound}; the two capacity terms, the depth term $C_{\mathrm{depth}}\int P_{\mathrm{implicit}}\,d\nu_T$ and the width term $C_{\mathrm{width}}/\sqrt{m}$, are the same for both encodings. A shift-invariant architecture therefore avoids the collapse caused by coordinate shifts, while its absolute accuracy remains bounded by the depth and width terms and by the residual $\epsilon$. Accordingly, the recovery from $0.4215$ (Large Wide MIXED APE) to $0.5425$ (its RoPE counterpart) brings the model to the level of the other Large architectures. That level is saturated in depth and width over the evaluated range: on each evaluation set the best accuracies at $L=12$, $24$ and $48$ lie within one point of each other, and the width comparison of Appendix~\ref{app:empirical_calibration} shows gains below one point at the largest widths.

\paragraph{Controlling $\epsilon$ and the dimension of $\mathcal{X}$.}
Enlarging $\mathcal{X}$ with finer descriptors (e.g., token identities, $n$-gram statistics, operand values) would drive the sufficiency residual $\epsilon$ of Assumption~\ref{assum:epsilon_relaxed} toward zero, at the price of a sample complexity for estimating $W_1$ that grows rapidly with the dimension \citep{weed2019sharp}. The twelve structural coordinates balance a residual $\epsilon$ against stable estimation of the structural shift (Appendix~\ref{app:w1_estimation}).

\paragraph{How $K_f$ enters the experiments.}
$K_f$ enters through its theoretical bounds: $K_{f,\mathrm{APE}}=\Omega(1)$ and $K_{f,\mathrm{RoPE}}=\mathcal{O}(1/T_{\max})$ follow from Assumption~\ref{assum:ar_sensitivity} (Lemma~\ref{lemma:ape_shattering} and Theorem~\ref{thm:lipschitz_gap}), and Section~\ref{sec:experiments} examines their empirical consequences across architectures and training measures.

\section{Empirical Check of Structural Sufficiency (\texorpdfstring{Assumption~\ref{assum:epsilon_relaxed}}{Assumption 1})} %
\label{app:gbc_check}
We check Assumption~\ref{assum:epsilon_relaxed} for two trained models by predicting their per-problem loss from $\boldsymbol{\chi}$, within and across the three search measures. All metrics are computed out of fold on freshly generated trajectories ($50{,}000$ for one model and $3{,}000$ of them for the other) and compared with explicit baselines.

\textbf{Setup.} The label is the zero-one loss $\ell(f,z)$ of Appendix~\ref{app:risk}: whether model $f$, decoding greedily with the settings of Appendix~\ref{app:reproducibility}, fails on the problem $x_{\mathrm{init}}^{(z)}$ of the reference trajectory $z$. The model is prompted with this problem only, so $\ell$ is a deterministic function of it (repeated prompts always received the same label). Data: $50{,}000$ newly generated reference trajectories (generator seed $2026$; $16{,}667/16{,}667/16{,}666$ from $\mu_{BFS},\mu_{DFS},\mu_{MIXED}$). Of their problems, $37.7\%$ ($39.3\%$ of the subset used for $f_2$) also occur in at least one training set, almost all of them with $3$ operands; Table~\ref{tab:gbc} therefore also reports the problems absent from all training sets. Models: Small Deep MIXED RoPE ($f_1$, all $50{,}000$ trajectories) and Large Deep BFS APE ($f_2$, the first $1{,}000$ per measure), two of the $54$ configurations of Section~\ref{sec:experiments}, which differ in scale, training measure and positional encoding. Every input is computed from the reference trajectory $z$, never from the model's output: $\boldsymbol{\chi}=\Phi(z)$; $\boldsymbol{\chi}$ without the reference solution indicator $\chi_{10}$; the token length of $z$ (LEN); the problem alone (PROB: target, operand count and sorted operands, i.e., the prompt up to operand order); and the operand count alone. $\hat{\mathcal{L}}$ is scikit-learn's \texttt{HistGradientBoostingClassifier} \citep{pedregosa2011scikit} with default settings, evaluated out of fold under stratified 5-fold cross-validation. McFadden's pseudo-$R^2$ is $1-\log L_{\hat{\mathcal{L}}}/\log L_0$, where $L_0$ is the likelihood of a null model that predicts the training-fold base rate \citep{mcfadden1979quantitative}. The code and the per-problem labels are available on GitHub (Appendix~\ref{app:code}). The classifier fits $P(\ell{=}0)$, as in Table~\ref{tab:gbc}; with $P(\ell{=}1)$ the fits differ only through the default early stopping, which is active for the training sets of $f_1$ (more than $10{,}000$ samples) and draws a validation split that depends on the coding. Switching the coding changes every reported value by at most $0.003$, except the risk estimates of $f_1$ classifiers fitted on BFS or DFS trajectories alone (up to $0.04$), which we report below under $P(\ell{=}1)$, with the value under $P(\ell{=}0)$ in parentheses.

\begin{table}[h]
\centering\small
\caption{Per-trajectory prediction of $\ell$ by $\hat{\mathcal{L}}$: out-of-fold ROC-AUC and McFadden pseudo-$R^2$ (mean over five folds; $\pm$ population s.d.\ over folds for $\boldsymbol{\chi}$). ``Not in training'': problems absent from all training sets ($n=31{,}173$ for $f_1$ and $1{,}822$ for $f_2$; mostly with $4$ operands). ``Within $\mu$'': cross-validated within the trajectories of one search measure ($16{,}667/16{,}667/16{,}666$ for $f_1$, $1{,}000$ each for $f_2$).}
\label{tab:gbc}
\begin{tabular}{lcccc}
\toprule
 & \multicolumn{2}{c}{$f_1$ ($n=50{,}000$)} & \multicolumn{2}{c}{$f_2$ ($n=3{,}000$)} \\
\cmidrule(lr){2-3}\cmidrule(lr){4-5}
Input & ROC-AUC & McFadden & ROC-AUC & McFadden \\
\midrule
$\boldsymbol{\chi}$ & $0.927\pm0.002$ & $0.51\pm0.01$ & $0.929\pm0.007$ & $0.48\pm0.03$ \\
$\boldsymbol{\chi}$ without $\chi_{10}$ & $0.926$ & $0.50$ & $0.927$ & $0.47$ \\
LEN & $0.903$ & $0.44$ & $0.917$ & $0.46$ \\
PROB & $0.890$ & $0.40$ & $0.906$ & $0.41$ \\
Operand count only & $0.797$ & $0.27$ & $0.853$ & $0.40$ \\
\midrule
$\boldsymbol{\chi}$, not in training & $0.905$ & $0.44$ & $0.910$ & $0.36$ \\
$\boldsymbol{\chi}$, within $\mu_{BFS}$ & $0.905$ & $0.44$ & $0.940$ & $0.49$ \\
$\boldsymbol{\chi}$, within $\mu_{DFS}$ & $0.952$ & $0.60$ & $0.912$ & $0.39$ \\
$\boldsymbol{\chi}$, within $\mu_{MIXED}$ & $0.918$ & $0.48$ & $0.919$ & $0.41$ \\
\bottomrule
\end{tabular}
\end{table}

\textbf{Per trajectory.} Out of fold, $\boldsymbol{\chi}$ predicts $\ell$ with ROC-AUC $0.927$ and McFadden pseudo-$R^2$ $0.51$ for $f_1$ ($0.929$ and $0.48$ for $f_2$); removing $\chi_{10}$ changes little. McFadden's index runs considerably lower than the $R^2$ of linear regression, and values of $0.2$--$0.4$ already represent an excellent fit \citep{mcfadden1979quantitative}. Every input in Table~\ref{tab:gbc} lies in or above that range, and among the inputs evaluated on the full sample, $\boldsymbol{\chi}$ scores highest for both models. For $f_1$, $\boldsymbol{\chi}$ exceeds LEN and PROB on both metrics in every fold. For $f_2$, it exceeds them in ROC-AUC in every fold and in mean McFadden value ($0.48$ against $0.41$ for PROB and $0.46$ for LEN, which it exceeds in three of five folds). Problem size carries a large part of the signal (accuracy is about $0.8$ with $3$ operands and $0.1$--$0.2$ with $4$), and $\boldsymbol{\chi}$ goes beyond it: ROC-AUC $0.927$ against $0.797$ and McFadden $0.51$ against $0.27$ for the operand count alone ($f_1$), and $0.929$ against $0.853$ and $0.48$ against $0.40$ ($f_2$). On problems absent from all training sets, $\boldsymbol{\chi}$ remains strongly predictive (ROC-AUC $0.905$ and $0.910$, McFadden $0.44$ and $0.36$ for $f_1$ and $f_2$). Within a single search measure, $\boldsymbol{\chi}$ reaches ROC-AUC $0.905$--$0.952$ ($f_1$) and $0.912$--$0.940$ ($f_2$). PROB is nearly the whole prompt, which determines $\ell$, yet scores below $\boldsymbol{\chi}$: the classifier extracts more predictive signal from the structural coordinates than from the problem itself.

\textbf{Across search measures.} The three search measures share one problem distribution and differ only in the reference search trace, which the model never sees (Appendix~\ref{app:eval_scope}). On freshly generated problems, a fixed model therefore has the same risk under all three measures up to sampling noise ($0.502/0.494/0.500$ for $f_1$, $0.511/0.511/0.490$ for $f_2$). Assumption~\ref{assum:epsilon_relaxed} thus holds for every pair of these measures with $\epsilon$ equal to the largest gap between a measure's risk and the pooled risk ($0.004$ for $f_1$, $0.014$ for $f_2$), attained by the constant $\mathcal{L}_f$ equal to the pooled risk; a constant fitted on one measure misestimates the risk of another by at most $0.007$ ($f_1$) and $0.021$ ($f_2$). A transfer analysis fits $\hat{\mathcal{L}}$ on all trajectories of one search measure and applies it to another: it ranks the target trajectories with ROC-AUC $0.87$--$0.95$, and its mean prediction matches the target risk within $0.043$ ($0.061$) for $f_1$ and $0.038$ for $f_2$. We take this transfer error as the empirical counterpart of $\epsilon$; it measures how stable the $\boldsymbol{\chi}$-conditional risk is across search procedures.

\textbf{Sub-populations.} We group the trajectories into $54$ ($f_1$) and $18$ ($f_2$) cells of search measure $\times$ operand count $\times$ target band; unlike the three search measures, cells that differ in operand count or target band pose different problems. In each cell we compare the observed risk with the mean cross-fitted prediction of $\hat{\mathcal{L}}$ on $\boldsymbol{\chi}$, fitted on trajectories from all cells: the absolute deviation averages $0.031$ (maximum $0.114$) for $f_1$ and $0.038$ (maximum $0.081$) for $f_2$, against $0.297$ and $0.353$ for a constant (LEN: $0.030$ and $0.048$; PROB: $0.010$ and $0.016$). If $\hat{\mathcal{L}}$ were perfectly calibrated (labels simulated from its own predictions), the largest deviation would stay below $0.037$ ($f_1$) and $0.070$ ($f_2$) in $95\%$ of samples; both observed maxima exceed these levels. Fitted on a single search measure and applied to the cells of the other two, $\hat{\mathcal{L}}$ deviates from the cell risk by at most $0.196$ ($0.237$) for $f_1$ and $0.116$ for $f_2$. Across cells, the cross-fitted predictions of $\hat{\mathcal{L}}$ on $\boldsymbol{\chi}$ reach a coefficient of determination of $0.983/0.985$ ($f_1/f_2$) for the observed risks, against about $0$ for a constant; because the cells are defined by problem difficulty, simpler summaries reach comparable or higher values ($0.986/0.977$ for LEN, $0.998/0.996$ for PROB, $0.969/0.964$ for the operand count alone).

\textbf{Conclusion.} $\boldsymbol{\chi}$ is a strong predictor of per-problem success: by ROC-AUC it ranks above LEN, PROB and the operand count for both models, and it stays predictive on problems absent from all training sets. Across the three search measures, which share one problem distribution, a fixed model has the same risk up to sampling noise, so Assumption~\ref{assum:epsilon_relaxed} holds with $\epsilon=0.004$ ($f_1$) and $0.014$ ($f_2$).

\section{Experimental Details and Reproducibility}
\label{app:reproducibility}

This appendix gives the compute environment, model implementations, training hyperparameters, task format and evaluation procedure for the 54 Transformer configurations of Section~\ref{sec:experiments}.

\subsection{Compute Environment}
Dataset generation, model training and evaluation ran on nodes with 4 or 8 NVIDIA A100 (40 GB) GPUs in a Linux cluster managed by Slurm. %

Each training run was allocated 1 NVIDIA A100 GPU, 16 CPU cores, and 128 GB of system RAM, with a wall-clock limit of 7 days; training took 2.7--4.6\,h (Small), 8.2--21.9\,h (Medium) and 18.3--58.7\,h (Large), about 907 GPU-hours in total. The evaluation pipeline was allocated 1 NVIDIA A100 GPU, 24 CPU cores, and 128 GB of system RAM (about 324 GPU-hours in total). Labeling the problems of Appendix~\ref{app:gbc_check} by greedy decoding took about 17 further GPU-hours on the same hardware (single-GPU jobs with 8 CPU cores and 48 GB of RAM); the remaining computations (generating the problems of Appendix~\ref{app:gbc_check}, the $\widehat W_1$ estimates of Appendix~\ref{app:w1_estimation} and the analyses of Appendix~\ref{app:gbc_check}) used CPUs only. For Appendix~\ref{app:beyond}, Experiment~3 took about 23.5 GPU-hours on NVIDIA A100 (80 GB) GPUs, with grading on CPUs, Experiment~4 about 0.45 GPU-hours on NVIDIA H100 (80 GB) GPUs, Experiment~5 about 106 GPU-hours on NVIDIA H100 and L40S GPUs, Experiment~6 about 4.5 GPU-hours on NVIDIA A100 (80 GB) and H100 GPUs, and Experiments~7 and~8 about 119 and 13 GPU-hours on NVIDIA A100 and L40S GPUs. Experiments~1 and~2 used CPUs only. %
 
The software environment consisted of Python 3.10, PyTorch 2.5.1 (CUDA 12.4), Hugging Face \texttt{transformers} 4.57.1, \texttt{datasets} 4.4.1 and \texttt{accelerate} 1.12.0; attention used the PyTorch SDPA kernel for GPT-NeoX and the eager implementation for GPT-Neo. %
 During evaluation, the PyTorch cuDNN backend was set to deterministic mode (\texttt{torch.backends.cudnn.deterministic = True}).

\paragraph{Retraining check.}
Retraining the 18 Small models of Section~\ref{sec:experiments} from scratch with the released training code and configuration, on training data regenerated by the released generator, reproduces their accuracies on the three evaluation sets within $0.02$ (mean absolute difference $0.006$ over the $54$ model--set pairs); the retraining ran on NVIDIA L40S and A100 GPUs and took about 57 GPU-hours.

\subsection{Architectural Implementations}
\label{app:implementations}

To isolate the positional encoding variable ($\mathrm{PE} \in \{\text{APE}, \text{RoPE}\}$), we implemented each encoding with a Hugging Face architecture that provides it natively:
\begin{itemize}
    \item \textbf{APE Configurations:} Evaluated using the \texttt{GPTNeoForCausalLM} architecture. The absolute positional embeddings were configured up to the maximum context length ($1024$). Global attention was enforced in all layers (i.e., \texttt{attention\_types=[[["global"], n\_layer]]}) to prevent confounding effects from sliding-window or local attention approximations.
    \item \textbf{RoPE Configurations:} Evaluated using the \texttt{GPTNeoXForCausalLM} architecture, which natively supports rotary mappings. We set the rotary percentage to $1.0$ (applying RoPE to all attention head dimensions) and disabled parallel residual connections (\texttt{use\_parallel\_residual=False}) to keep the computational graph as close as possible to the APE baseline, so that translation equivariance is the principal independent variable (the two implementations retain minor differences, e.g., attention-logit scaling and input/output embedding tying). %

\end{itemize}
Both architectural families used GELU activations (the exact GELU in GPT-NeoX and its tanh approximation, \texttt{gelu\_new}, in GPT-Neo) and vocabulary dimension $|\mathcal{V}| = 50{,}257$. During training of the Medium and Large models, memory use was reduced by enabling gradient checkpointing (\texttt{gradient\_checkpointing=True}), %
 which implicitly disables key-value caching during the forward pass. Training used \texttt{bfloat16} mixed precision, and evaluation loaded the model weights in \texttt{bfloat16}.

\subsection{Training Hyperparameters}

All 54 configurations were trained from scratch under a single shared hyperparameter configuration for controlled capacity comparisons:
\begin{itemize}
    \item \textbf{Data \& Context:} The models were trained on 500,000 trajectories per training measure. Sequences were truncated or padded to $T_{\max} = 1024$.
    \item \textbf{Loss Function:} Standard causal language modeling (next-token prediction) using cross-entropy loss. Padding tokens were ignored during loss computation by overriding their labels to $-100$; because the padding token is the end-of-text token that also opens and closes each trajectory, these delimiters were excluded from the loss as well.
    \item \textbf{Batch Size:} The effective batch size was $32$ trajectories per optimization step across all runs: a per-device batch of $32$ for the Small models, and a per-device batch of $4$ with $8$ gradient accumulation steps for the Medium and Large models.

    \item \textbf{Learning Rate \& Optimizer:} We used the AdamW optimizer with a single peak learning rate of $\eta = 8 \times 10^{-5}$ across all scales (Small, Medium, Large). The learning rate followed a linear warmup schedule for the first $1{,}000$ steps and then decayed linearly to zero. %
 Weight decay was uniformly set to $0.01$.
    \item \textbf{Epochs \& Seeds:} All models were trained for exactly 3 full epochs. Results for each configuration come from a single training run, with the global random seed for data shuffling and weight initialization fixed to $42$ (\texttt{random.seed(42)}, \texttt{torch.manual\_seed(42)}).
\end{itemize}

\subsection{The Countdown Task and the Stream of Search Format} %
\label{app:countdown_example}
In Countdown, a set of input numbers must be combined with the operations $\{+,-,\times,\div\}$, each number used once, to reach a target. Consider the target $10$ with numbers $[2,3,5]$. The depth-first search below first expands $2\times5=10$ (remaining $[3,10]$), whose children $3+10=13$ and $10-3=7$ miss the target; it backtracks to the root, expands $2\times3=6$ (remaining $[5,6]$), again meets only dead ends ($11$ and $1$), backtracks once more, and finally expands $2+3=5$ (remaining $[5,5]$) and $5+5=10$, which reaches the target. Stream of Search \citep{gandhi2024streamsearchsoslearning} serializes such a tree, including every dead end and backtracking step, into a single token sequence. The trajectory below is the one produced by the depth-first generator used for $\mu_{DFS}$ on this instance (sum heuristic, pruning threshold equal to the target); the model is prompted with the first line and must generate the rest.
{\footnotesize
\begin{verbatim}
Current State: 10:[2, 3, 5], Operations: []
Exploring Operation: 2*5=10, Resulting Numbers: [3, 10]
Generated Node #0,0: 10:[3, 10] Operation: 2*5=10
Moving to Node #0,0
Current State: 10:[3, 10], Operations: ['2*5=10']
Exploring Operation: 3+10=13, Resulting Numbers: [13]
13,10 unequal: No Solution
Moving to Node #0,0
Current State: 10:[3, 10], Operations: ['2*5=10']
Exploring Operation: 10-3=7, Resulting Numbers: [7]
7,10 unequal: No Solution
Moving to Node #0
Current State: 10:[2, 3, 5], Operations: []
Exploring Operation: 2*3=6, Resulting Numbers: [5, 6]
Generated Node #0,1: 10:[5, 6] Operation: 2*3=6
...                          (two dead ends below node #0,1, then backtrack to #0)
Exploring Operation: 2+3=5, Resulting Numbers: [5, 5]
Generated Node #0,2: 10:[5, 5] Operation: 2+3=5
Moving to Node #0,2
Current State: 10:[5, 5], Operations: ['2+3=5']
Exploring Operation: 5+5=10, Resulting Numbers: [10]
10,10 equal: Goal Reached
\end{verbatim}
}

\subsection{Evaluation Procedure}

Each model was evaluated on a set of 2,000 trajectories per structural measure ($\nu_{BFS}$, $\nu_{DFS}$, $\nu_{MIXED}$). For $\nu_{BFS}$ and $\nu_{DFS}$ these are the first 2,000 trajectories generated for the corresponding training measure (the generator is seeded identically), so in-distribution scores measure how well a model fits its training measure; $\nu_{MIXED}$ combines 1,000 trajectories from each. %
 The autoregressive generation was prompted identically using the initial problem state: \texttt{<|endoftext|>Current State: [Target]:[Nums], Operations: []} followed by a newline. %

To remove sampling noise, we used greedy decoding (\texttt{do\_sample=False}, \texttt{num\_beams=1}) with at most $512$ new tokens. Evaluation used a batch size of $32$ with left padding for RoPE models and a batch size of $1$ (no padding) for APE models. %
 The generated sequences were then parsed by the deterministic scoring function of the original Stream of Search implementation \citep{gandhi2024streamsearchsoslearning}, which checks the arithmetic of the reported operations and whether the target state is reached; %
 these scores define the zero-one loss and hence the expected zero-one risk $\mathcal{R}_{\mu_T}(f)$ reported in Section~\ref{sec:experiments}.

\paragraph{Comparisons across measures.}\label{app:eval_scope}
The three measures differ only in the search procedure that serializes the reference trajectory: the generator draws the number of operands, the target and the operands in the same way for all three, so they share one problem distribution. Because the model is prompted with the problem alone and the scorer checks only the model's own output, a fixed model has the same risk under $\mu_{BFS}$, $\mu_{DFS}$ and $\mu_{MIXED}$. Their structural measures, however, are far apart (Appendix~\ref{app:w1_estimation}): the DFS reference trajectories backtrack nearly four times as often as the BFS ones (mean $\chi_6=0.18$ against $0.05$ on the evaluation sets), with tree depth $\chi_1\le 3$ against $\chi_1\le 4$. On the freshly generated problems of Appendix~\ref{app:gbc_check}, Large Deep BFS APE has risk $0.511$, $0.511$ and $0.490$ under the three measures, whereas its accuracy on the evaluation sets falls from $0.4815$ (BFS) to $0.4555$ (DFS). Differences between the evaluation sets therefore reflect their particular $2{,}000$-problem samples. Averaged over the 18 models of each training measure, accuracy on the BFS, MIXED and DFS sets is $0.4740$, $0.4699$ and $0.4513$ (BFS-trained), $0.5132$, $0.5214$ and $0.5105$ (DFS-trained), and $0.5175$, $0.5207$ and $0.5110$ (MIXED-trained): the DFS set is the hardest for every training measure, and the drop from the BFS to the DFS set is larger for BFS-trained models ($2.27$ points, against $0.27$ and $0.65$ for DFS- and MIXED-trained models). In comparisons of architectures and training measures on a common evaluation set (Section~\ref{sec:experiments}), every model is scored on the same problems.

\section{Code Availability}
\label{app:code}
Code for data generation, training, evaluation and analysis is available at \url{https://github.com/Yyuzrah/AMTA4R}. %

\section{Empirical Checks of the Terms of the Main Bound}
\label{app:beyond}
\providecommand{\ontoR}{\raggedright\hspace{0pt}}
\renewcommand{\topfraction}{0.9}\renewcommand{\textfraction}{0.08}\renewcommand{\floatpagefraction}{0.85}

On Countdown, the projection $\Phi$ of Theorem~\ref{thm:unified_bound} maps a trajectory to the twelve coordinates $\boldsymbol\chi$ of Appendix~\ref{app:trajectory-space}. Appendix~\ref{app:ontology} derives the requirements that the terms of the bound place on the coordinates of a projection. We then report eight experiments that we designed, each of which concerns one part of the bound. Experiment~1 examines whether more coordinates improve the projection on Countdown. Experiments~2 and~3 ask whether $\widehat W_1$ separates groups of agent, embodied and robot trajectories whose tasks differ in structure, and whether a change of agent program alone shifts the trajectories and the success of one model (Appendix~\ref{app:c2}). Experiments~4--6 test Lemma~\ref{lemma:ape_shattering}, the basis of the factor $K_{f,\mathrm{PE}}$, on attention scores, on answers under shifted position indices and on two Countdown models (Appendix~\ref{app:c1}). Experiments~7 and~8 concern the depth term and whether splitting a deeply nested task into separate model calls restores the success that one call loses (Appendix~\ref{app:handoff}). The code repository (\url{https://github.com/Yyuzrah/AMTA4R}) holds the code, all per-pair, per-model and per-run results, and a catalog of candidate coordinates for Countdown and six other trajectory domains.

\subsection{How the Coordinates of \texorpdfstring{$\Phi$}{Phi} Are Chosen}
\label{app:ontology}

A projection $\Phi_D$ for a trajectory domain $D$ maps a trajectory to $d_\Phi$ coordinates. The upper bound of Theorem~\ref{thm:unified_bound} adds to the source risk the shift term $K_{f,\mathrm{PE}}W_1(\nu_S,\nu_T)$ and twice the residual $\epsilon$ of Assumption~\ref{assum:epsilon_relaxed}, where $\nu_S$ and $\nu_T$ are the distributions of $\Phi(z)$ on the source and target tasks and $K_{f,\mathrm{PE}}$ is the Lipschitz constant of the risk as a function of $\Phi(z)$. The lower bound is the largest of three terms: the source risk minus the shift term and $2\epsilon$, a width term $C_{\mathrm{width}}/\sqrt m-\epsilon$ for a model of width $m$, and a depth term $C_{\mathrm{depth}}\int P_{\mathrm{implicit}}\,d\nu_T-\epsilon$. Every term except the width term places a requirement on the coordinates, and so does the estimation of $W_1$ from samples. We state these requirements in terms of events, the steps of a trajectory, such as the actions and observations of an agent or the visited nodes of a Countdown search.

The residual $\epsilon$ accounts for what $\Phi(z)$ leaves out, such as the operand values of a Countdown problem (Appendix~\ref{app:empirical_sufficiency}). A coordinate therefore depends on the kinds of events and on their positions in the sequence of events, and leaves content such as operands or object names to $\epsilon$. Since $\nu_S$ and $\nu_T$ describe tasks, every coordinate is computed from the reference trajectory, the trajectory of a reference solver for the same problem, and may include the reference's own outcome, such as whether the reference search reached the target ($\chi_{10}$). The shift term needs $\Phi$ to be bounded, so each coordinate is finite and grows at most linearly with the largest trajectory length $T_{\max}$. Because $K_{f,\mathrm{PE}}$ is taken in the Euclidean distance, a unit step along a coordinate must mean the same everywhere. Categories therefore enter as counts, one per category, as the Countdown operators enter through $\chi_{11}$ (multiplications and divisions) and $\chi_{12}$ (additions and subtractions), and any weighting of the coordinates is the same for every compared measure.

The gap $K_{f,\mathrm{APE}}\gg K_{f,\mathrm{RoPE}}$ of Theorem~\ref{thm:lipschitz_gap} uses the operator $\mathcal T_k$ of Definition~\ref{def:translation_operator}, which inserts one event that changes no state, such as the identity node \texttt{x=x}, written with $k$ tokens. The proof bounds $K_{f,\mathrm{APE}}$ from below by the jump $\gamma$ in risk divided by the distance $\delta$ by which the insertion moves $\Phi(z)$, so $\delta$ must be positive and bounded in $k$ and $T_{\max}$. Every coordinate therefore changes under $\mathcal T_k$ by a bounded amount, and one coordinate counts the inserted event. The search volume $\chi_4$, the number of node references, grows by one independently of $k$, whereas the token count of the trajectory grows by $k$ and would drive $\gamma/\delta$ to zero for long insertions. The depth term needs one coordinate that counts the nesting that one forward pass must resolve (Assumption~\ref{ass:implicit-complexity}), on Countdown $\chi_1-1$, the nesting depth of the bracket word of Lemma~\ref{lemma:dyck_homomorphism}. The set of coordinates is small, because the number of samples needed to estimate $W_1$ grows quickly with the dimension \citep{weed2019sharp}.

Lemmas~\ref{lemma:error_decomposition} and~\ref{lemma:width_capacity} and Assumption~\ref{ass:width-capacity} are formulated on a general space $\mathcal X$, so Theorem~\ref{thm:unified_bound}, stated on $\mathbb R^{d_\Phi}$, keeps its form for any projection $\Phi_D$ with these properties, provided its assumptions hold in the domain. The twelve Countdown coordinates $\boldsymbol\chi$ have these properties: they are counts, rates and depths of events of the reference trajectory, the coordinates other than $\chi_4$ change under $\mathcal T_k$ by an amount of order $1/\chi_4$ or not at all (Appendix~\ref{app:attention-translation}), and Appendix~\ref{app:gbc_check} measures how well they predict success.

\subsubsection{Experiment 1: More Coordinates on Countdown}
\label{app:refine}

Adding coordinates can only raise $W_1$, since a distance over more coordinates is at least as long, and can only lower the smallest attainable $\epsilon$, since a risk function of the old coordinates remains available. The upper bound becomes tighter only if $2\epsilon$ falls by more than $K_{f,\mathrm{PE}}W_1$ grows, and $W_1$ is harder to estimate in more dimensions \citep{weed2019sharp}. Experiment~1 measures both effects on Countdown with four nested projections $\Phi_{12}=\boldsymbol\chi\subset\Phi_{16}\subset\Phi_{24}\subset\Phi_{34}$, where the subscript is the number of coordinates. The $22$ further coordinates are computed on the reference trajectory and enter class by class, in the order of the catalog in the code repository. $\Phi_{16}$ adds four coordinates of the shape of the search tree, such as the revisit ratio, $\Phi_{24}$ eight of how the search moves and what it must remember, such as the longest streak of failures, and $\Phi_{34}$ ten more of memory, nesting, actions and length, such as the number of lines. Four of the further coordinates, the token length among them, are counted in tokens; the two measurements concern $\epsilon$ and $W_1$, and the bounded change under $\mathcal T_k$ enters only the gap of Theorem~\ref{thm:lipschitz_gap}.

The first measurement repeats for each projection the prediction of success of Appendix~\ref{app:gbc_check}, in which a gradient-boosting classifier predicts from the coordinates of $z$ whether a trained model fails on the problem of $z$. It uses the two models of that appendix, $f_1$ (Small Deep MIXED RoPE, $12$ layers) with outcomes on $50{,}000$ problems and $f_2$ (Large Deep BFS APE, $48$ layers) with outcomes on $3{,}000$ problems. We report the out-of-fold ROC-AUC and McFadden pseudo-$R^2$ over ten repetitions of stratified five-fold cross-validation, with the same folds for all projections. The second measurement estimates $\widehat W_1$ between the training trajectories of $\mu_{BFS}$, $\mu_{DFS}$ and $\mu_{MIXED}$, together with the within-measure baselines, as in Appendix~\ref{app:w1_estimation}.

From $\Phi_{12}$ to $\Phi_{34}$ the ROC-AUC changes by at most $0.003$, staying between $0.927$ and $0.929$ for both models, and the McFadden pseudo-$R^2$ by at most $0.016$. The within-measure baselines grow by a factor of $1.6$--$2.3$ and the between-measure distances by $1.8$--$2.0$, for example from $0.19$ to $0.44$ for the baseline of $\mu_{BFS}$ and from $2.97$ to $5.65$ between $\mu_{BFS}$ and $\mu_{DFS}$. The ratio of the mean baseline to the mean between-measure distance stays at $0.13$. On Countdown the further coordinates thus roughly double $\widehat W_1$ and its noise floor without improving the prediction of success, and the twelve coordinates are the better choice. The same two measurements can be run in any domain to decide whether a coordinate is worth adding.

\subsection{Agent, Embodied and Robot Trajectories}
\label{app:c2}

Experiments~2 and~3 apply projections $\Phi_D$ and the estimate $\widehat W_1$ to recorded logs of language agents, code agents and robots. Experiment~2 compares groups of reference trajectories, and Experiment~3 runs one language model in two agent harnesses, the programs that build the prompt, offer the model its tools and run its commands.

\subsubsection{Experiment 2: Structural Shift between Groups of Trajectories}
\label{app:c2-groups}

Experiment~2 tests whether $\widehat W_1$ reports a clear shift between groups of trajectories whose tasks differ in structure, and a small one between groups whose tasks share it. We use reference trajectories from four public sources. The first are the expert trajectories released with ETO \citep{song2024eto,song2024etodata}, a method for training language agents, in ALFWorld \citep{shridhar2021alfworld} (household tasks, $3{,}119$ trajectories), ScienceWorld \citep{wang2022scienceworld} (science experiments, $1{,}483$) and WebShop \citep{yao2022webshop} (shopping, $1{,}824$, of which $905$ by humans and $919$ by GPT-4). The second are two sets of code-agent runs on GitHub issues of SWE-rebench \citep{badertdinov2025swerebench}, released by Nebius: $12{,}467$ resolved runs on $835$ issues of a 70B Llama model (\texttt{swe-agent-llama-70b}) in the SWE-agent harness \citep{yang2024sweagent,nebius2024sweagenttrajectories}, and $32{,}161$ on $3{,}792$ issues of Qwen3-Coder-480B-A35B-Instruct \citep{qwen2025qwen3} in the OpenHands harness \citep{wang2025openhands,trofimova2025openhandstrajs}. The third are the $33{,}252$ teleoperated robot-arm demonstrations of BridgeData V2 \citep{walke2023bridgedata} with a non-empty instruction, with the per-step movement primitives released with ECoT \citep{zawalski2025ecot,embodiedcot2024bridgefeatures}. The fourth are $4{,}000$ levels per instruction kind of the grid world BabyAI \citep{chevalierboisvert2019babyai}, generated with BALROG \citep{paglieri2024balrog} and Minigrid \citep{chevalierboisvert2023minigrid} and solved by the BabyAI bot.

For each source, a parser turns each log into a sequence of events and marks which events open and close a level of nesting, such as taking and putting down an object in ALFWorld or the explicit subgoal stack of the BabyAI bot. Each projection has $12$ to $16$ coordinates, among them the maximum and mean depth, counts of steps by type and, depending on the domain, revisits and failures. All projections except the code-agent one include a few surface coordinates (S), totals such as the number of actions.

We form $52$ pairs of groups in six families: the six ALFWorld task types ($15$ pairs), the three ETO environments ($3$), human against GPT-4 demonstrations in WebShop ($1$), SWE-agent against OpenHands on the $554$ issues resolved by both and on all resolved issues ($2$), seven BridgeData skills ($21$) and five BabyAI instruction kinds ($10$). $\widehat W_1$ is computed as on Countdown (Appendix~\ref{app:w1_estimation}), with the coordinates z-scored on the pooled family and samples of $N=2{,}000$ trajectories, or half the smaller group if that is less ($139$--$2{,}000$). The baseline of a group is $\widehat W_1$ between two disjoint samples of that group, and the ratio $R$ divides the mean between-group distance over ten seeds by the larger mean baseline. A pair is separated when $R\ge3$ and the between-group $\widehat W_1$ has a coefficient of variation over seeds of at most $0.10$.

\begin{figure}[!htbp]
\centering
\includegraphics[width=\textwidth]{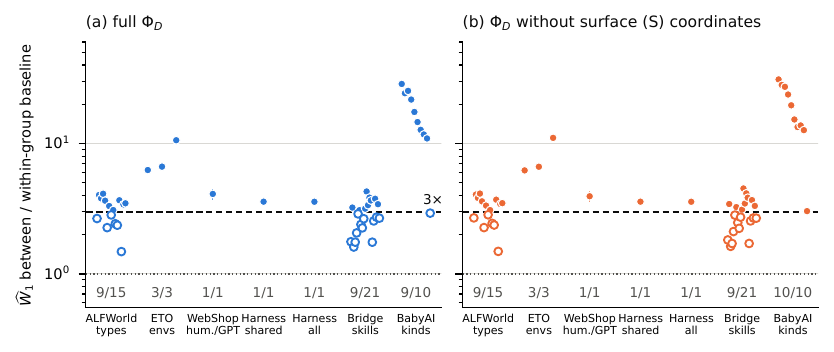}
\caption{Structural shift between groups of reference trajectories (Experiment~2). Vertical axis: ratio $R$ of the between-group $\widehat W_1$ to the larger of the two within-group baselines, on a logarithmic scale. One marker per pair of groups, spread sideways within a column. (a) All coordinates of $\Phi_D$. (b) Without the surface coordinates (S). Columns from left to right: the six ALFWorld task types, the three ETO environments, human against GPT-4 demonstrations in WebShop (hum./GPT), SWE-agent against OpenHands, in which harness and model both differ, on the issues resolved by both (Harness shared) and on all resolved issues (Harness all), the seven Bridge skills and the five BabyAI instruction kinds. Filled marker: separated pair. Hollow marker: pair not separated. Dashed line, labeled $3\times$: $R=3$. Dotted line: $R=1$. Bars: standard deviation over seeds of the between-group $\widehat W_1$, divided by the larger baseline; most are hidden by the markers. Numbers at the bottom: separated pairs out of all pairs of the column.}
\label{fig:c2}
\end{figure}

Figure~\ref{fig:c2} shows $R$ for all pairs. Of the $52$ pairs, $33$ are separated, and all coefficients of variation are at most $0.095$. The separated pairs include all three pairs of ETO environments ($R=6.26$--$10.60$), nine of the ten BabyAI pairs ($10.94$--$28.65$), the pair of human and GPT-4 demonstrations ($4.11$), both code-agent pairs ($3.58$) and $9$ of the $15$ ALFWorld pairs. Pairs whose tasks share a sequence of subgoals stay below $R=3$, such as heating against cooling an object in ALFWorld ($1.48$, the smallest ratio of all pairs) and put next to against a pick-up followed by a go-to in BabyAI ($2.92$). In BridgeData $9$ of the $21$ pairs are separated: the demonstrations of one skill vary widely, and the large baselines ($1.23$--$1.99$, against at most $1.39$ for ALFWorld and $0.61$ for ETO) lower $R$. Without the surface coordinates (panel b), $34$ pairs are separated. The estimate $\widehat W_1$ on $\Phi_D$ thus separates $33$ of the $52$ pairs of groups of reference trajectories, among them most of those that differ in environment, instruction kind or demonstrator, and it estimates the distance $W_1(\nu_S,\nu_T)$ of the shift term.

\subsubsection{Experiment 3: One Model in Two Agent Harnesses}
\label{app:c2-harness}

The code-agent pairs of Experiment~2 change the harness and the model together. Experiment~3 holds the model fixed, changes only the harness, and compares the model's own episodes, which are the trajectories that run-time monitors, verifiers and trajectory filters read. We run Qwen3-8B \citep{qwen2025qwen3}, an open-weight model with eight billion parameters, in bf16 with its thinking mode disabled, served by vLLM \citep{kwon2023vllm} in two harnesses with the same sampling settings (temperature $0.7$, top-$p$ $0.8$, top-$k$ $20$, at most $4{,}096$ tokens per response) and at most $100$ model calls per episode: mini-SWE-agent 2.4.6 \citep{yang2024sweagent}, which gives the model a single shell tool, and OpenHands, which has separate tools for shell commands and for viewing and editing files. The tasks are $300$ of the $500$ instances of SWE-bench Verified \citep{jimenez2024swebench,openai2024swebenchverified}, human-checked GitHub issues of Python repositories, drawn in proportion to the number of instances per repository. Each instance is run once in each harness, in a fresh container built from the official image of the instance. We project the episodes onto the $16$ coordinates of the code-agent projection of Experiment~2 and compute $R$ with ten seeds and $N=150$, half the episodes of each harness.

Between the harnesses $\widehat W_1=3.81$, with a coefficient of variation of $0.024$, against baselines of $1.48$ for mini-SWE-agent and $1.74$ for OpenHands, so $R=2.19$. Computed on each coordinate alone, $\widehat W_1$ between the harnesses exceeds that coordinate's own baseline by a factor of $1.95$ to $11.5$. The largest differences concern reading code, where the harnesses offer different tools. Under mini-SWE-agent the model makes $50.4$ read calls per episode against $17.3$ under OpenHands and $73.8$ tool calls in all against $43.5$. The success changes as well. On the $299$ instances whose reference fix passes the tests, the resolve rate rises from $1.7\%$ under mini-SWE-agent to $6.4\%$ under OpenHands, by $4.7$ points (95\% confidence interval $[2.0,7.7]$, McNemar test $p=0.001$).

Changing the harness alone thus moves the trajectories of Qwen3-8B along every coordinate and changes how often the agent succeeds, so a monitor, verifier or trajectory filter calibrated on one harness receives shifted inputs under the other.

\subsection{Lemma~\ref{lemma:ape_shattering} under Shifted Position Indices}
\label{app:c1}

Lemma~\ref{lemma:ape_shattering} states that shifting the position indices of a sub-trajectory changes the attention scores of an APE model almost surely, while under RoPE they stay the same (Appendix~\ref{app:proof-rope-equivariance}). APE adds a learned vector for each position to the embedding of the token at that position (Appendix~\ref{app:attention-translation}), RoPE \citep{su2023roformerenhancedtransformerrotary} rotates queries and keys so that a score depends on the positions only through their distance, and ALiBi \citep{press2022trainshorttestlong} subtracts from each score a slope per head times the distance. Experiments~4--6 shift all position indices without inserting a token, so the tokens and their distances stay the same, and in exact arithmetic every attention score under RoPE and every attention weight under ALiBi is unchanged.

Experiments~4 and~5 use $33$ public base models, trained only to predict the next token. Twelve use APE: GPT-2 \citep{radford2019language} with 124M to 1.5B parameters, GPT-Neo \citep{black2021gptneo} with 125M to 2.7B and OPT \citep{zhang2022opt} with 125M to 6.7B. Five use ALiBi: BLOOM \citep{bigscience2022bloom} with 560M to 7.1B. Sixteen use RoPE: six Pythia models \citep{biderman2023pythia} with 160M to 6.9B, GPT-J 6B \citep{wang2021gptj}, OLMo-2 \citep{olmo2024olmo2} with 1B and 7B, SmolLM2 \citep{allal2025smollm2} with 135M to 1.7B, and Qwen2.5 \citep{qwen2024qwen25} with 0.5B to 7B.

\subsubsection{Experiment 4: Attention Scores of Pretrained Models}
\label{app:c1-kernel}

Experiment~4 compares the attention scores of full pretrained architectures with and without a shift of all position indices, to test whether those of APE models change substantially and those of RoPE models stay unchanged up to rounding. We use eight of the $33$ models: GPT-2 124M and 355M and GPT-Neo 125M and 1.3B with APE, and Pythia 160M, 410M and 1.4B and SmolLM2 360M with RoPE. Each model reads $512$ sequences of $512$ tokens, $128$ from each of four families: Countdown reference trajectories of $\mu_{DFS}$ and of $\mu_{BFS}$, Dyck-$4$ strings of four kinds of brackets nested at most $12$ deep, and source files of the Python 3.10 standard library \citep{cpython31019}. Each sequence runs once with position indices $0,\dots,511$ and once with $k,\dots,511+k$ for $k\in\{16,64,256\}$, within the context of every model ($96$ combinations of model, shift and family). All runs are in fp32 with the TF32 mode of the GPU turned off, and each is repeated in float64 to separate real changes from rounding. We record the largest change of the attention scores before the softmax, divided by the largest absolute score, in the first layer and over all layers, the share of positions whose top-1 next-token prediction changes, and the change of the log-likelihood of the text in nats per token.

In all $48$ combinations with an APE model, the largest first-layer change is $0.42$--$1.19$ of the largest absolute score, the same order as the scores themselves. The shift changes $14.0$--$87.4\%$ of the top-1 predictions and lowers the log-likelihood by $0.20$--$6.77$ nats per token. Under RoPE the change in fp32 is at most $6.2\times10^{-6}$ in the first layer and $3.2\times10^{-3}$ over all layers, of the same order as the difference between the fp32 and float64 computations without any shift, and in float64 it is at most $4.4\times10^{-12}$, so the fp32 change is rounding. RoPE changes at most $0.038\%$ of the top-1 predictions, all at near-ties between the two most likely tokens. The contrast of Lemma~\ref{lemma:ape_shattering} thus holds in the attention scores and predictions of these eight pretrained models.

\subsubsection{Experiment 5: Answers under Shifted Position Indices}
\label{app:serving}

Experiment~5 applies the shifts of position indices that occur when a model is served (through padding, cached prefixes and sliding windows) and checks whether they change the answers of APE models and leave those of RoPE and ALiBi models unchanged. Each of the $33$ models answers $3{,}000$ items, scored by exact match: $500$ for each of five few-shot tasks (copying $6$--$24$ words, reversing $2$--$4$ words, closing a Dyck-$4$ prefix, adding two numbers of $1$--$2$ digits, and completing a line of a Countdown search trace) and $500$ key--value retrieval items whose prompts grow with the context of the model, up to $3{,}000$ tokens. Decoding is greedy with cached keys and values, in fp32 except in condition (vi) below, and each item is decoded in the same batch as its unshifted version.

We shift the position indices by $k\in\{16,64,256,1024\}$, within the trained context, in four ways that occur in serving: (i) all indices; (ii) left padding with $k$ tokens and position indices that count the padding; (iii) a cached prefix, the prompt up to its final question, computed $k$ positions later; and (iv) a sliding window with attention sinks \citep{xiao2024streamingllm}, which keeps the $128$ most recent cache slots and the first four, shifted by $k$ and compared with the same window without the shift (few-shot tasks only). In the control (v), a cached prefix of $P$ tokens keeps its indices and the rest of the prompt starts at $P+k$, so every distance between the prefix and a later token grows by $k$. Condition (vi) repeats (i) in bf16. A change in accuracy is significant when its paired bootstrap 95\% confidence interval over items excludes $0$.

\begin{figure}[t]
\centering
\includegraphics[width=\textwidth]{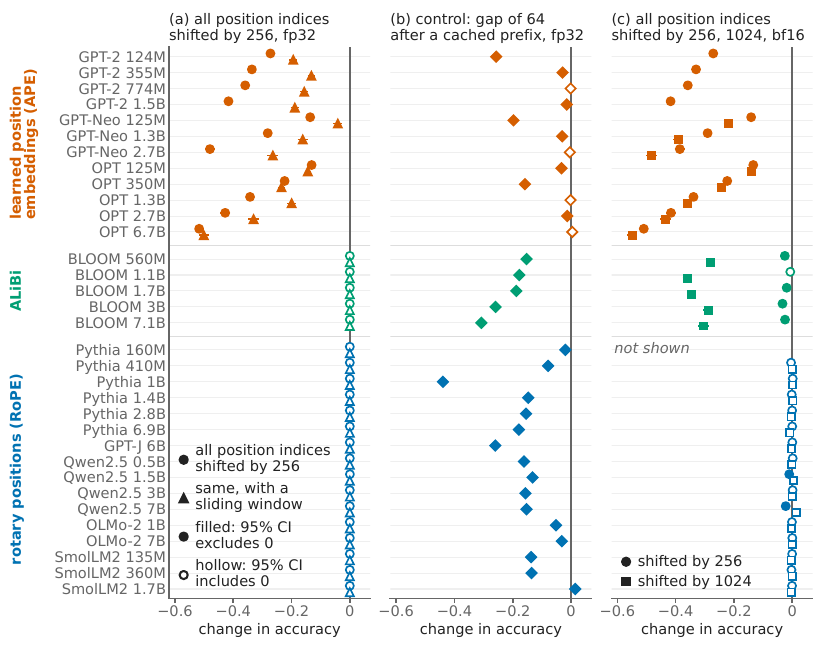}
\caption{Change in accuracy caused by shifting the position indices, for the $33$ models of Experiment~5, one row per model, grouped and colored by positional encoding (APE, ALiBi, RoPE). Horizontal axis: accuracy with the shift minus accuracy without it. Bars: paired bootstrap 95\% confidence intervals over items, mostly hidden by the markers. Filled marker: the interval excludes $0$. Hollow marker: it includes $0$. (a) fp32. Circles: all position indices shifted by $256$, condition (i). Left padding (ii) and a cached prefix placed $256$ positions later (iii) give the same changes for the APE models and changes of at most $0.0003$ for the others, and are not drawn. Triangles: the same shift inside a sliding window (iv), compared with the sliding window without the shift (few-shot tasks only). (b) Control (v), fp32, diamonds: the cached prefix keeps its position indices and the rest of the prompt starts $64$ positions later. (c) bf16: all position indices shifted by $256$ (circles) and by $1024$ (squares, where the shifted prompt fits the trained context). Pythia 160M is not shown in (c), since its bf16 accuracy without the shift is $0.012$.}
\label{fig:shift}
\end{figure}

Figure~\ref{fig:shift} shows the results. In fp32 under (i)--(iv), the $21$ ALiBi and RoPE models change $40$ of their $948{,}000$ answers, and every accuracy changes by at most $0.0003$. The twelve APE models lose $0.13$--$0.52$ in accuracy at $k=256$ under (i)--(iii) and $0.04$--$0.50$ under (iv), and under (i) at $k=256$ they change $70$--$100\%$ of their answers. In the control (v), the answers change under all three encodings, by up to $0.44$ in accuracy at $k=64$, and the change is significant for $63$ of the $64$ combinations of RoPE model and shift and for $19$ of the $20$ of ALiBi model and shift. In bf16 the APE losses remain. The RoPE models other than Pythia 160M (not shown in Figure~\ref{fig:shift}c) stay within $0.022$ of their unshifted bf16 accuracy, with significant changes only for Qwen2.5 1.5B and 7B at $k\ge64$. The ALiBi models lose $0.007$--$0.033$ at $k=256$ and $0.28$--$0.36$ at $k=1024$, consistent with BLOOM adding the slope times the absolute position of the key, which equals the slope times the distance up to one constant per query but has a bf16 rounding error that grows with the position index. The shifts of position indices that occur in serving thus change the answers of the APE models, while the ALiBi and RoPE models in fp32 keep theirs as long as the distances between tokens are kept.

\subsubsection{Experiment 6: Two Countdown Models of Section~\ref{sec:experiments}}
\label{app:g1}

Experiment~6 turns to the two models of Appendix~\ref{app:gbc_check}, $f_1$ (RoPE, $L=12$) and $f_2$ (APE, $L=48$): we shift all their position indices in fp32 and record the effect on their routing decisions and their success on Countdown. The routing decisions are the node identifier after \texttt{Moving to Node} and the operation after \texttt{Exploring Operation} (Appendix~\ref{app:countdown_example}). We use those after the first leaf-level operation line of each trajectory, a line that explores an operation on a node with two numbers, such as \texttt{Exploring Operation: 3+10=13, Resulting Numbers: [13]}. After this line the model must decide where the search goes next. The line comes before the end of the search and directly after a dead end at the same node or a return to that node, and the node still has an unexplored operation whose result misses the target. In the teacher-forced test, the model reads a reference trajectory, and we compare its routing predictions after this line with and without a shift of all position indices by $16$ or $64$. In the free-running test, the model generates its own trajectory, generation resumes at this line with all indices shifted by $16$ or $64$, and the outcome is scored as in Section~\ref{sec:experiments}.

On the $2{,}863$ reference trajectories of the three evaluation sets in which such a line, followed by at least one routing decision, lies within the first $896$ tokens, the shift changes none of the $94{,}278$ routing predictions of the RoPE model, and $25.7\%$ (shift $16$) and $28.3\%$ (shift $64$) of those of the APE model. The free-running test uses the first $1{,}000$ problems of each evaluation set and applies to the $2{,}128$ problems for the RoPE model and the $1{,}143$ for the APE model whose output contains such a line with at least $64$ of the $512$ tokens of the generation budget left. For the RoPE model the outcome of none of these problems changes (success $0.594$). For the APE model, success falls from $0.527$ by $0.40$ (shift $16$) and by $0.39$ (shift $64$). The contrast of the lemma thus carries over to the task success of two models of the paper, consistent with the mechanism of Assumption~\ref{assum:ar_sensitivity}, in which changed routing decisions turn into failed searches.

\subsection{Splitting Deeply Nested Tasks into Separate Model Calls}
\label{app:handoff}

Lemma~\ref{lemma:dyck_homomorphism} maps a Countdown search to a word of brackets, and any task with nested subgoals can be written in the same way. Lemma~\ref{lemma:depth_bottleneck} bounds the risk from below when the nesting that one forward pass must resolve exceeds $\alpha_{\mathrm{circ}}L$, which gives the depth term. Agent systems often spread deep nesting over several calls, for example by passing a subtask to a sub-agent. Experiments~7 and~8 measure whether such splitting restores success and how deep each call should be.

In such a word, \texttt{( a} opens a subgoal of type \texttt{a} and \texttt{) a} closes it, and at every closing bracket the model predicts the type that is closed. With calls of depth at most $2$, the plan \texttt{( a ( b ( c ( d ) d ) c ) b ) a} becomes the call \texttt{( a ( b ) b ) a}, in which \texttt{b} is an empty subgoal, followed by \texttt{( b ( c ) c ) b} and \texttt{( c ( d ) d ) c}, each in a fresh context. A task is solved only if every prediction of every call is correct, and the program that splits the task uses only its nesting structure. On a calibration set, separate from the test set, we measure how often one call solves a task of each depth, fit to these rates a curve that does not increase with depth, and read from it $d_{\mathrm{ext}}$, the largest depth at which one call succeeds at least half the time, and $\hat d$, the largest depth at which it succeeds at least $80\%$ of the time. Figure~\ref{fig:split} shows the results.

\begin{figure}[!htbp]
\centering
\includegraphics[width=\textwidth]{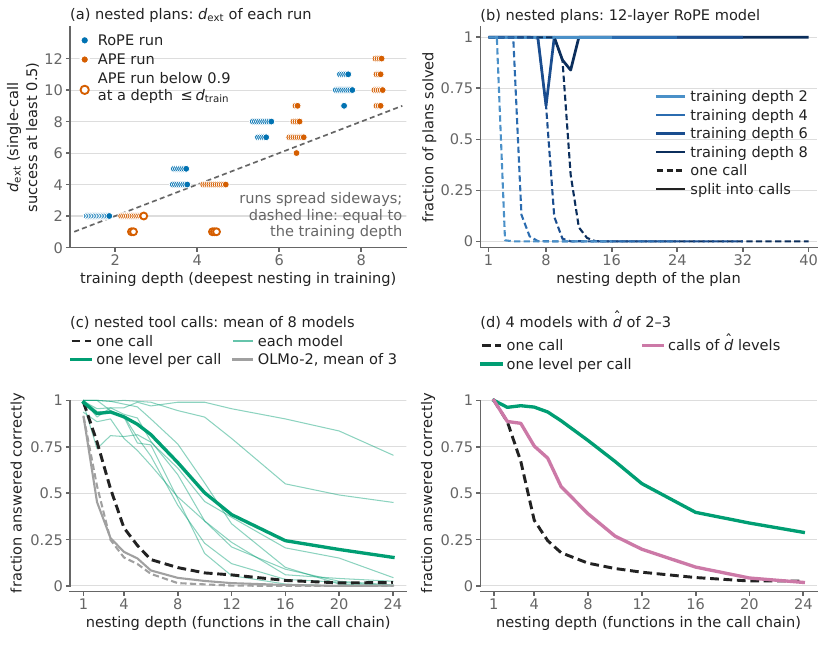}
\caption{Splitting deeply nested tasks into separate model calls (Experiments~7 and~8). (a) Experiment~7: $d_{\mathrm{ext}}$, the largest depth that one call solves at least half the time, for each of the $120$ runs (vertical axis), against the training depth $d_{\mathrm{train}}$ (horizontal axis). RoPE runs are drawn left and APE runs right of each training depth, and runs with equal values are spread sideways. Dashed line: $d_{\mathrm{ext}}=d_{\mathrm{train}}$. Hollow markers: the six runs of the two-layer width-$624$ APE model at $d_{\mathrm{train}}=2$ and $4$, whose single-call success stays below $0.9$ at some depth up to $d_{\mathrm{train}}$. (b) Experiment~7, 12-layer RoPE model: fraction of plans solved (mean of $3$ seeds) against the nesting depth of the plan, one shade per training depth. Dashed: one call. Solid: plans deeper than $d_{\mathrm{ext}}$ split into calls of depth $\min(d_{\mathrm{ext}},d_{\mathrm{train}})$, and other plans solved in one call. The dips between $d_{\mathrm{train}}$ and $d_{\mathrm{ext}}$ come from plans that are still solved in one call. (c, d) Experiment~8: fraction of the $200$ test chains answered correctly against the depth $d$, the number of functions in the chain. (c) Mean of the eight Qwen, Phi and Mistral models with one call (dashed) and with one level per call (thick solid green). Thin lines: each of these models with one level per call. Gray, with the same line styles: mean of the three OLMo-2 models. (d) Mean of the four models whose $\hat d$ is $2$ or $3$: one call (dashed), one level per call (green) and calls of $\hat d$ levels (pink).}
\label{fig:split}
\end{figure}

\subsubsection{Experiment 7: Nested Plans}
\label{app:nestplan}

Experiment~7 trains small models from scratch on nested plans to test whether splitting restores success beyond $d_{\mathrm{ext}}$, and whether $d_{\mathrm{ext}}$ follows the number of layers, as the threshold $\alpha_{\mathrm{circ}}L$ suggests, or the depth of the training plans. A plan has $48$ subgoals of $8$ types and maximum nesting depth exactly $d$: a spine of $d$ nested subgoals and $48-d$ further subgoals, each attached at random to a subgoal of depth below $d$ or to the top level. Every plan has $194$ tokens regardless of its depth, and it is solved when the type predicted at each of its $48$ closing brackets is correct, all other tokens being given.

We train the GPT-NeoX (RoPE) and GPT-Neo (APE) architectures of Appendix~\ref{app:reproducibility} in five shapes, width $256$ with $L\in\{2,4,8,12\}$ layers and a two-layer model of width $624$ whose number of parameters matches that of the 12-layer model within $1.3\%$, on plans whose depth is drawn uniformly from $1$ to $d_{\mathrm{train}}\in\{2,4,6,8\}$. Training takes $30{,}000$ steps of $256$ plans (AdamW, peak learning rate $10^{-3}$ with cosine decay, loss on the predicted types), and the $40$ configurations (two encodings, five shapes, four training depths) with three seeds each give $120$ runs. Five configurations in which some run stays below a single-call success of $0.9$ at a depth up to $d_{\mathrm{train}}$ are trained again for $60{,}000$ steps, after which three of them reach $0.9$ at every such depth in every run. For each run, $d_{\mathrm{ext}}$ comes from $200$ calibration plans per depth. A plan deeper than $d_{\mathrm{ext}}$ is split into calls of depth $D=\min(d_{\mathrm{ext}},d_{\mathrm{train}})$, which keeps every call within the training depths and needs $D\ge2$. The test sets have $500$ plans per depth, up to depth $4d_{\mathrm{train}}+8$.

In $38$ of the $40$ configurations, every run reaches a single-call success of at least $0.9$ at every depth up to $d_{\mathrm{train}}$. In all $114$ runs of these configurations, splitting gives a mean success of at least $0.987$ over the depths beyond $d_{\mathrm{ext}}$, and at least $0.91$ at each of them, against at most $0.04$ for one call, as Figure~\ref{fig:split}b shows for the 12-layer RoPE model. $d_{\mathrm{ext}}$ follows $d_{\mathrm{train}}$: it is $2$, $4$--$5$, $6$--$9$ and $9$--$12$ for $d_{\mathrm{train}}=2$, $4$, $6$ and $8$, with no systematic dependence on the number of layers, the width or the positional encoding, and exceeds $d_{\mathrm{train}}$ by at most four levels (Figure~\ref{fig:split}a). Training on deeper plans reduces the number of calls: beyond $d_{\mathrm{ext}}$, a plan needs $20.8$ calls on average at $d_{\mathrm{train}}=2$ and $4.5$--$4.6$ at $d_{\mathrm{train}}=8$. The two other configurations are the two-layer width-$624$ APE model at $d_{\mathrm{train}}=2$ and $4$: five of their six runs have $d_{\mathrm{ext}}=1$, too shallow for a split with $D\ge2$, and in the sixth splitting raises the success beyond $d_{\mathrm{ext}}=2$ from $0.00$ to $0.998$. In models of $2$ to $12$ layers, the depth that one call resolves thus follows the training depth independently of $L$, and splitting into calls within the training depths restores success on deeper plans.

\subsubsection{Experiment 8: Nested Tool Calls with Open-Weight Language Models}
\label{app:llm-handoff}

Experiment~8 measures $\hat d$ for pretrained language models, whose training depth is unknown, on a task shaped like the tool calls of an agent: a registry lists one-line Python functions, and a question asks for the value of a chain of nested calls, such as \texttt{bell(4)} for the registry
{\footnotesize
\begin{verbatim}
def lamp(x): return x + 2
def door(x): return lamp(x) + 3
def bell(x): return door(x) + 1
\end{verbatim}
}
\noindent a chain of depth $d=3$ whose value is $10$. With one level per call, the chain is split from the innermost function outward: the first call returns $\texttt{lamp(4)}=6$, and the next call sees the line of \texttt{lamp} replaced by this value and asks for \texttt{door(4)}. Only the answer of the last call is scored.

A registry has $26$ lines: the $d$ functions of the chain, named by common English nouns and each adding a constant $c\in\{1,2,3\}$, two distractor functions, and $24-d$ comment lines that keep the length of the prompt the same for every $d$. Four worked examples precede every prompt. The model may generate at most eight tokens, and the first integer among them is scored by exact match, so the model must follow the chain within its forward passes. The calibration set has $100$ and the test set $200$ items for each $d\in\{1,2,3,4,5,6,8,10,12,16,20,24\}$. We evaluate eleven instruction-tuned models in bf16 with greedy decoding: Qwen2.5 \citep{qwen2024qwen25} (7B, 14B and 32B parameters), OLMo-2 \citep{olmo2024olmo2} (7B, 13B, 32B), Qwen3 \citep{qwen2025qwen3} (8B, 14B, 32B, thinking mode disabled), Phi-4 14B \citep{abdin2024phi4} and Mistral-Small-24B-Instruct-2501 \citep{mistral2025small3}. We compare four ways of answering: one call; one call whose prompt also lists the open call chain, such as \texttt{bell(4) -> door(4) -> lamp(4)}; one level per call, with $d$ calls; and calls of $\hat d$ levels, with $\lceil d/\hat d\rceil$ calls and $\hat d$ measured for each model on the calibration set.

The accuracy of one call falls quickly as the chain grows (Figure~\ref{fig:split}c), from $0.88$--$1.00$ at $d=1$ to at most $0.52$ at $d=4$ and at most $0.15$ at $d=8$. $\hat d$ is $3$ for Qwen2.5 32B, $2$ for Qwen3 14B, Phi-4 and Mistral-Small-24B, and $1$ for the other seven models. Listing the chain changes the single-call accuracy by at most $0.09$ on average over depths. One level per call raises the accuracy of the eight Qwen, Phi and Mistral models at $d=8$ from $0.01$--$0.15$ to $0.43$--$0.99$, and Mistral-Small-24B answers $0.70$ of the chains correctly at $d=24$. For the three OLMo-2 models, one level per call gives $0.04$--$0.06$ at $d=8$, against $0.00$--$0.03$ in one call, and $0.37$--$0.59$ at $d=2$, against $0.43$--$0.72$. For the four models whose $\hat d$ is $2$ or $3$ (Figure~\ref{fig:split}d), calls of $\hat d$ levels need $4$ or $3$ calls at $d=8$ and reach $0.24$--$0.51$, against $0.43$--$0.99$ with one level per call. Every call after the first uses the value returned by the previous call, so the errors of the calls compound, and by the definition of $\hat d$ a call of $\hat d$ levels errs in up to a fifth of the cases. One level per call therefore outperforms calls of $\hat d$ levels at $d=8$ for all four models where the two differ, since each of its calls stays at a depth where a single call is almost always correct.

\clearpage
\subsection{Summary of the Eight Experiments}
\label{app:uses}

Table~\ref{tab:apps} summarizes the eight experiments. Each concerns one part of Theorem~\ref{thm:unified_bound}: Experiment~1 the balance between the residual $\epsilon$ and the estimation of $W_1$, Experiments~2 and~3 the distance $W_1$ in the shift term, Experiments~4--6 the factor $K_{f,\mathrm{PE}}$ through Lemma~\ref{lemma:ape_shattering}, and Experiments~7 and~8 the depth term.

\begin{table}[!h]
\centering\footnotesize
\setlength{\tabcolsep}{3pt}
\caption{Results of the eight experiments of Appendix~\ref{app:beyond}. \emph{Experiment}: number, short name and section. \emph{Setting}: the models and data. \emph{Finding}: the main numbers and what they show.}
\label{tab:apps}
\begin{tabular}{@{}p{0.20\linewidth}p{0.28\linewidth}p{0.48\linewidth}@{}}
\toprule
\ontoR\textbf{Experiment} & \ontoR\textbf{Setting} & \ontoR\textbf{Finding} \tabularnewline
\midrule
\ontoR 1. More coordinates on Countdown (\ref{app:refine}) & \ontoR 12, 16, 24 and 34 coordinates; $f_1$ and $f_2$; training trajectories of the three measures & \ontoR ROC-AUC changes by at most $0.003$, while $\widehat W_1$ between measures grows by a factor of $1.8$--$2.0$; the twelve coordinates are the better choice \tabularnewline
\ontoR 2. Groups of agent, embodied and robot trajectories (\ref{app:c2-groups}) & \ontoR $52$ pairs of groups from the ETO expert trajectories, the Nebius code-agent runs, BridgeData V2 and BabyAI & \ontoR $33$ pairs separated ($34$ without surface coordinates); tasks that share a sequence of subgoals give the smallest ratios \tabularnewline
\ontoR 3. One model in two agent harnesses (\ref{app:c2-harness}) & \ontoR Qwen3-8B in mini-SWE-agent and OpenHands, $300$ SWE-bench Verified instances & \ontoR $R=2.19$, and every coordinate moves by $1.95$--$11.5$ times its own baseline; the resolve rate rises from $1.7\%$ to $6.4\%$ \tabularnewline
\ontoR 4. Attention scores of pretrained models (\ref{app:c1-kernel}) & \ontoR Four APE and four RoPE models; shifts of $16$, $64$ and $256$ & \ontoR Under APE the first-layer change is $0.42$--$1.19$ of the largest score; under RoPE the change is rounding (at most $4.4\times10^{-12}$ in float64) \tabularnewline
\ontoR 5. Answers under shifted position indices (\ref{app:serving}) & \ontoR $33$ public models; all indices shifted, left padding, cached prefix, sliding window & \ontoR In fp32 the ALiBi and RoPE models change $40$ of $948{,}000$ answers, and the APE models lose $0.13$--$0.52$ in accuracy at $k=256$; a gap after a cached prefix changes the answers under every encoding \tabularnewline
\ontoR 6. Two Countdown models (\ref{app:g1}) & \ontoR $f_1$ (RoPE) and $f_2$ (APE) of Section~\ref{sec:experiments} & \ontoR No outcome of the RoPE model changes; the success of the APE model falls by $0.40$ and $0.39$ (shifts $16$ and $64$) \tabularnewline
\ontoR 7. Nested plans (\ref{app:nestplan}) & \ontoR $120$ runs ($40$ configurations, three seeds) trained from scratch, $d_{\mathrm{train}}\in\{2,4,6,8\}$ & \ontoR In the $114$ runs that learned their training depths, splitting gives a mean success of at least $0.987$ beyond $d_{\mathrm{ext}}$, against at most $0.04$ for one call; $d_{\mathrm{ext}}$ follows $d_{\mathrm{train}}$ \tabularnewline
\ontoR 8. Nested tool calls (\ref{app:llm-handoff}) & \ontoR $11$ open-weight models of 7B--32B parameters & \ontoR With one level per call, the accuracy of eight models at $d=8$ rises from $0.01$--$0.15$ to $0.43$--$0.99$; calls of $\hat d$ levels reach $0.24$--$0.51$ \tabularnewline
\bottomrule
\end{tabular}
\end{table}
\clearpage

\newpage
\section*{NeurIPS Paper Checklist}

\begin{enumerate}

\item {\bf Claims}
    \item[] Question: Do the main claims made in the abstract and introduction accurately reflect the paper's contributions and scope?
    \item[] Answer: \answerYes{} %
    \item[] Justification: The abstract and Section \ref{sec:intro} state our theoretical bounds (Lipschitz constants, $\text{TC}^0$ limits) and the empirical scope (54 trained models), consistent with the theorems and experiments that follow.


\item {\bf Limitations}
    \item[] Question: Does the paper discuss the limitations of the work performed by the authors?
    \item[] Answer: \answerYes{} %
    \item[] Justification: The paragraph on limitations and future work in Section \ref{sec:conclusion} discusses the scope of the $\epsilon$-relaxed structural sufficiency assumption, its extension to open-ended natural language reasoning, and next steps for the framework.


\newpage\item {\bf Theory assumptions and proofs}
    \item[] Question: For each theoretical result, does the paper provide the full set of assumptions and a complete (and correct) proof?
    \item[] Answer: \answerYes{} %

    \item[] Justification: All theoretical assumptions are stated explicitly: Assumptions \ref{assum:epsilon_relaxed} and \ref{assum:ar_sensitivity} in the main text, and Assumptions \ref{ass:tsybakov-margin}--\ref{ass:implicit-complexity} in Appendix \ref{app:capacity-results}. Formal, complete proofs are provided in Appendices \ref{app:structural-results}, \ref{app:pe-results}, \ref{app:capacity-results}, and \ref{app:main-results}.


    \item {\bf Experimental result reproducibility}
    \item[] Question: Does the paper fully disclose all the information needed to reproduce the main experimental results of the paper to the extent that it affects the main claims and/or conclusions of the paper (regardless of whether the code and data are provided or not)?
    \item[] Answer: \answerYes{} %
    \item[] Justification: We detail the 54 Transformer configurations in Table \ref{tab:models}, the data generation process in Section \ref{sec:experiments}, and the full training and evaluation settings (hyperparameters, prompts, decoding, evaluation sets) in Appendix \ref{app:reproducibility}; the code is publicly available (Appendix~\ref{app:code}). Appendix~\ref{app:beyond} summarizes the models, data, conditions and results of its experiments; their full designs and per-pair, per-model and per-run results are in the code repository (Appendix~\ref{app:code}). %


\item {\bf Open access to data and code}
    \item[] Question: Does the paper provide open access to the data and code, with sufficient instructions to faithfully reproduce the main experimental results, as described in supplemental material?
    \item[] Answer: \answerYes{} %

    \item[] Justification: Code and instructions for regenerating the synthetic data and reproducing the main results are available on GitHub (Appendix~\ref{app:code}). %


\item {\bf Experimental setting/details}
    \item[] Question: Does the paper specify all the training and test details (e.g., data splits, hyperparameters, how they were chosen, type of optimizer) necessary to understand the results?
    \item[] Answer: \answerYes{} %

    \item[] Justification: The training distributions ($\mu_{BFS}, \mu_{DFS}, \mu_{MIXED}$), model aspect ratios, layer depths, and widths are detailed in Section \ref{sec:experiments} and Table \ref{tab:models}; optimizer, learning rate, batch sizes, schedule, epochs, seeds and the evaluation sets are given in Appendix \ref{app:reproducibility}. %


\item {\bf Experiment statistical significance}
    \item[] Question: Does the paper report error bars suitably and correctly defined or other appropriate information about the statistical significance of the experiments?
    \item[] Answer: \answerYes{} %

    \item[] Justification: We report standard deviations to analyze structural variance across different architectural aspect ratios under matched core-parameter budgets in Appendix \ref{app:baseline_variance} (Figure \ref{fig:app_baseline_variance}). Each configuration is trained once (seed 42), so these error bars measure variation across aspect ratios. We also report a one-sided sign test over the 27 matched, separately trained APE/RoPE model pairs (RoPE ahead in 25, $p<10^{-5}$; Section \ref{sec:experiments}), $\widehat W_1$ as mean $\pm$ s.d.\ over 10 subsamples (Appendix \ref{app:w1_estimation}), and s.d.\ over five folds (Appendix \ref{app:gbc_check}). In Appendix~\ref{app:c2}, $\widehat W_1$ is the mean (s.d.) over 10 seeds, and a pair is separated when its between-group $\widehat W_1$ is at least three times the larger of its two within-group baselines and has a coefficient of variation over seeds of at most $0.10$.


\item {\bf Experiments compute resources}
    \item[] Question: For each experiment, does the paper provide sufficient information on the computer resources (type of compute workers, memory, time of execution) needed to reproduce the experiments?
    \item[] Answer: \answerYes{} %
    \item[] Justification: Appendix \ref{app:reproducibility} reports the hardware (NVIDIA A100 40 GB, one GPU per run), CPU and memory allocations, per-scale training wall-clock times, and the total compute (about 907 GPU-hours for training and 324 GPU-hours for evaluation of the 54 configurations, plus about 17 GPU-hours to label the problems of Appendix~\ref{app:gbc_check}; the $\widehat W_1$ estimates of Appendix~\ref{app:w1_estimation} and the analyses of Appendix~\ref{app:gbc_check} used CPUs only; for Appendix~\ref{app:beyond}, about 23.5 GPU-hours for Experiment~3, 0.45 for Experiment~4, 106 for Experiment~5, 4.5 for Experiment~6, 119 for Experiment~7 and 13 for Experiment~8 on NVIDIA A100, H100 and L40S GPUs, Experiments~1 and~2 on CPUs only; about 57 GPU-hours for the retraining check). %

    
\item {\bf Code of ethics}
    \item[] Question: Does the research conducted in the paper conform, in every respect, with the NeurIPS Code of Ethics \url{https://neurips.cc/public/EthicsGuidelines}?
    \item[] Answer: \answerYes{} %
    \item[] Justification: This is a theoretical and empirical machine learning study that fully complies with the NeurIPS Code of Ethics.


\item {\bf Broader impacts}
    \item[] Question: Does the paper discuss both potential positive societal impacts and negative societal impacts of the work performed?
    \item[] Answer: \answerYes{} %
    \item[] Justification: The broader impact paragraph of Section \ref{sec:conclusion} discusses the positive implications of our findings for the design of reasoning architectures and for inference-time scaling. As a negative impact, our bounds, which hold under stated assumptions and were evaluated on Countdown, could be cited to certify an architecture as reliable for consequential reasoning in place of task-specific evaluation. More capable reasoning models could also be misused, as with other capability gains; the paper studies a controlled formal task and releases research code, so any path to such misuse is indirect.
    
\item {\bf Safeguards}
    \item[] Question: Does the paper describe safeguards that have been put in place for responsible release of data or models that have a high risk for misuse (e.g., pre-trained language models, image generators, or scraped datasets)?
    \item[] Answer: \answerNA{} %

    \item[] Justification: This paper studies theoretical capacity bounds and evaluates them on a formal combinatorial game (Countdown); the released research code poses no high risk of misuse or dual use.


\item {\bf Licenses for existing assets}
    \item[] Question: Are the creators or original owners of assets (e.g., code, data, models), used in the paper, properly credited and are the license and terms of use explicitly mentioned and properly respected?
    \item[] Answer: \answerYes{} %
    \item[] Justification: We credit all assets: the Stream of Search code used for Countdown data generation and scoring \citep{gandhi2024streamsearchsoslearning} (Apache-2.0 license), Hugging Face Transformers, Accelerate and Datasets (Apache-2.0), PyTorch (BSD-3-Clause), scikit-learn \citep{pedregosa2011scikit} (BSD-3-Clause), the POT library \citep{flamary2021pot} (MIT), and the pretrained GPT-Neo tokenizer (EleutherAI/gpt-neo-1.3B, MIT); for Appendix~\ref{app:beyond}, the pretrained GPT-2 models \citep{radford2019language} (modified MIT license) and GPT-Neo models \citep{black2021gptneo} (MIT), the Pythia and SmolLM2 models \citep{biderman2023pythia,allal2025smollm2} (Apache-2.0), the Python standard library (PSF License 2.0), the ETO trajectories \citep{song2024etodata} (Apache-2.0) of ALFWorld (MIT), ScienceWorld (Apache-2.0) and WebShop (MIT), the SWE-agent and OpenHands trajectory datasets of Nebius \citep{nebius2024sweagenttrajectories,trofimova2025openhandstrajs} (CC-BY-4.0; the SWE-agent dataset card also refers its users to the Llama 3.1 license), the ECoT features \citep{embodiedcot2024bridgefeatures} (MIT) of BridgeData V2 \citep{walke2023bridgedata} (CC-BY-4.0), Minigrid \citep{chevalierboisvert2023minigrid} (Apache-2.0) with the BabyAI bot \citep{chevalierboisvert2019babyai} (BSD-3-Clause), the OPT models \citep{zhang2022opt} (OPT-175B license agreement, non-commercial research use), the BLOOM models \citep{bigscience2022bloom} (BigScience BLOOM RAIL License v1.0), GPT-J \citep{wang2021gptj} (Apache-2.0), the Qwen2.5 models \citep{qwen2024qwen25} (Apache-2.0; Qwen2.5 3B under the Qwen Research License), the OLMo-2 and Qwen3 models \citep{olmo2024olmo2,qwen2025qwen3} (Apache-2.0), Phi-4 \citep{abdin2024phi4} (MIT), Mistral-Small-24B-Instruct-2501 \citep{mistral2025small3} (Apache-2.0), vLLM \citep{kwon2023vllm} (Apache-2.0), mini-SWE-agent and OpenHands \citep{yang2024sweagent,wang2025openhands} (MIT), and SWE-bench Verified \citep{jimenez2024swebench,openai2024swebenchverified} (SWE-bench code and harness MIT); see Section \ref{sec:experiments} and Appendices \ref{app:gbc_check}, \ref{app:reproducibility} and \ref{app:beyond}. %

\item {\bf New assets}
    \item[] Question: Are new assets introduced in the paper well documented and is the documentation provided alongside the assets?
    \item[] Answer: \answerYes{} %
    \item[] Justification: Our code, including the scripts that generate the synthetic data, is available on GitHub with documentation (Appendix~\ref{app:code}). %

\item {\bf Crowdsourcing and research with human subjects}
    \item[] Question: For crowdsourcing experiments and research with human subjects, does the paper include the full text of instructions given to participants and screenshots, if applicable, as well as details about compensation (if any)? 
    \item[] Answer: \answerNA{} %
    \item[] Justification: This research consists of mathematical proofs and computational experiments, and does not involve human subjects or crowdsourcing.

\item {\bf Institutional review board (IRB) approvals or equivalent for research with human subjects}
    \item[] Question: Does the paper describe potential risks incurred by study participants, whether such risks were disclosed to the subjects, and whether Institutional Review Board (IRB) approvals (or an equivalent approval/review based on the requirements of your country or institution) were obtained?
    \item[] Answer: \answerNA{} %

    \item[] Justification: Not applicable, as no human participants were involved in this study.

\item {\bf Declaration of LLM usage}
    \item[] Question: Does the paper describe the usage of LLMs if it is an important, original, or non-standard component of the core methods in this research? Note that if the LLM is used only for writing, editing, or formatting purposes and does \emph{not} impact the core methodology, scientific rigor, or originality of the research, declaration is not required.
    \item[] Answer: \answerNA{} %
    \item[] Justification: LLMs are the subject of our theoretical and empirical analysis; no LLM is an important, original, or non-standard component of our core methods.

\end{enumerate} %

\end{document}